 \documentclass[final,3p,times]{elsarticle}
\usepackage{amsmath, amssymb, amsthm}
\usepackage{geometry}
\usepackage{hyperref}
\usepackage{enumitem}
\usepackage{graphicx}
\usepackage{hyperref}
\usepackage{physics}
\usepackage{amsfonts}
\usepackage{times}
\usepackage{color}
\usepackage{bm}
\usepackage{bbold}
\usepackage[bbgreekl]{mathbbol}
\usepackage{breqn}
\usepackage{algorithm}
\usepackage{algorithmic}
\usepackage{accents}
\usepackage{multicol}
\usepackage{caption}
\usepackage{subcaption}
\usepackage{float}
\usepackage{tikz}
\usetikzlibrary{arrows.meta, positioning}

\font\ppppppcarac=ptmr8y at 4pt
\font\pppppcarac=ptmr8y at 5pt
\font\ppppcarac=ptmr8y at 6pt
\font\pppcarac=ptmr8y at 7pt
\font\ppcarac=ptmr8y at 8pt
\font\pcarac=ptmr8y at 9pt

\font\Ppcarac=ptmr8y at 10pt

\font\bf=ptmb8y at 10pt

\newcommand{\bfA}{{\textbf{A}}}
\newcommand{\bfB}{{\textbf{B}}}

\newcommand{\bfD}{{\textbf{D}}}

\newcommand{\bfF}{{\textbf{F}}}

\newcommand{\bfH}{{\textbf{H}}}

\newcommand{\bfJ}{{\textbf{J}}}

\newcommand{\bfS}{{\textbf{S}}}

\newcommand{\bfU}{{\textbf{U}}}
\newcommand{\bfV}{{\textbf{V}}}
\newcommand{\bfW}{{\textbf{W}}}
\newcommand{\bfX}{{\textbf{X}}}
\newcommand{\bfY}{{\textbf{Y}}}

\newcommand{\bfa}{{\textbf{a}}}
\newcommand{\bfb}{{\textbf{b}}}
\newcommand{\bfc}{{\textbf{c}}}

\newcommand{\bfe}{{\textbf{e}}}
\newcommand{\bff}{{\textbf{f}}}
\newcommand{\bfg}{{\textbf{g}}}
\newcommand{\bfh}{{\textbf{h}}}

\newcommand{\bfn}{{\textbf{n}}}

\newcommand{\bfu}{{\textbf{u}}}

\newcommand{\bfw}{{\textbf{w}}}
\newcommand{\bfx}{{\textbf{x}}}
\newcommand{\bfy}{{\textbf{y}}}
\newcommand{\bfz}{{\textbf{z}}}

\newcommand{\bfzero}{{ \hbox{\bf 0} }}

\newcommand{\bfGamma}{{\bm{\Gamma}}}

\newcommand{\bfPhi}{{\bm{\Phi}}}

\newcommand{\bfalpha}{{\bm{\alpha}}}
\newcommand{\bfbeta}{{\bm{\beta}}}
\newcommand{\bfgamma}{{\bm{\gamma}}}

\newcommand{\bfeta}{{\bm{\eta}}}
\newcommand{\bftheta}{{\bm{\theta}}}

\newcommand{\bfvarphi}{{\bm{\varphi}}}
\newcommand{\bfpsi}{{\bm{\psi}}}

\newcommand{\bfxi}{{\bm{\xi}}}

\newcommand{\FF}{{\mathbb{F}}}

\newcommand{\HH}{{\mathbb{H}}}

\newcommand{\MM}{{\mathbb{M}}}
\newcommand{\NN}{{\mathbb{N}}}
\newcommand{\PP}{{\mathbb{P}}}
\newcommand{\RR}{{\mathbb{R}}}

\newcommand{\XX}{{\mathbb{X}}}
\newcommand{\YY}{{\mathbb{Y}}}

\newcommand{\aaeclair}{{\mathbb{a}}}
\newcommand{\bb}{{\mathbb{b}}}

\newcommand{\hh}{{\mathbb{h}}}

\newcommand{\xx}{{\mathbb{x}}}
\newcommand{\yy}{{\mathbb{y}}}

\DeclareMathAlphabet{\mathonebb}{U}{bbold}{m}{n}
\def\11{{\ensuremath{\mathonebb{1}}}}

\newcommand{\curB}{{\mathcal{B}}}
\newcommand{\curD}{{\mathcal{D}}}
\newcommand{\curE}{{\mathcal{E}}}
\newcommand{\curG}{{\mathcal{G}}}
\newcommand{\curH}{{\mathcal{H}}}

\newcommand{\curJ}{{\mathcal{J}}}
\newcommand{\curK}{{\mathcal{K}}}
\newcommand{\curL}{{\mathcal{L}}}

\newcommand{\curN}{{\mathcal{N}}}

\newcommand{\curP}{{\mathcal{P}}}

\newcommand{\curT}{{\mathcal{T}}}

\newcommand{\curV}{{\mathcal{V}}}
\newcommand{\curW}{{\mathcal{W}}}

\newcommand{\curZ}{{\mathcal{Z}}}

\usepackage{mathrsfs}
\newcommand{\curC}{{\mathscr{C}}}       
\newcommand{\curS}{{\mathscr{S}}}       

\newcommand{\bfcurH}{{\boldsymbol{\mathcal{H}}}}

\newcommand{\bfcurU}{{\boldsymbol{\mathcal{U}}}}

\font\vcarac=ptmr8y at 10pt

\newcommand{\vc}{{\hbox{\vcarac{v}}}}    
\newcommand{\uc}{{\hbox{\vcarac{u}}}}    

\newcommand{\adp}{{{\hbox{{\ppppcarac ad}}}}}

\newcommand{\adam}{{{\hbox{{\ppppcarac adam}}}}}

\newcommand{\ANN}{{\hbox{\pcarac ANN}}}
\newcommand{\ANNp}{{\hbox{\ppppcarac ANN}}}
\newcommand{\ANNpp}{{\hbox{\pppppcarac ANN}}}

\newcommand{\CI}{\hbox{{\pcarac CI}}}
\newcommand{\CIp}{\hbox{{\ppcarac CI}}}

\newcommand{\criter}{\hbox{{\ppcarac Cr}}}

\newcommand{\CRPS}{{\hbox{{\ppcarac CRPS}}}}
\newcommand{\CRPSp}{{\hbox{{\ppppcarac CRPS}}}}
\newcommand{\CRPSpp}{{\hbox{{\pppppcarac CRPS}}}}

\newcommand{\pdiag}{{\hbox{{\ppppcarac diag}}}}

\newcommand{\pelem}{{\hbox{{\pppcarac elem}}}}

\newcommand{\err}{\hbox{{\Ppcarac err}}}
\newcommand{\perr}{\hbox{{\ppcarac err}}}

\newcommand{\pgrid}{{\hbox{{\pppcarac grid}}}}
\newcommand{\ppgrid}{{\hbox{{\pppppcarac grid}}}}

\newcommand{\pin}{{\hbox{{\pppcarac in}}}}
\newcommand{\ppin}{{\hbox{{\pppppcarac in}}}}
\newcommand{\pppin}{{\hbox{{\pppppcarac in}}}}

\newcommand{\pint}{{\hbox{{\pppcarac int}}}}

\newcommand{\pmax}{{\hbox{{\ppppcarac max}}}}

\newcommand{\NLL}{{\hbox{\pcarac NLL}}}

\newcommand{\NLLg}{{\hbox{\ppppcarac NLL}}}

\newcommand{\obs}{{\hbox{{\pppcarac obs}}}}

\newcommand{\obspp}{{\hbox{{\pppppcarac obs}}}}

\newcommand{\opt}{{\,\hbox{{\pppcarac opt}}}}
\newcommand{\optp}{{\,\hbox{{\ppppcarac opt}}}}
\newcommand{\optpp}{{\,\hbox{{\pppppcarac opt}}}}

\newcommand{\pout}{{\hbox{{\pppcarac out}}}}
\newcommand{\ppout}{{\hbox{{\pppppcarac out}}}}
\newcommand{\pppout}{{\hbox{{\pppppcarac out}}}}

\newcommand{\pPCA}{{{\hbox{{\pppppcarac PCA}}}}}

\newcommand{\pPoisson}{{\hbox{{\pppcarac poisson}}}}

\newcommand{\pprc}{{{\hbox{{\ppppcarac prc}}}}}

\newcommand{\pSB}{{\hbox{{\ppppppcarac SB}}}}

\newcommand{\psim}{{\hbox{{\ppppcarac sim}}}}

\newcommand{\psp}{{{\hbox{{\ppppcarac sp}}}}}

\newcommand{\test}{{\hbox{{\pppcarac test}}}}
\newcommand{\testp}{{\hbox{{\ppppcarac test}}}}
\newcommand{\testpp}{{\hbox{{\pppppcarac test}}}}

\newcommand{\tor}{{\hbox{{\ppppcarac tor}}}}

\newcommand{\trainingpp}{{\hbox{{\ppppcarac train}}}}

\newcommand{\ptrial}{{\hbox{{\pppcarac trial}}}}
\newcommand{\pptrial}{{\hbox{{\ppppcarac trial}}}}
\newcommand{\ppptrial}{{\hbox{{\pppppcarac trial}}}}

\newcommand{\TRIAL}{{\hbox{{\pppcarac TRIAL}}}}
\newcommand{\PGD}{{\hbox{{\pppcarac PGD}}}}

\newcommand{\Unit}{\mathbf{1}}

\newcommand{\wind}{{\hbox{{\ppppcarac wind}}}}

\newtheorem{proposition}{Proposition}

\journal{ArXiv}
\numberwithin{equation}{section}

\title{Supervised Learning of Random Neural Architectures Structured by Latent Random Fields on Compact Boundaryless Multiply-Connected Manifolds}

\author[1]{Christian Soize \corref{cor1}}
\ead{christian.soize@univ-eiffel.fr}
\cortext[cor1]{Corresponding author: C. Soize, christian.soize@univ-eiffel.fr}
\address[1]{Universit\'e Gustave Eiffel, MSME UMR 8208, 5 bd Descartes, 77454 Marne-la-Vall\'ee, France}

\date{April 2025}

\begin{document}

\begin{frontmatter}

\begin{abstract}
This paper introduces a new probabilistic framework for supervised learning in neural systems. It is designed to model complex, uncertain systems whose random outputs are strongly non-Gaussian given deterministic inputs. The architecture itself is a random object stochastically generated by a latent anisotropic Gaussian random field defined on a compact, boundaryless, multiply-connected manifold. The objective is not to propose an alternative to existing fixed-topology architectures or to outperform current models in specific empirical settings. Rather, the goal is to establish a novel conceptual and mathematical framework in which neural architectures are realizations of a geometry-aware, field-driven generative process. In this formulation, both the neural topology and synaptic weights emerge jointly from a latent random field, integrating structural and topological uncertainty as intrinsic features of the model. A reduced-order parameterization governs the spatial intensity of an inhomogeneous Poisson process on the manifold, from which neuron locations are sampled. Input and output neurons are identified via extremal evaluations of the latent field, while connectivity is established through geodesic proximity and local field affinity. Synaptic weights are conditionally sampled from the field realization, inducing stochastic output responses even for deterministic inputs. To ensure scalability, the architecture is sparsified via percentile-based diffusion masking, yielding geometry-aware sparse connectivity without ad hoc structural assumptions. Supervised learning is formulated as inference on the generative hyperparameters of the latent field, using a negative log-likelihood loss estimated through Monte Carlo sampling from single-observation-per-input datasets. The paper initiates a mathematical analysis of the model, establishing foundational properties such as well-posedness, measurability, and a preliminary analysis of the expressive variability of the induced stochastic mappings, which support its internal coherence and lay the groundwork for a broader theory of geometry-driven stochastic learning. A numerical illustration on a torus is provided, not to demonstrate predictive superiority over existing models, but to elucidate how topological, anisotropic, and stochastic field parameters influence the emergent neural architecture and its induced function class.
\end{abstract}


\end{frontmatter}

\section{Introduction}
\label{Section1}
%
\noindent \textit{(i) Objective of the paper}.
The objective of this paper is to develop a mathematically rigorous probabilistic framework for supervised learning in neural systems, designed to model complex, uncertain systems whose random outputs are strongly non-Gaussian given deterministic inputs.
The architectures are not fixed but are instead stochastically generated. These architectures are defined on compact, boundaryless, multiply-connected manifolds and are governed by latent Gaussian random fields. The proposed model generates both the neural topology and synaptic weights as emergent properties of a reduced-order representation of such a latent field.

Neurons are stochastically positioned via a Poisson point process whose spatial intensity is modulated by the latent field. Input and output neurons are selected from these points based on extremal evaluations of the field, and connections are established based on geodesic proximity and field affinity. Random weights are then assigned conditionally on the field realization, inducing a stochastic response even for deterministic inputs.

To enable scalability, a sparsification procedure is introduced that significantly reduces network connectivity via percentile-thresholded diffusion kernels. Crucially, supervised learning is not framed as direct weight optimization but as the statistical inference of the hyperparameters controlling the generative process. A negative log-likelihood loss is defined from the conditional density of outputs given inputs, which is estimated via Monte Carlo sampling over network realizations. This formulation allows training with only a single observation per input, reflecting the model’s capacity to internalize structural uncertainty.\\

\noindent \textit{(ii) Related work and background}.
Supervised learning of neural networks has been a foundational paradigm in modern machine learning, where models are trained to approximate functions by minimizing empirical risk over labeled datasets \cite{Bishop1995,Goodfellow2016}. Classical architectures, including multilayer perceptrons (MLPs), convolutional neural networks (CNNs), and recurrent neural networks (RNNs), have been extensively studied in this context, with well-established training algorithms such as stochastic gradient descent (SGD) and its variants \cite{Rumelhart1986,Lecun1998}.
More recently, research has moved beyond fixed architectures toward random and stochastic neural models. These include, randomly initialized networks used as priors in Bayesian inference \cite{Neal2012,Lee2017}, and, neural networks with random weights, as universal feature generators, exemplified by random kitchen sinks and extreme learning machines \cite{Rahimi2007,Huang2006}. A growing body of work explores neural architecture search (NAS), where network structure itself is treated as a variable to optimize, though often in a deterministic or heuristic fashion \cite{Zoph2016,Elsken2019}.
In parallel, theoretical developments have begun to address the stochastic modeling of network architectures, where randomness is incorporated not only in weights but also in connectivity patterns, depth, or activation functions \cite{Xie2019,Chen2020}. Such approaches aim to capture richer inductive biases, model uncertainty, or accommodate data with complex latent geometry. Notably, the randomization of architectures has been used for model averaging, robustness, and in Bayesian deep learning frameworks \cite{Blundell2015,Gal2016}.
Among the significant advances in geometry-aware learning architectures are Graph Neural Networks (GNNs), which extend deep learning to non-Euclidean domains such as graphs, manifolds, and combinatorial complexes. GNNs operate by aggregating and transforming information across nodes and their neighborhoods, capturing local and global relational structure \cite{Scarselli2008,Hamilton2017,Kipf2017}.
Foundational variants include Graph Convolutional Networks (GCNs) \cite{Kipf2017}, Graph Attention Networks (GATs) \cite{Velickovic2018}, and Message Passing Neural Networks (MPNNs) \cite{Gilmer2017}, which have been further extended to dynamic, heterogeneous, and hierarchical settings \cite{Battaglia2018,Zhou2020b}.
Additionally, recent studies have drawn connections between GNNs and manifold learning, where graph Laplacians serve as discrete approximations of differential operators on smooth manifolds \cite{Bronstein2017,Belkin2003,Grover2025}. However, these models typically assume fixed input graphs, with deterministic structure.
Unlike approaches that endow fixed graphs with probabilistic weights or node-level uncertainty, the present work defines a probability distribution over entire neural architectures, including the placement of neurons and their connectivity, as a function of a continuous latent field defined over a manifold.
Efforts to incorporate probabilistic structure into GNNs include Bayesian GNNs \cite{Hasanzadeh2020,Wang2024} and methods for learning distributions over graph topologies \cite{You2018}. However, these approaches remain largely constrained to discrete graph spaces and do not leverage continuous latent processes or intrinsic geometric priors \cite{Chen2025}.
In contrast, the proposed model treats the neural architecture as an emergent, field-driven object, where both topology and weight distribution arise from a latent Gaussian random field on a compact manifold. This coupling allows the stochastic neural representation to inherit smoothness, anisotropy, and topological complexity directly from the geometry.
While some advances in manifold learning—such as intrinsic graph convolution, diffusion models~\cite{DeBortoli2022}, and variational autoencoders on curved spaces—have begun to address geometric complexity, these typically rely on fixed embeddings or learned kernels rather than stochastic generative fields.

In parallel, Gaussian process-based surrogate modeling has provided alternative approaches for uncertainty quantification in physics-informed systems. Related surrogate modeling strategies based on Gaussian processes have also been developed in computational physics contexts \cite{Yang2019c,Bilionis2012,Tripathy2016}. Gaussian process regression provides a surrogate modeling technique that enables exploration of posterior distributions~\cite{Williams2006,Kennedy2001}. In the deterministic setting, Physics-Informed Neural Networks PINNs \cite{Raissi2019,Yuan2025} provide a framework for solving dynamical systems governed by ordinary or partial differential equations. In the probabilistic setting, extending PINNs to nonlinear stochastic partial differential equations  poses significant challenges, and several strategies have been proposed in the recent literature \cite{Chen2021,Yang2021b}. Generative PINNs, including those based on variational autoencoders (VAE-PINNs) or generative adversarial networks (GAN-PINNs) \cite{Yang2022,Zhong2023}, aim to model the full probability distribution of the solution field.

One of the goals of the proposed model is to represent non-Gaussian output behavior induced by random mappings, using a small number of neurons (e.g., $20$ input, $80$ hidden, $100$ output) and limited training data (e.g., $1000$ samples). These non-Gaussian features emerge from the structural randomness of the model itself and are not artificially injected as output noise. The challenge of constructing surrogate models for stochastic systems in the small-data regime aligns with the probabilistic learning framework on manifolds is proposed in \cite{Soize2016,Soize2019b,Soize2020b,Soize2021a,Soize2024,Soize2024b}.

Classical and generalized Gaussian processes  have been defined via kernel constructions, \cite{Kree1986,Rasmussen2006,Malyarenko2014,Akian2022,Zhi2023}). The work of Lindgren et al.\ \cite{Lindgren2011,Lindgren2022} established a connection between Gaussian fields and Gaussian Markov random fields (GMRFs) through SPDEs of the form $\curL = (\tau_0 - \Delta)^{\alpha/2} u(\bfx) = \curW(\bfx)$, with $\alpha = \nu + d/2$, $\nu > 0$, on a Euclidean domain $\bfx \in \RR^d$. This formulation is effective for stationary and isotropic fields and has been extended to anisotropic and manifold-based settings.
When the fractional order $\alpha/2$ is the integer $1$, the SPDE becomes a standard elliptic equation. In such a case, the weak formulation of the corresponding boundary value problem allows the use of finite element methods for constructing realizations of the Gaussian random field, and this formulation can be extended to a multiply-connected manifold on which the Gaussian random field lives (see, for instance, \cite{Staber2018}).
For fractional powers $\alpha \in (0,2)$, Bonito et al. \cite{Bonito2015} introduced efficient numerical algorithms for approximating fractional powers of elliptic operators, although their work focuses on flat domains without explicit spatial anisotropy or complex geometric constraints.
More recent developments adapt SPDE models to cortical surfaces but still often assume relatively simple topologies and mostly isotropic behavior. Fuglstad et al. \cite{Fuglstad2015,Fuglstad2015b} proposed modeling spatial anisotropy in Gaussian processes through local metric tensors and deformation models. Similarly, Borovitskiy et al. \cite{Borovitskiy2020} developed intrinsic kernels on manifolds based on heat kernels and geodesic distances; these approaches typically require kernel construction or metric learning.
Recent advances in geometric machine learning highlight manifold-based probabilistic modeling, including generative diffusion models \cite{DeBortoli2022}, manifold VAEs, and intrinsic graph convolutional frameworks \cite{Bronstein2021,Monti2017}.

A distinct yet relevant line of research involves generative models of graphs and functions, including GraphVAE \cite{Simonovsky2018}, GraphAF \cite{Shi2020}, and other deep generative approaches that learn probability distributions over graph topologies or node features. These models differ significantly from the current work: they typically learn from observed graphs, use discrete latent representations, and are not grounded in continuous field-based geometry. Nevertheless, they share the ambition of sampling from a graph distribution rather than optimizing over fixed structures. Similarly, in the context of Gaussian processes on manifolds, recent works \cite{Borovitskiy2020,Wilson2020} construct intrinsic kernels via heat kernels or spectral approximations, allowing for manifold-aware inference with GP priors. While such methods focus on function-space uncertainty, the present model treats architectural uncertainty as primary, generating both the structure and the mapping as a single stochastic object from a latent field. This distinguishes the proposed approach from classical GP regression, while complementing its geometric foundations.\\

\noindent \textit{(iii) A short survey on the methodology proposed}.
The construction of the random neural architecture on the compact, boundaryless, multiply-connected manifold $\curS$ is performed s follows.
\begin{itemize}
\item A real-valued, anisotropic Gaussian random field $\{U(\bfx),\bfx \in \curS\}$ is constructed on a given compact, boundaryless, multiply-connected 2-dimensional manifold immersed in the 3-dimensional Euclidean space $\RR^3$. This random field is defined
     {\color{black}
      as the solution of a stochastic partial differential equation (SPDE) with a strongly elliptic operator $\curL_{n_h}$
     that incorporates a family of  anisotropic coefficient fields $[K_{n_h}]$ on the $\curS$-manifold, indexed by $n_h\in\NN$;
      the solution
       }
       is a second-order,  centered, Gaussian real-valued random field on $\curS$. The manifold $\curS$ is geometrically approximated by a manifold $\curS_h$, constructed using the finite element method. The Gaussian random field on the manifold $\curS_h$ is obtained by discretizing  the weak formulation of the stochastic PDE on $\curS_h$. For fixed $n_h$, this centered, second-order, Gaussian,  real-valued random field is  parameterized by the anisotropic field $[K_{n_h}]$,
       {\color{black} which is itself parameterized by two bounded and positive-valued fields $h^{(1)}_{n_h}(\bfx)$ and $h^{(2)}_{n_h}(\bfx)$  for $\bfx \in \curS$}.

\item An inhomogeneous Poisson point process is introduced with an intensity related to the Gaussian field $U$, allowing the generation of a set of $N$ points $\curV = \{\bfx^i, i =1,\ldots, N\} \subset \curS_h$, which represent the random locations of $N$ neurons on $\curS_h$.

\item A sparsity binary mask is constructed to define the neural graph topology (edge construction), using a self-tuning diffusion kernel combined with percentile-based sparsification, resulting in a significantly reduced number of edges $n_e \ll N(N-1)$. This mask is then applied to a randomly generated dense weight matrix to obtain a sparse random weight matrix. The random dense matrix is constructed as follows: for each pair  $(\bfx^i,\bfx^j) \in \curV \times \curV$, the geodesic distance $d_g(\bfx^i,\bfx^j)$ on $\curS_h$ is computed. A random field strength difference $\Delta S_{ij} = S_i - S_j$ is introduced, where $S_i = U(\bfx^i)$. Preliminary (fully connected) random edge weights $[\bfW]$ are defined by $[\bfW]_{ij} = g(d_g(\bfx^i,\bfx^j), \Delta S_{ij}); \sigma_i,\sigma_j, \zeta_s \,\sigma_S$, where $g: \RR^+ \times \RR \rightarrow \RR^+$ is a prescribed mapping that depends on the local geometric bandwidths $\sigma_i$, $\sigma_j$, and a feature similarity scale $\zeta_s \, \sigma_\bfS$. The parameter $\zeta_s > 0$ controls the relative sensitivity of edge weights to differences in neuron feature values $S_i$, normalized by $\sigma_\bfS$, where $\sigma_\bfS^2 = N^{-1}\operatorname{tr}[C_\bfS]$ and $[C_\bfS]$ is the covariance matrix of the random vector $\bfS = (S_1, \ldots, S_N)$. This model encodes adaptive geometric proximity and field-based affinity: edge weights are high when neurons are geodesically close and have similar field values, and exponentially small otherwise, due to either spatial separation or field mismatch. Since the fully connected graph contains $N(N-1)$ edges, which is computationally prohibitive, sparsification is achieved                                                                                              by applying the binary sparsity mask to the fully connected weight matrix.

\item The hyperparameter of the random neural network on the manifold $\curS_h$ is denoted by $\bftheta_t\in\curC_{\bftheta_t}\subset\RR^{n_{\bftheta_t}}$, where $n_{\bftheta_t}$ is relatively small. It is written as $\bftheta_t = (\bftheta,\bfbeta)$, where $\bftheta = (h_1, h_2,\zeta_s,\bfbeta^{\,(1)}_{n_h}\! , \bfbeta^{\,(2)}_{n_h})$.
    The positive scalars $h_1$ and $h_2$ are  hyperparameters, and $\bfbeta^{\,(1)}_{n_h}$ and $\bfbeta^{\,(2)}_{n_h}$ are $\RR^{n_h}$- valued hyperparameters corresponding to low-rank representations of
    {\color{black}
    the positive-valued fields $\{h^{(1)}_{n_h}(\bfx) ,\bfx\in\curS\}$ and $\{h^{(2)}_{n_h}(\bfx),\bfx\in\curS\}$.
    These fields are, in addition, assumed to be bounded below by a positive constant.
     These hyperparameters specify the anisotropic tensor field $\{[K_{n_h}(\bfx)],\bfx\in\curS\}$, which governs the construction of the Gaussian random field $U$. With hyperparameter vector $\bftheta$, the field $U$ is denoted by $\{U(\bfx ;\bftheta),\bfx\in \curS_h\}$.
     }
    The positive scalar $\zeta_s$ controls the measure of similarity between field values at neuron locations.
     {\color{black}
    The parameter $\tau_0 > 0$ is fixed; it shifts the spectrum of $\curL_{n_h}$ upward, widening the spectral gap and improving conditioning.
     }
     The hyperparameter $\bfbeta\in\RR^{n_b}$ contains the $n_b$ coefficients of a low-rank representation of the bias vector $\bfb = \sum_{j=1}^{n_b} \beta_j \, \bfh^j$ with $n_b\leq N$.
    A realization of the random neural architecture on the manifold $\curS_h$ is generated for each given value of $\bftheta$.
     Given the $n_e$ edges resulting from the sparsification, the  weights are modeled by a non-Gaussian random sparse matrix $[\bfW_\bftheta^\psp]$, which is constructed using the Gaussian random field $U(\cdot ; \bftheta)$ on $\curS_h$, evaluated at
    the $N$ neuron locations $\{\bfx^i_\bftheta, i =1,\ldots, N\}$. For this value of $\bftheta$, we obtain an $N$-neuron system with random weights $[\bfW_\bftheta^\psp]$ and deterministic biases given by the  vector $\bfb$. This system allows us to estimate the probability density function used in constructing the loss function for a supervised input-output training dataset, which is used to identify the vector-valued hyperparameter $\bftheta_t$.
  \end{itemize}
The two-stage generative process underlying the stochastic neural architecture is summarized in Fig.~\ref{figure0}. The first row describes the geometry-driven construction of the architecture, while the second row details the assignment of weights and the inference of generative hyperparameters from supervised data.\\

\begin{figure}[ht]
\centering
\begin{tikzpicture}[
    node distance=1.4cm and 1.5cm,
    every node/.style={draw, rounded corners, align=center, minimum height=1.1cm, minimum width=3.3cm},
    arrow/.style={-{Latex[length=3mm]}, thick}
  ]

  \node (field) {Latent Gaussian \\ random field};
  \node (poisson) [right=of field] {Poisson process \\ for neuron sampling};
  \node (connect) [right=of poisson] {Sparse graph \\ construction};

  \node (weights) [below=of field] {Random edge weights \\ conditioned on field};
  \node (network) [right=of weights] {Stochastic neural system \\ $[\bfW_\bftheta^\psp], \bfb$};
  \node (inference) [right=of network] {Training: inference \\ of hyperparameters $\bftheta_t$};

  \draw [arrow] (field) -- (poisson);
  \draw [arrow] (poisson) -- (connect);

  \draw [arrow] (weights) -- (network);
  \draw [arrow] (network) -- (inference);

  \draw [arrow] (connect) -- (network);
  \draw [arrow] (field) -- (weights);

\end{tikzpicture}
\caption{Two-stage schematic of the proposed generative pipeline. The first row shows the geometric construction of the random neural architecture from a latent field. The second row describes the construction of weights and the supervised learning of hyperparameters.}
\label{figure0}
\end{figure}
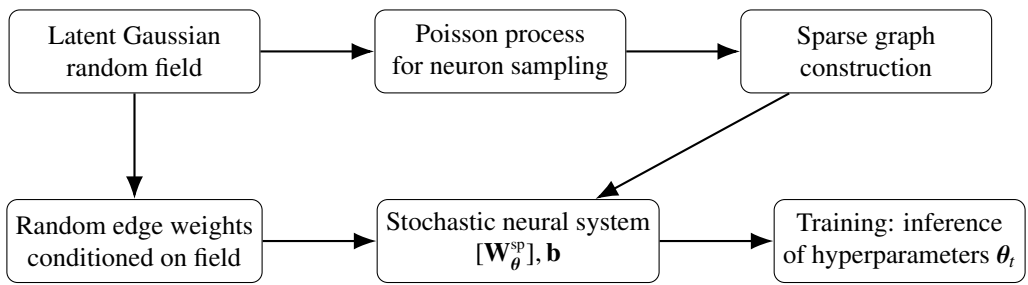

\noindent \textit{(iv) Overview of the learning problem and loss formulation}.
The learning problem is formulated as the statistical inference of the generative hyperparameter vector $\bftheta_t= (\bftheta,\bfbeta)$, where $\bftheta$ governs the construction of the stochastic neural architecture and $\bfbeta$ encodes the neuron biases. Rather than training synaptic weights directly via gradient-based optimization, as in conventional supervised learning, the proposed framework defines a conditional output probability density function $p^\ANNpp(\yy \,\vert \,\xx; \bftheta_t)$,
which emerges from the random mappings induced by realizations of the latent field. A single deterministic
{\color{black} input $\xx\in\RR^{n_\ppin}$ gives rise to a $\RR^{n_\ppout}$-valued random output}
 $\YY$ through architecture sampling, random weight generation, and forward computation.
{\color{black}
For the supervised learning, the training dataset consists of $n_d$ points $(\xx^i,\yy^i)\in\RR^{n_\ppin}\times\RR^{n_\ppout}$.
The set $\{\xx^i, i=1,\ldots,n_d\}$ denote $n_d$ independent realizations of an $\RR^{n_\ppin}$-valued random variable $\XX$ with a probability measure $P_\XX (d\xx)$.
For each input $\xx^i$, there is a single output $\yy^i \in\RR^{n_\ppout}$, corresponding to a realization of the random variable $\YY$ with values in $\RR^{n_\ppout}$.
}
The supervised learning objective is to minimize the empirical negative log-likelihood over observed input–output pairs,
$\curJ(\bftheta_t) = -\sum_{i=1}^{n_d} \log p^\ANNpp(\yy^i \,\vert \,\xx^i; \bftheta_t)$.
Since the conditional densities $p^\ANNpp(\yy \,\vert \,\xx; \bftheta_t)$ are analytically intractable, they are estimated via Monte Carlo simulation over realizations of  neural architectures generated from $\bftheta_t$.
{\color{black} This allows learning from datasets in which each input vector is associated with a single realization of the random output vector, a regime that is poorly suited to standard deterministic neural networks}.
The details of the sampling strategy, likelihood estimation, and optimization procedure are presented in Section~\ref{Section8}. This overview emphasizes the key conceptual inversion: learning is recast as likelihood-based identification of the generative hyperparameters, rather than direct optimization of weights on a fixed architecture.\\

\noindent \textit{(v) Novelty and contribution}.
This paper introduces a novel probabilistic framework for supervised learning in neural systems whose structure is not fixed, but stochastically generated by a latent Gaussian random field defined on a compact, boundaryless, multiply-connected manifold $\curS$. Unlike conventional neural networks with predefined connectivity, the proposed approach models both the topology and the random weights of the neural network as emergent properties of a hidden random field. Specifically, a Gaussian
{\color{black} random field $U(\cdot ; \bftheta)$},
parameterized by a hyperparameter vector  $\bftheta$, governs the spatial intensity of a Poisson point process over the manifold $\curS$, yielding the stochastic locations of neural nodes. From these points, a subset of $n_\pin$ input and $n_\pout$ output neurons is selected based on the
{\color{black} extreme values of $U(\cdot ; \bftheta)$}
evaluated at the Poisson-distributed locations. The connectivity graph is then constructed using both spatial proximity and dependence on the realized field. The weights are randomly generated conditionally on the field, allowing variability in the neural responses for fixed inputs. The training task consists of estimating the hyperparameters $\bftheta_t$ from supervised input-output data, using the negative log-likelihood loss function constructed from conditional Monte Carlo simulations. The originality of the proposed framework lies in the following:
\begin{itemize}
\item It defines neural architectures as realizations from a stochastic generative process on compact, boundaryless, multiply-connected manifolds, enabling principled modeling of geometric uncertainty and topology.
\item Both neuron placements and synaptic weights emerge from a reduced-order latent Gaussian random field and a Poisson spatial process, without assuming a fixed topology or connectivity.
\item The selection of input and output neurons is also stochastic, based on extremal evaluations of the latent field at Poisson-sampled locations, which further internalizes the structural randomness into the architecture.
\item Graph sparsification is not imposed heuristically but arises from percentile-based diffusion masking, modulated by geodesic distances and field-affinity scores.
\item The resulting network defines a stochastic mapping from deterministic inputs to random outputs, capturing functional uncertainty induced by the latent generative process.
\item Supervised learning is formulated as statistical inference of the hyperparameters $\bftheta_t$ that govern the architecture generator, rather than direct weight optimization on a fixed network.
\item A negative log-likelihood loss is minimized using conditional Monte Carlo estimation across architecture realizations, enabling learning {\color{black} from a dataset for which each input vector is associated with a single realization of the random output vector}.
\end{itemize}
We propose a Stochastic Graph Neural Network (GNN) whose architecture is sampled based on hyperparameter vector $\bftheta_t = (\bftheta,\bfbeta)$. Unlike standard GNNs or feedforward neural networks, the graph structure varies across simulations, enabling flexible and adaptive representation learning.\\

\noindent This formulation is particularly advantageous for implementing Physics-Informed Neural Networks (PINNs) with randomized architectures. Specifically, the latent random field defined on a manifold can serve as the stochastic germ for constructing the random fields that appear as coefficients in the partial differential equations (PDEs) governing inhomogeneous physical systems. In this way, the randomized neural network architecture becomes intrinsically aligned with the mathematical structure of the physical model underlying the inverse statistical problem to be solved from a target dataset. For example, in the case of complex systems with periodic random microstructures, a boundaryless manifold proves useful. Conversely, when the physical domain exhibits boundaries, the model incorporates a manifold with boundary, along with appropriately defined boundary conditions for the stochastic PDE that generates the latent random field. %
{\color{black} In this approach, the vector-valued stochastic germ used in the stochastic PDE of the physics model  is also used for the reduced-order latent random field of the ANN, and consequently, ensures}
that the neural network architecture is statistically consistent with, and structurally adapted to, the physical and probabilistic characteristics of the problem.\\

\noindent \textit{(vi) Organization of the paper}.
The paper is structured as follows:
\begin{itemize}
\item Section~\ref{Section2} introduces the construction of a parameterized anisotropic Gaussian random field on a compact, boundaryless, multiply-connected 2D manifold embedded in $\RR^3$. This includes the formulation of the associated SPDE, its weak form, finite element discretization, and the reduced-order representation derived from a spectral decomposition of the discretized covariance.

\item Section~\ref{Section3} presents the inhomogeneous Poisson point process used for neuron placement on the manifold. A rejection-based sampling procedure is described, guided by the latent field intensity.

\item Section~\ref{Section4} details the construction of the stochastic neural graph. Input and output neurons are selected based on field extremality, and edge connectivity is defined using a diffusion kernel modulated by geodesic distance and field-derived feature similarity. Sparsification is achieved through a percentile-thresholding strategy.

\item Section~\ref{Section5} describes the generation of random edge weights conditioned on the latent field and sparsified graph topology, yielding realizations of the random sparse weight matrix $[\bfW_\bftheta^\psp]$.

\item Section~\ref{Section6} introduces a low-rank representation for neuron biases to reduce the number of trainable parameters. While this representation is optional, it supports efficient bias modeling in large networks.

\item Section~\ref{Section7} presents the forward computation in the Stochastic Graph Neural Network (SGNN), formulated as a constrained nonlinear system involving random architecture realizations.

\item Section~\ref{Section8} defines the supervised learning problem as the identification of the generative hyperparameters $\bftheta_t$, using a Monte Carlo-based negative log-likelihood loss and a constrained optimization scheme.

\item Section~\ref{Section9} outlines the validation and testing methodology, including overfitting diagnostics, accuracy evaluation, confidence interval estimation, computation of the continuous ranked probability score (CRPS), and comparison of the conditional probability density functions of random outputs for a given deterministic input.

\item Section~\ref{Section10} is devoted to the mathematical analysis of the SGNN model, focusing on its well-posedness and statistical properties.

\item Section~\ref{Section11} provides a numerical illustration of the proposed model using a random neural architecture defined on a torus, viewed as a compact, boundaryless, multiply-connected manifold.

\item Section~\ref{Section12} concludes with a discussion of the results, limitations, and potential extensions of the proposed framework.
\end{itemize}

\noindent \textit{(vii) Important remarks on positioning and intent}.
The aim of this work is not to propose an alternative to existing fixed-topology neural architectures or to outperform current models in specific empirical settings. Rather, it is to explore a new conceptual and mathematical framework in which the neural architecture itself is a random object, a realization of a geometry-aware, field-driven generative process defined over compact boundaryless manifolds. This formulation integrates structural uncertainty at the topological level and offers a probabilistically grounded perspective where connectivity, weights, and node roles emerge jointly from a latent anisotropic field. We view this not as a refinement of current methods, but as the opening of a distinct research direction, one that invites reconsideration of the architectural priors, inductive biases, and stochastic mechanisms underpinning learning on non-Euclidean and manifold-structured domains.

While this approach shares certain high-level concerns with Bayesian machine learning—such as the modeling of uncertainty and stochastic representation, it departs from the Bayesian deep learning paradigm. In particular, the proposed framework does not involve prior distributions over neural weights, nor does it rely on posterior inference within a fixed architecture. Instead, both the architecture and the weight distribution are emergent properties of a generative process defined by a latent Gaussian random field on a manifold. The resulting randomness is structural and geometric in nature, rather than epistemic. Learning is therefore recast as the inference of the hyperparameters governing the generative field, via a likelihood-based formulation, rather than as posterior inference over parameters. This distinction underlines the originality of the framework: it establishes a stochastic foundation for learning systems in which function, geometry, and topology are jointly governed by an underlying random field, rather than externally imposed or optimized through conventional learning procedures.

This work also initiates the mathematical analysis of a stochastic neural model whose architecture is generated by a latent Gaussian random field on a compact manifold. While a complete theoretical characterization remains open, the results already establish fundamental properties, such as a preliminary analysis of the expressive variability of the induced stochastic mappings, supporting the model internal coherence and expressive potential. These initial developments lay the groundwork for a broader mathematical theory of geometry-driven stochastic learning systems.

Finally, the numerical illustration presented in this work is not intended to demonstrate superior predictive accuracy or to benchmark against existing architectures. Rather, it serves to elucidate how topological, anisotropic, and latent stochastic parameters influence the emergent neural architecture and its induced function class. These diagnostics illustrate the operational consequences of the proposed generative model, highlighting how structural uncertainty and geometric priors shape the expressive behavior of the resulting stochastic neural system.\\

\noindent {\textit{(viii) Conventions for the variables, vectors, and matrices}}.\\
\noindent $x,\eta$: lower-case Latin or Greek letters are deterministic real variables.\\
$\bfx,\bfeta$: boldface lower-case Latin or Greek letters are deterministic vectors.\\
$X$: upper-case Latin letters are real-valued random variables.\\
$\bfX$: boldface upper-case Latin letters are vector-valued random variables.\\
$[x]$: lower-case Latin letters between brackets are deterministic matrices.\\
$[\bfX]$: boldface upper-case letters between brackets are matrix-valued random variables.\\

\noindent {\textit{(ix) Algebraic notations}}.

\noindent $\NN$: set of natural numbers including $0$.\\
$\RR$: set of real numbers.\\
$\RR^m$: Euclidean vector space of dimension $m$.\\
\noindent $\MM_{m,n}$, $\MM_m$ : set of the $(m\times n)$, $(m\times m)$,  real matrices.\\
$\MM_m^+$, $\MM_m^{+0}$: set of  the positive-definite, positive,  $(m\times m)$ real matrices.\\
$[I_{m}]$: identity matrix in $\MM_m$.\\
$\langle \, \bfx ,\bfy\,\rangle$: usual Euclidean inner product of $\bfx$ and $\bfy$.\\
$\Vert \, \bfx\,\Vert$: Euclidean norm of $\bfx$ equal to $\langle \, \bfx ,\bfx\,\rangle^{1/2}$.\\
$\Vert\, [x]\, \Vert_F$: Frobenius norm of matrix  $[x]$.\\

\noindent {\textit{(x) Convention used for random variables}}.
In this paper, for any finite integer $m\geq 1$, the Euclidean space $\RR^m$ is equipped with the $\sigma$-algebra $\curB_{\RR^m}$. If $\bfY$ is a $\RR^m$-valued random variable defined on the probability space $(\Omega,\curT,\curP)$, $\bfY$ is a  mapping $\omega\mapsto\bfY(\omega)$ from $\Omega$ into $\RR^m$, measurable from $(\Omega,\curT)$ into $(\RR^m,\curB_{\RR^m})$, and $\bfY(\omega)$ is a realization (sample) of $\bfY$ for $\omega\in\Omega$. The probability distribution of $\bfY$ is the probability measure $P_\bfY(d\bfy)$ on the measurable set $(\RR^m,\curB_{\RR^m})$ (we will simply say on $\RR^m$). The Lebesgue measure on $\RR^m$ is noted $d\bfy$ and when $P_\bfY(d\bfy)$ is written as $p_\bfY(\bfy)\, d\bfy$, $p_\bfY$ is the probability density function (PDF) on $\RR^m$ of $P_\bfY(d\bfy)$ with respect to $d\bfy$. We denotes by $E$ the mathematical expectation operator that is such that $E\{\bfY\} = \int_{\RR^m} \bfy \, P_\bfY(d\bfy)$. Random variable $\bfY$ is a second-order random variable if $E\{\Vert\bfY\Vert^2\} < +\infty$. The set of all the second-order random variables defined on $(\Omega,\curT,\curP)$ with values in $\RR^m$ is a Hilbert space denoted as $L^2(\Omega,\RR^m)$, equipped with the inner product
$\langle \bfY ,\bfY'\rangle_\Omega  = E\{ \langle \bfY ,\bfY'\rangle \}$ and the associated norm  $\Vert\bfY\Vert_\Omega = \langle \bfY ,\bfY \rangle_\Omega^{1/2}$.
Finally, a normalized Gaussian random vector $\bfV$ is a vector whose components are all statistically independent, with each component being Gaussian, centered, and having unit variance.

\section{Modeling of a parameterized anisotropic latent Gaussian random field on a compact, boundaryless, multiply-connected 2-dimensional manifold  immersed in 3D Euclidean space, and its reduced-order representation-based random generator}
\label{Section2}
The construction of a Gaussian random field on a compact, boundaryless, multiply-connected manifold $\curS$ of dimension 2 immersed in the Euclidean space of dimension 3 is based on the Karhunen-Lo\`eve expansion of random fields. However, the explicit construction of such a random field on the manifold $\curS$ requires not only truncating the expansion, but also constructing a manageable parametrization of the manifold with a finite (though large) number of parameters, and defining, over this parameterized representation of the manifold, the mean function and the covariance operator of the random field. We thus propose an efficient construction based on the following steps:
\begin{itemize}
\item Use the finite element method to construct an approximation $\curS_h$ of the multiply-connected manifold $\curS$,  using 3-node isoparametric triangles \cite{Zienkiewicz2005}. The parameters consist of all the 3D coordinates of the $n_o$ finite element nodes of the mesh. This yields a geometrical approximation $\curS_h$ of the manifold $\curS$ and allows the construction of the finite element basis  $\{\varphi_p(\bfx),\bfx\in\curS_h\}_{p=1}^{n_o}$ for representing any sufficiently smooth function defined on $\curS_h$.
\item Define the centered, second-order, Gaussian,  real-valued random field $U$ on the manifold $\curS$,  represented on $\curS_h$ using a relatively small number of hyperparameters. For this  Gaussian random field $U$ indexed by $\curS$, following the approach  in  \cite{Lindgren2011,Lindgren2022}, we define a stochastic linear boundary value problem on $\curS$, whose nonhogeneous source term is a Gaussian white generalized random field $\curB$ indexed by $\curS$. The coefficients of the linear differential operator depend on the location on $\curS$ and on hyperparameters that control the anisotropy of the random field. The unique  solution of this stochastic problem is the centered generalized Gaussian random field, represented by a classical centered, second-order, Gaussian, real-valued random field $U$ indexed by $\curS$ (see \cite{Kree1986}).
\item Construct the weak formulation of the stochastic boundary value problem and perform its spatial discretization  using the finite element model $\curS_h$ of the manifold $\curS$ \cite{Dautray2013}. This yields a finite-dimensional stochastic representation $U(\bfx) = \sum_{p=1}^{n_o} U_{p}\, \varphi_p(\bfx)$ for $\bfx\in\curS_h$, which satisfies a linear random matrix equation of the form $(\tau_0 \, [g]  + [\kappa_{n_h}]) \,\bfU = \bfB$,
    {\color{black} where $\bfU = (U_{1},\ldots, U_{n_o})$ and $\bfB$ is a centered, second-order, Gaussian, $\RR^{n_o}$-valued random variable}.
\item Express the covariance matrix $[C_\bfU]$ of the $\RR^{n_o}$-valued random variable $\bfU$, which is the solution of the above linear random matrix equation.
\item To reduce the stochastic dimension of the representation of the random field $U$, construct  a reduced-order stochastic model of $\bfU$ with dimension $m\ll n_o$, using a classical principal component analysis. This involves solving the eigenvalue problem for the matrix $[C_\bfU]$, retaining the $m$ largest eigenvalues (represented by the diagonal matrix $[\lambda^m]\in\MM_m^+$) and their associated eigenvectors (represented by the matrix $[\Phi^m] \in\MM_{n_o,m}$).
\item The final representation $\{U^m(\bfx), \bfx\in\curS_h\}$ of the random field $U$ on $\curS_h$ is given by
   $U^m(\bfx) = \langle \, \bfpsi^m(\bfx),\bfH\,\rangle$, where $\bfpsi^m(\bfx) = [\lambda^m]^{1/2} [\Phi^m]^T\, \bfvarphi(\bfx)$,
   {\color{black} $\bfvarphi(\bfx) = (\varphi_1(\bfx),\ldots ,\varphi_{n_o}(\bfx) )$},
   and $\bfH$ is the normalized Gaussian $\RR^m$-valued random variable with mean $E\{\bfH\} = \bfzero_m$ and covariance matrix
   $E\{\bfH\otimes\bfH\} = [I_m]$.
\end{itemize}
\subsection{Construction of a parameterized anisotropic latent Gaussian random field and its generator}
\label{Section2.1}
%
\noindent\textit{(i) Geometry}.
Let $\bfx = (x_1, x_2, x_3)$ denote a point in the Cartesian coordinate system of $\RR^3$.
We consider a smooth, compact, boundaryless, multiply-connected 2-dimensional manifold $\curS \subset \RR^3$, oriented by its outward unit {\color{black} normal $\bfn(\bfx)$}
 to $\curS$, and immersed in the Euclidean space $\RR^3$.\\

\noindent\textit{(ii) Gaussian white generalized random field}.
Let $d\sigma(\bfx)$ be the intrinsic surface measure on $\curS$.
Let $\{\curB(\bfx),\bfx\in\curS\}$ be a real-valued, centered,  Gaussian white generalized random field defined on the probability space $(\Omega,\curT,\curP)$ and indexed by $\curS$. Its covariance $C_\curB$ is a generalized function \cite{Guelfand1964,Kree1986}
such that, for all real-valued continuous function $f$  on $\curS$,
\begin{equation}\label{eq2.0}
\langle C_\curB , f\otimes f\rangle = \Vert \, f\,\Vert^2_{L^2(\curS)} \, ,
\end{equation}
which can, formally, be written as
$\int_\curS\int_\curS C_\curB(\bfx,\bfx')\,f(\bfx)\, f(\bfx')\, d\sigma(\bfx)
\, d\sigma(\bfx') = \int_\curS f(\bfx)^2\, d\sigma(\bfx)$.\\

\noindent\textit{(iii) Stochastic partial differential equation}.
Let $\{U(\bfx),\bfx\in\curS\}$ be the centered real-valued random field defined on the probability space $(\Omega,\curT,\curP)$ and indexed by $\curS$, such that
\begin{equation} \label{eq2.1}
E\{U(\bfx)\} =  0 \quad , \quad \forall \bfx\in\curS   \, .
\end{equation}
The random field $U$ on $\curS$ is constructed as the solution of the stochastic partial differential equation,
\begin{equation} \label{eq2.2}
(\tau_0 - \langle \,\nabla  , [K_{n_h}(\bfx)]\, \nabla\, \rangle ) \, U(\bfx)
                               = \curB(\bfx)  \quad , \quad \forall \bfx\in\curS  \, .
\end{equation}
In Eq.~\eqref{eq2.2}, $\tau_0 > 0$ is fixed, and $\bfx \mapsto [K_{n_h}(\bfx)]$ is a function from $\curS$ into $\MM_3^{+0}$.
Let $\curL_{n_h}$ be the operator defined by
\begin{equation} \label{eq2.2a}
\curL_{n_h} = \tau_0 -\langle \nabla , [K_{n_h}(\bfx)] \, \nabla \rangle\, ,
\end{equation}
where $\nabla$ denotes the surface gradient acting tangentially to $\curS$. At each point $\bfx \in \curS$,  the matrix $[K_{n_h}(\bfx)]$, when projected onto the tangent space $T_\bfx \curS$, defines a symmetric, positive-definite endomorphism in $\mathrm{End}(T_\bfx \curS)$.
{\color{black} In addition, the construction of the field $[K_{n_h}]$ is carried out in order that, there exit two constants $k_\star$ and $k^\star$ with $0 < k_\star \leq k^\star < +\infty$ such that
\begin{equation} \label{eq2.3}
k_\star\, \Vert\bfu\Vert^2 \, \leq \langle \,[K_{n_h}(\bfx)]\, \bfu , \bfu \, \rangle  \leq \, k^\star\, \Vert\bfu\Vert^2
                                \quad , \quad \forall \bfu\in T_\bfx\curS \quad , \quad \forall\bfx\in\curS  \, .
\end{equation}
{\color{black} These bounds} guarantee the uniform ellipticity of the operator $\curL_{n_h}$}.\\

\noindent\textit{Remark 1}. We consider a centered random field, as a non-zero mean function would not influence the proposed hyperparameterized probabilistic model of the random neural architecture, and thus would have no effect on the supervised learning process used to estimate the optimal hyperparameter value controlling the neural network.\\

\noindent\textit{(iv) Anisotropic $\curS$-manifold model}.
Let $\bfe^1(\bfx), \bfe^2(\bfx), \bfe^3(\bfx)$ be a direct local orthonormal basis at the point $\bfx \in \curS$, such that $\bfe^3(\bfx) = \bfn(\bfx)$. Consequently, $\bfe^1(\bfx)$ and $\bfe^2(\bfx)$ belong to the tangent space $T_\bfx \curS$. Let $h^{(1)}_{n_h}(\bfx)$ and $h^{(2)}_{n_h}(\bfx)$ be positive real-valued fields on $\curS$. We then define the diagonal tensor field $[K_{n_h}]$ by
\begin{equation}\label{eq2.4}
[K_{n_h}(\bfx)] = h^{(1)}_{n_h}(\bfx)\, \bfe^1(\bfx) \otimes \bfe^1(\bfx) + h^{(2)}_{n_h}(\bfx)\,\, \bfe^2(\bfx) \otimes \bfe^2(\bfx)
                        \quad , \quad \forall\, \bfx \in \curS \, ,
\end{equation}
where, for $k\in \{1,2\}$,  the field $h^{(k)}_{n_h}$  is an $n_h$-finite representation of a real-valued field on $\curS$.
It can be seen that $[K_{n_h}(\bfx)] \in \MM_3^{+0}$. When restricted to the tangent space $\mathrm{End}(T_\bfx \curS)$, it defines a symmetric, positive-definite endomorphism of $T_\bfx \curS$, that is, $[K_{n_h}(\bfx)]\, \vert \, {T_\bfx \curS} \in \mathrm{End}(T_\bfx \curS)$.
As such, it defines an anisotropic Riemannian metric on $\curS$, since the restriction of $[K_{n_h}(\bfx)]$ to $T_\bfx \curS$ is positive-definite. Consequently, the differential operator $\langle \nabla , [K_{n_h}(\bfx)] \nabla \rangle$ becomes intrinsic to the manifold $\curS$, that is, invariant under isometries and independent of the immersion into $\mathbb{R}^3$. For $k\in\{1,2\}$ and
{\color{black} for given $n_h \geq 1$,}
 a low-rank representation of $h^{(k)}_{n_h}(\bfx)$ is introduced and written as
{\color{black}
\begin{equation}\label{eq2.5}
0 < c_\star^{(k)} \,\leq \, h^{(k)}_{n_h}(\bfx) = h_k + \sum_{j=1}^{n_h} \beta_j^{\,(k)}\,\hh^{(k,j)}(\bfx) \, \leq \,  c^{(k)\star}\quad , \quad \forall \, \bfx \in \curS \, ,
\end{equation}
where $h_k > 0$ and $0 < c_\star^{(k)} \,\leq  c^{(k)\star}$, which implies that Eq.~\eqref{eq2.3} holds, and where
}
\begin{equation}\label{eq2.6}
\int_\curS \hh^{(k,j)}(\bfx)\, \hh^{(k,j')}(\bfx) \, d\sigma(\bfx) = \delta_{jj'}
\quad , \quad
\int_\curS \hh^{(k,j)}(\bfx)\,  d\sigma(\bfx) = 0 \, .
\end{equation}
For each given $h_k > 0$, the hyperparameters $\{\beta_j^{\,(k)}\}_j$ belong to an admissible set
{\color{black}
$\curC_{\adp ,k}$ ensuring that $h^{(k)}_{n_h}(\bfx)$ satisfies Eq.~\eqref{eq2.5}
}
for all $\bfx\in\curS$. We then have
\begin{equation}\label{eq2.7}
\bfbeta^{\,(k)}_{n_h} = (\beta_1^{\,(k)},\ldots , \beta_{n_h}^{\,(k)})\, \in \, \curC_{\adp ,k} \, \subset \, \RR^{n_h}
\quad , \quad k\in \{1,2\}\, .
\end{equation}
Note that for $n_h=0$, we have $h^{(k)}_{n_h}(\bfx) = h_k$, a positive constant.\\

\noindent\textit{(v) Weak formulation and existence of a unique second-order Gaussian random field}.
Let $\HH^1(\curS)\subset L^2(\curS)$ denote the Sobolev space of order $1$, equipped with the surface measure $d\sigma$. The weak formulation
{\color{black} of Eq.~\eqref{eq2.2} consists in finding $U$ in $L^2_\curP(\Omega,\HH^1(\curS))$, such that,  for all $u \in \HH^1(\curS)$
},
\begin{equation} \label{eq2.8}
\int_\curS \tau_0\, U(\bfx)\, u(\bfx) \, d\sigma(\bfx) + \int_\curS \langle\, [K_{n_h}(\bfx)]\, \nabla U(\bfx) \, , \nabla u(\bfx) \,\rangle \, d\sigma(\bfx) = B(u) \quad , \quad B(u) = \int_\curS \curB(\bfx)\, u(\bfx)\, d\sigma(\bfx) \, .
\end{equation}
It can be seen (see Eq.~\eqref{eq2.0}) that $B(u)$ is a centered, second-order, Gaussian, real-valued random variable such that
$E\{B(u)^2\} = \int_\curS u(\bfx)^2\, d\sigma(\bfx) < +\infty$. In Eq.~\eqref{eq2.8}, the partial derivatives $\partial U(\bfx)/\partial x_i$ are understood in the mean-square sense.
Under the assumptions on $\tau_0$ and the field $[K_{n_h}]$, it can be shown (using, for instance, \cite{Dautray2013,Kree1986}) that the random field $U$ indexed by $\curS$ is a centered, second-order, Gaussian, real-valued  random field, meaning that $E\{U(\bfx)^2\} < +\infty$ for all $\bfx\in\curS$.\\

\noindent \textit{(vi) Construction of an approximation by the finite element method}.
As previously explained, the manifold $\curS$ is meshed with $3$-node finite elements (triangles) with linear interpolation functions in the isoparametric finite element framework \cite{Zienkiewicz2005} yielding the geometrical approximation $\curS_h$ of $\curS$. The total number of finite element nodes is denoted by $n_o$.
Let $\{\varphi_p, p=1,\ldots , n_o\}$ be the corresponding finite element basis. Note that this functions are continuous on $\curS_h$. Keeping the same notation for the finite element approximation of $U$, we have
\begin{equation}\label{eq2.9}
U(\bfx) = \sum_{p=1}^{n_o} U_{p}\, \varphi_p(\bfx) \quad , \quad \forall\, \bfx\in\curS_h \, .
\end{equation}
Substituting $U$ defined by Eq.~\eqref{eq2.9} into Eq.~\eqref{eq2.8} (in which $\curS$ is replaced by $\curS_h$) and taking for $u$ the basis functions $\{\varphi_p\}_p$, we obtain the $\RR^{n_o}$-valued  stochastic equation,
\begin{equation}\label{eq2.10}
( \tau_0\,[g] + [\kappa_{n_h}] ) \, \bfU = \bfB \quad , \quad \bfU = (U_{1},\ldots , U_{n_o})
                    \quad , \quad \bfB = (B_1,\ldots , B_{n_o}) \, ,
\end{equation}
where $\bfU$ and $\bfB$  are a centered, second-order, Gaussian, $\RR^{n_o}$-valued random variables, and where the matrix $[g]\in \MM_{n_o}^+$, the matrix $[\kappa_{n_h}]\in \MM_{n_o}^{+}$, and the $\RR^{n_o}$-valued random variable $\bfB$ are written, for all $p$ and $q$ in $\{1,\ldots , n_o\}$, as
\begin{equation}\label{eq2.11}
[g]_{qp} = \! \int_{\curS_h}\! \varphi_p(\bfx)\, \varphi_q(\bfx) \, d\sigma(\bfx) \,\,\, , \,\,\,
[\kappa_{n_h}]_{qp} = \! \int_{\curS_h} \!\langle \, [K_{n_h}(\bfx)]\, \nabla \varphi_p(\bfx) , \nabla\varphi_q(\bfx)\, \rangle \, d\sigma(\bfx) \,\,\, , \,\,\, B_q = \! \int_{\curS_h} \!\curB(\bfx)\, \varphi_q(\bfx)\, d\sigma(\bfx)\, .
\end{equation}
With the finite element method, $[g]$ and $[\kappa_{n_h}]$ are sparse matrices.
The sparse positive-definite symmetric real matrix $\tau_0\,[g] + [\kappa_{n_h}] \in \MM_{n_o}^+$  is thus invertible, and the random field $U$, defined by Eq.~\eqref{eq2.6}, can be rewritten, for all $\bfx\in\curS_h$,  as
\begin{equation}\label{eq2.12}
U(\bfx) = \langle \, \bfU , \bfvarphi(\bfx)\, \rangle  \quad, \quad
\bfU = (\tau_0\,[g] + [\kappa_{n_h}])^{-1}\, \bfB \quad , \quad
\bfvarphi(\bfx) = (\varphi_1(\bfx), \ldots , \varphi_{n_o}(\bfx))\in\RR^{n_o} \, .
\end{equation}
%

\noindent \textit{(vii) Covariance matrix of $\bfU$}.
Under the introduced hypotheses, $\bfU$ is a centered, second-order,  $\RR^{n_o}$-valued random variable, whose covariance matrix  $[C_{\bfU}] = E\{\bfU \otimes \bfU\} \in \MM_{n_o}^+$ is invertible.
{\color{black} Using Eqs.~\eqref{eq2.11} and \eqref{eq2.12}}
 yields
\begin{equation}\label{eq2.13}
[C_{\bfU}] = (\tau_0 [g] + [\kappa_{n_h}])^{-1} \, [g] \, (\tau_0 [g] + [\kappa_{n_h}])^{-T} \, .
\end{equation}
Note that $(\tau_0 [g] + [\kappa_{n_h}])^{-T} = (\tau_0 [g] + [\kappa_{n_h}])^{-1}$, and since $[g]$ is invertible, we can write
$[C_{\bfU}]^{-1} = (\tau_0 [g] + [\kappa_{n_h}]) \, [g]^{-1} \, (\tau_0 [g] + [\kappa_{n_h}])$. At this stage of the model construction, the two matrices $[C_{\bfU}]$ and $[C_{\bfU}]^{-1}$ appear as symmetric, full $n_o\times n_o$ real matrices,
{\color{black} depending on $\tau_0$, $h_1$, $h_2$, $\bfbeta^{\,(1)}_{n_h}$, and $\bfbeta^{\,(2)}_{n_h}$}.\\

\noindent \textit{(viii) Eigenvalue problem, reduced-order representation $U^m$ of the random field $U$, and its generator}.
Let $\lambda_1 \geq \lambda_2 \geq \ldots \geq \lambda_m > 0$ be the $m \leq n_o$ largest eigenvalues of $[C_{\bfU}]$, ordered in decreasing order, and let $\bfPhi^1,\ldots, \bfPhi^m$ be the associated orthonormal eigenvectors. Define $[\lambda^m]\in\MM_m^+$ as the diagonal matrix such that $[\lambda^m]_{\alpha\beta} = \lambda_\alpha\, \delta_{\alpha\beta}$, and let $[\Phi^m] = [\bfPhi^1 \ldots \bfPhi^m] \in \MM_{n_o,m}$ be the rectangular matrix whose columns are these eigenvectors. Then we have
\begin{equation}\label{eq2.14}
[C_{\bfU}]\, [\Phi^m] =  [\Phi^m] \, [\lambda^m]\quad , \quad [\Phi^m]^T\, [\Phi^m] = [I_m] \, .
\end{equation}
For $m \leq n_o - 1$, let  $\bfU^m$ be the reduced-order representation of the centered random vector $\bfU$ defined by
\begin{equation}\label{eq2.15}
\bfU^m = [\Phi^m]\, [\lambda^m]^{1/2} \, \bfH  \, ,
\end{equation}
where $\bfH=(H_1,\ldots,H_m)$ is a centered, normalized second-order, Gaussian, $\RR^m$-valued random variable, with independent components,
\begin{equation}\label{eq2.16}
E\{\bfH\} = \bfzero_m \quad , \quad E\{\bfH\otimes\bfH\} = [I_m]\, .
\end{equation}
Note that $\bfU^m$ depends on the hyperparameters $h_1$, $h_2$, $\bfbeta^{\,(1)}_{n_h}$, and $\bfbeta^{\,(2)}_{n_h}$ (we recall that $\tau_0$
{\color{black} is prescribed  and is not considered as a hyperparameter). The model is defined by Eqs.~\eqref{eq2.15} and \eqref{eq2.16}, where $m$ is considered as a fixed parameter of the model rather than as a}
hyperparameter to be identified in the context of supervised learning.
Indeed, fixing $m$ amounts to specifying the model of the Gaussian field. A larger fixed value of $m$ introduces more statistical fluctuations across different scales in the field, which can help increase the variability of the network random weights. In summary, the goal is not to construct a "good" approximation of $\bfU$, but rather to control the stochastic model of the Gaussian field,  in part through the choice of $m$. Nevertheless, computing the projection error is useful, as it quantifies the level of "filtering" induced by the selected value of $m$.
Once the parameter $m$ is fixed such that $1\leq m \ll n_o$, the approximation error can then be computed by
\begin{equation}\label{eq2.17}
\err_\pPCA(m) = \frac{E\{\Vert\, \bfU - \bfU^m\,\Vert^2\}}{ E\{\Vert\, \bfU\,\Vert^2\}} =
 1 - \frac{\sum_{\alpha=1}^m \lambda_\alpha}{\tr\{[C_{\bfU}]\}} \quad , \quad  m < n_o -1  \, .
\end{equation}
Using Eqs.~\eqref{eq2.12} and \eqref{eq2.15} yields the reduced-order representation $U^m$ of the random field $U$,
\begin{equation} \label{eq2.18}
U^m(\bfx) =  \langle \, \bfpsi^m(\bfx) , \bfH\, \rangle \quad , \quad  \forall\, \bfx\in\curS_h\, ,
\end{equation}
where  $\bfpsi^m(\bfx)=(\psi^m_1(\bfx),\ldots, \psi^m_m(\bfx))\in\RR^m$
{\color{black} is written}
as
\begin{equation} \label{eq2.19}
 \bfpsi^m(\bfx) = [\lambda^m]^{1/2} [\Phi^m]^T\, \bfvarphi(\bfx)\, .
\end{equation}
Note that since $\bfvarphi \in C^0(\curS_h,\RR^{n_o})$, we have $\bfpsi^m \in C^0(\curS_h,\RR^{m})$.
Let $\{\bfeta^\ell ,\ell=1,\ldots, n_\psim\}$ be $n_\psim$ independent realizations of $\bfH$.
From Eq.~\eqref{eq2.18}, the realizations $\{ u^{m,\ell}(\bfx),\bfx\in\curS_h\}_{\ell=1}^{n_\psim}$ of $\{U^m(\bfx),\bfx\in\curS_h\}$ are explicitly computed by
{\color{black}
\begin{equation}\label{eq2.20}
u^{m,\ell}(\bfx) = \langle \, \bfpsi^m(\bfx) , \bfeta^\ell\, \rangle \quad , \quad  \forall\, \bfx\in\curS_h\, .
\end{equation}
}
Finally, considering that $m$, $n_h \geq 0 $, and $\tau_0 > 0$ are fixed independently of $\bftheta$, the parameterization of the probabilistic model proposed for the random field $U^m$ involves the hyperparameters $h_1$, $h_2$, and, if $n_h > 0$, the vectors
$\bfbeta^{\,(1)}_{n_h}\in \curC_{\adp ,\, 1}$ and $\bfbeta^{\,(2)}_{n_h}\in \curC_{\adp ,\, 2}$, where $\curC_{\adp ,\, k} \subset\RR^{n_h}$ is defined by Eq.~\eqref{eq2.7}
{\color{black} (see Eq.~\eqref{eq8.17})}.
\subsection{Numerical strategies for computing the reduced-order representation $U_c^m$ of the random field $U_c$}
\label{Section2.2}
To construct the reduced-order representation $U^m$ (defined by Eq.~\eqref{eq2.18}) of the random field $U$, we need to compute the dominant eigenspace of dimension $m$ of the full covariance matrix $[C_{\bfU}] \in \MM^+_{n_o}$ (see Eq.~\eqref{eq2.14}), which is defined by Eq.~\eqref{eq2.13}.\\

\noindent \textit{(i) Direct computation}. If the $m$-dimensional dominant eigenspace of the full matrix $[C_{\bfU}]$ can be directly computed, the matrices $[\Phi^m]$ and $[\lambda^m]$ used in the reduced-order representation $U^m(\bfx)$ (defined by Eq.~\eqref{eq2.18}) are obtained.\\

\noindent \textit{(ii) An alternative attempt}. If method (a) is too computationally expensive, one may consider computing the eigenspace associated with the smallest eigenvalues of
{\color{black} $[C_{\bfU}]^{-1} = (\tau_0 [g] + [\kappa])^{T} \, [g]^{-1} \, (\tau_0 [g] + [\kappa])$}
and then deducing the dominant eigenspace of $[C_{\bfU}]$.
{\color{black} Unfortunately, although $[g]$ is a sparse matrix, $[g]^{-1}$ is a full matrix. Thus there is no computational advantage to using this formulation, and alternatives  must be considered.}\\

\noindent \textit{(iii) A scalable alternative that avoids memory limitations}. It consists in constructing a random generator of $\bfU$ and then to use the generated realizations for computing the matrices $[\Phi^m]$ and $[\lambda^m]$.
Such a generator of $\bfU$ is based on the equation  $\bfU = (\tau_0\,[g] + [\kappa])^{-1}\, \bfB$
{\color{black} from Eq.~\eqref{eq2.10}},
where $\bfB$ is a centered, second-order, Gaussian random variable with values in $\RR^{n_o}$, whose covariance matrix is the sparse matrix $[g] \in \MM^+_{n_o}$.\\

\noindent (iii.a) If the Cholesky factorization $[g] = [L_g]^T \, [L_g]$, with $[L_g]$ an upper triangular matrix in $\MM_{n_o}$, can be computed, then the random vector $\bfB$ can be written as $\bfB = [L_g]^T\, \bfGamma$, where $\bfGamma$ is a normalized Gaussian $\RR^{n_o}$-valued random variable ($E\{\bfGamma\} = \bfzero_{n_o}$ and $E\{\bfGamma\otimes\bfGamma\} = [I_{n_o}]$). Let $\{\bfgamma^\ell,\ \ell = 1,\ldots, n_\psim\}$ be $n_\psim$ independent realizations of $\bfGamma$, grouped columnwise in the rectangular matrix $[\gamma] \in \MM_{n_o,n_\psim}$.
Let $[u] \in \MM_{n_o,n_\psim}$ be the rectangular matrix that groups columnwise the $n_\psim$ independent realizations $\{\bfu^\ell,\ \ell = 1,\ldots, n_\psim\}$ of $\bfU$, such that $\bfu^\ell = (\tau_0\,[g] + [\kappa])^{-1}\, [L_g]^T\, \bfgamma^\ell$. Then the matrix $[u]$ is computed by solving the linear system of equations with a sparse matrix,
\begin{equation} \label{eq2.21}
(\tau_0\,[g] + [\kappa]) \, [u] = [L_g]^T\,[\gamma] \, .
\end{equation}

\noindent (iii.b) If the Cholesky factorization of $[g]$ is computationally too expensive, a lumped approximation (commonly used with the finite element method in structural dynamics \cite{Bathe1976}) can be employed to construct a diagonal approximation $[g^\pdiag] \in \MM^+_{n_o}$ of $[g] \in \MM^+_{n_o}$, such that $[g^\pdiag]_{ij} = \delta_{ij}\, \sum_{j'=1}^{n_o} [g]_{ij'}$. Therefore, the diagonal entry $(i,i)$ of $[g^\pdiag]^{1/2}$ is $\sqrt{[g]_{ii}}$. In that case, $\bfB = [g^\pdiag]^{1/2}\, \bfGamma$, and Eq.~\eqref{eq2.21} is replaced by

\begin{equation} \label{eq2.22}
(\tau_0\,[g] + [\kappa]) \, [u] = [g^\pdiag]^{1/2}\,[\gamma] \, .
\end{equation}

\noindent (iii.c) Let $\underline{\hat\bfu} = \frac{1}{n_\psim} \sum_{\ell=1}^{n_\psim}\bfu^\ell$ be the estimate of $\underline \bfu = \bfzero_{n_o}$ and let $[\hat u]\in\MM_{n_o,n_\psim}$ be the matrix formed  by the centered realizations $\{\bfu^\ell - \underline{\hat\bfu} \, ,  \ell = 1,\ldots, n_\psim\}$ as columns. Consequently,
$\frac{1}{n_\psim - 1} [\hat u]\, [\hat u]^T$ is the estimate of $[C_{\bfU}]$.
Let $[\hat\Phi^m]\,[\hat S^m]\,[\hat\Phi^m]^T = [\hat u]$ be the truncated singular value decomposition (SVD) of $[\hat u]\in \MM_{n_o,n_\psim}$, where  the $m$ largest singular values are retained (\cite{Golub1993}). The diagonal entries of $[\hat S^m]$ are the singular values $\hat S_1\geq  \ldots \geq  \hat S_m > 0$, in  decreasing order. Thus, $[\hat\Phi^m]\in \MM_{n_o,m}$ satisfies $[\hat\Phi^m]^T [\hat\Phi^m] = [I_m]$. We then define the diagonal matrix $[\hat\lambda^m] = \frac{1}{n_\psim -1} \, [\hat S^m]^2$. Hence, {\color{black} $[\hat\Phi^m]$ and $[\hat\lambda^m]$ are approximations of $[\Phi^m]$ and $[\lambda^m]$,}
 respectively,  with convergence as $n_\psim$ goes to infinity.\\
\section{Generation of neurons via an inhomogeneous Poisson process}
\label{Section3}
{\color{black} In this section, the hyperparameter $\bftheta =(h_1,h_2,\zeta_s,\bfbeta^{\,(1)}_{n_h},\bfbeta^{\,(2)}_{n_h})$, defined in Section~\ref{Section1}-(iii) (where $\zeta_s$ is defined in Section~\ref{Section5}), is fixed}.\\

\noindent\textit{(i) Principle of the construction}.
For all $\bfx$ fixed in $\curS_h$, let $\Lambda(\bfx;\bftheta)$ be the variance of the real-valued random variable $U^m(\bfx;\bftheta)$.
Taking into account Eqs.~\eqref{eq2.16}, \eqref{eq2.19}, and \eqref{eq2.18}, $\Lambda(\bfx;\bftheta)$ is written as
\begin{equation}\label{eq3.1}
\Lambda(\bfx;\bftheta) = \Vert \bfpsi^m(\bfx;\bftheta)\Vert^2
               \quad , \quad \bfpsi^m(\bfx;\bftheta) = [\lambda^m(\bftheta)]^{1/2} [\Phi^m(\bftheta)]^T\, \bfvarphi(\bfx) \in\RR^m\, .
\end{equation}
To generate the positions $\{\bfx^i_\bftheta, i=1,\ldots , N\}$ of $N$ neurons on the manifold $\curS_h$, we employ a spatial point process that adapts to the local variability of the latent Gaussian field $\{U^m(\bfx;\bftheta),\bfx\in\curS_h\}$.
It should be noted that the variance $\Lambda(\bfx;\bftheta)$ may vanish at certain points on the manifold $\curS_h$,
especially when the reduced-order dimension $m$ is small. In such regions, the intensity function of the inhomogeneous Poisson process becomes zero, and no neurons will be sampled there. This behavior is consistent with the probabilistic structure of the model, in which neuron allocation is concentrated in areas of higher uncertainty (i.e., higher local variance). However, if large parts of $\curS_h$ exhibit vanishing variance, this may result in under-sampling or complete omission of those regions from the neural graph. In practice, this effect is mitigated by ensuring that the reduced-order dimension $m$ captures enough variability of the latent field over $\curS$.
We thus define an inhomogeneous Poisson point process $\PP_\Lambda$ on $\curS_h$ with intensity $\Lambda(\bfx;\bftheta) \, d\sigma(\bfx)$, where $d\sigma(\bfx)$ denotes the surface measure on the manifold $\curS_h$ induced by its embedding in $\RR^3$.
This means that, for any measurable subset $A \subset \curS_h$, the random number $\curN_A(\bftheta)$ of points falling in $A$ is a Poisson random variable with mean value
\begin{equation}\label{eq3.2}
E\{\curN_A(\bftheta)\} = \int_A \Lambda(\bfx;\bftheta)\, d\sigma(\bfx)\, .
\end{equation}
{\color{black} Conditional on the event $\curN_{\!\curS_h}(\bftheta) = N$,}
 the locations of the $N$ points are independently and identically distributed over $\curS_h$, with probability density function $p(\bfx;\bftheta)$ with respect to the measure $d\sigma(\bfx)$, given by
\begin{equation}\label{eq3.3}
p(\bfx;\bftheta) = \frac{\Lambda(\bfx;\bftheta)}{\int_{\curS_h} \Lambda(\bfy;\bftheta)\, d\sigma(\bfy)} \quad , \quad \bfx \in \curS_h\,.
\end{equation}
The set of accepted points $\{\bfx^i_\bftheta, i=1,\ldots , N\} \subset \curS_h$ constitutes the $N$ neurons of the graph,
\begin{equation}\label{eq3.4}
\curV_\bftheta = \{ \bfx^1_\bftheta, \bfx^2_\bftheta, \dots, \bfx^N_\bftheta \} \quad , \quad \bfx^i_\bftheta \sim \PP_\Lambda \, .
\end{equation}
In practice, the realization of $\PP_\Lambda$ is approximated numerically by the thinning procedure described below.
With this construction, the higher the local variance $\Lambda(\bfx;\bftheta)$, the more likely a neuron is placed at $\bfx$, thus allocating model capacity where the field is more uncertain or varying.
This construction follows standard properties of inhomogeneous Poisson point processes, including the conditional i.i.d. distribution of points given their number \cite{Daley2003,Baddeley2015}.\\

\noindent\textit{(ii) Sampling from an inhomogeneous Poisson process via rejection method}.
To construct a sample of neurons according to the inhomogeneous Poisson point process $\PP_\Lambda$ defined on the manifold $\curS_h$, we employ a
{\color{black} rejection-type (approximate thinning)} algorithm \cite{Lewis1979,Robert2004}, which is particularly suitable for spatial domains discretized by finite elements. The method enables efficient simulation of a realization of $\PP_\Lambda$ by filtering candidate points sampled from a proposal distribution over $\curS_h$ based on the target intensity function $\Lambda(\bfx;\bftheta)$. This process is fully compatible with the mesh-based representation of $\curS_h$ described in the geometric framework.
First, we generate a large set of $M$ candidate points $\{\bfz^j, j=1,\ldots, M\}\subset \curS_h$ using a proposal probability measure $\pi(\bfx)\, d\sigma(\bfx)$. A typical and computationally efficient choice for $\pi(\bfx)$ is the piecewise uniform distribution over the triangulated mesh of $\curS$, with each triangle contributing proportionally to its surface area. Such sampling can be performed using barycentric coordinates and local area weighting over each triangular element, following standard techniques in mesh-based Monte Carlo integration.
Then, for each candidate point $\bfz^j$, the normalized intensity value $\Lambda(\bfz^j;\bftheta)$ is computed using Eq.~\eqref{eq3.1}, which reflects the local variance of the latent Gaussian field $\{U^m(\bfx;\bftheta),\bfx\in\curS_h\}$ at that location. To control the thinning acceptance rates, we compute the maximum observed value over the sample,
\begin{equation}\label{eq3.5}
\Lambda_\pmax(\bftheta) = \max_{j=1,\ldots , M} \Lambda(\bfz^j;\bftheta) \, .
\end{equation}
The value $\Lambda_\pmax(\bftheta)$ is an empirical estimate obtained from the finite set of candidate points, and may
{\color{black} overestimate}
 the true supremum of $\Lambda(\bfx;\bftheta)$ over $\curS_h$; such {\color{black} overestimation} can reduce the efficiency of
 {\color{black} the rejection step} by decreasing the expected number of accepted neurons.
Each candidate point $\bfz^j$ is then retained (i.e., selected as a neuron) with  probability
\begin{equation}\label{eq3.6}
p_j(\bftheta) = \frac{\Lambda(\bfz^j;\bftheta)}{\Lambda_\pmax(\bftheta)}  \, ,
\end{equation}
by drawing an independent Bernoulli random variable $B_j(\bftheta) \sim \text{Bernoulli}(p_j(\bftheta))$. The resulting set of accepted neurons is then given by
\begin{equation}\label{eq3.7}
\curV_\bftheta = \left \{ \bfz^j \in \{\bfz^1,\dots,\bfz^M\} \, \vert \, B_j(\bftheta) = 1 \right\} \subset \curS_h \, .
\end{equation}
This set approximates a realization of the inhomogeneous Poisson point process $\PP_\Lambda$.
Note that this thinning algorithm provides an approximation of the inhomogeneous Poisson process $\PP_\Lambda$, which becomes accurate in the limit of a large number of candidate points $M$ and when the proposal distribution $\pi(\bfx)$ is close to uniform with respect to the surface measure $d\sigma(\bfx)$; exact sampling would require $\pi(\bfx)\propto 1$ and $M\rightarrow +\infty$.
{\color{black} Let $\curN_{\!\curS_h}(\bftheta)$ denote}
 the random variable representing the number of accepted points.
This sampling procedure results in a realization of a point process whose spatial distribution approximates the target inhomogeneous Poisson process $\PP_\Lambda$.
{\color{black} In particular, for any region $A \subset \curS_h$, the expected number of accepted points in $A$ approximates $\int_A \Lambda(\bfx;\bftheta)\, d\sigma(\bfx)$, and, conditional on the random number $\curN_{\!\curS_h}(\bftheta)$ of accepted points, the accepted locations are approximately i.i.d. over $\curS_h$ according to the normalized intensity function $p(\bfx;\bftheta)$ defined in Eq.~\eqref{eq3.3}.}
The expected number of accepted candidate points (i.e., neurons) is then approximated by
{\color{black}
\begin{equation}\label{eq3.8}
E\{\curN_{\!\curS_h}(\bftheta)\} = \frac{M}{\vert \curS_h\vert} \frac{1}{\Lambda_\pmax(\bftheta)} \int_{\curS_h} \Lambda(\bfx;\bftheta)\, d\sigma(\bfx)\, .
\end{equation}
}
This formula holds under the assumption that the proposal distribution $\pi(\bfx)$ is approximately uniform with respect to the surface measure $d\sigma(\bfx)$.
Consequently, regions on the manifold where the variance $\Lambda(\bfx;\bftheta)$ is large, typically indicating regions of higher uncertainty or variability of the latent field $U^m(\bfx;\bftheta)$, receive proportionally more neurons \cite{Spantini2017}.
From a numerical standpoint, this method is straightforward to implement and highly parallelizable.\\

\noindent\textit{(iii) On the relationship between the latent Gaussian random field and the spatial distribution of neurons}.
It is important to distinguish between the pointwise variance function $\bfx \mapsto \Lambda(\bfx;\bftheta)$ and any particular realization $\bfx \mapsto U^m(\bfx;\bftheta)$ of the associated latent Gaussian random field. While the variance $\Lambda(\bfx;\bftheta)$ quantifies the local uncertainty of the field at each point of the manifold, a realization may take high or low values regardless of the corresponding variance level. In particular, regions where the field vanishes or takes extreme values may still exhibit either high or vanishing variance, depending on the structure of the reduced-order representation.\\

\noindent\textit{(iv) Implementation of random sampling in the algorithm}.
In the numerical realization of the inhomogeneous Poisson process on the discretized manifold $\curS_h$, the sampling of neuron locations is performed using four independent random vectors: $\bfcurU_\pelem$, $\bfcurU_a$, $\bfcurU_b$, and  $\bfcurU_\pPoisson$, each with entries that are independent, uniformly distributed random variables on $[0,1]$. First, we introduce the random vector
$\bfcurU_\pelem$ with values in $[0 , 1]^M$, which is used to select triangles on the mesh in proportion to their surface areas. The number $M$ of candidate points is generated by $M= p_M \times  n_\pelem$, where $n_\pelem$ is the number of finite elements (triangles) in the mesh,  and $p_M$ is a given tunable parameter representing the number of points per element (for instance, we can choose $p_M=2$). Each entry in $\bfcurU_\pelem$ determines, through inversion of the cumulative area function, which mesh element is chosen for candidate point placement. Within each selected triangle, the random vectors $\bfcurU_a$  and $\bfcurU_b$, with values in $[0 , 1]^M$,  are employed as barycentric coordinates to generate a point uniformly inside the triangle. If their sum exceeds one, a reflection is applied to ensure the sample remains within the triangle. This approach results in a set of $M$ candidate points $\{\bfz^j\}$ that are approximately uniformly distributed over $\curS_h$ with respect to the surface measure.
Once the intensities $\Lambda(\bfz^j;\bftheta)$ at all candidate points have been computed, a discrete probability distribution proportional to these intensities is constructed. The final selection of $N$ neurons from the $M$ candidates is then achieved by drawing $N$ independent samples according to this distribution. The $[0,1]^N$-valued random variable $\bfcurU_\pPoisson$ is mapped via the cumulative distribution function of the normalized intensities, yielding the indices of the selected neuron locations. This sampling procedure ensures that the empirical distribution of neurons closely matches the theoretical spatial distribution
{\color{black} prescribed by the intensity function of the inhomogeneous Poisson process},
thereby concentrating neurons in regions of higher variance of the latent Gaussian field and accurately reflecting the probabilistic structure of the model.
\section{Neural connectivity graph construction from geometry and latent fields}
\label{Section4}
In Section~\ref{Section3}, we introduced the stochastic model for generating $N$ neurons with locations  on the manifold $\curS_h$, governed by an inhomogeneous Poisson process whose intensity is modulated by the local variability of the latent Gaussian field $\{U^m(\bfx; \bftheta), \bfx \in \curS\}$. The hyperparameter vector $\bftheta$ controls the statistical properties of this field and is inferred from supervised input–output training data. Consequently, the neurons locations $\{\bfx^1_\bftheta,\ldots ,\bfx^N_\bftheta\}$ depend on $\bftheta$, which is considered fixed throughout this section.
For a given value of $\bftheta$, we consider a
{\color{black} realization $\{u^m_\bftheta(\bfx), \bfx \in \curS_h\}$ of the random field $\{U^m(\bfx; \bftheta), \bfx \in \curS_h\}$.}
Note that all generated realizations are assumed to be  statistically independent across the different values of $\bftheta$ used by the training algorithm (see Section~\ref{Section8}).
Using this value of $\bftheta$, the generation of the geometric–probabilistic neural topology begins using the realization $\{u^m_\bftheta(\bfx), \bfx \in \curS_h\}$. More precisely, for a fixed realization $\omega_\bftheta \in \Omega$, we denote $u^m_\bftheta(\bfx) = U^m(\bfx;\bftheta;\omega_\bftheta)$.
From the resulting set of $N$ neurons,  subsets of $n_\pin$ input and $n_\pout$ output neurons (independent of $\bftheta$) are selected based on the lowest and highest values, respectively,  of  $\{u^m_\bftheta(\bfx_\bftheta^j) , j=1,\ldots, N\}$, evaluated at the Poisson-sampled locations, as previously explained.
The neural connectivity graph is defined by a binary sparsity mask derived by combining spatial proximity with dependence on the {\color{black} realization $\{u^m_\bftheta(\bfx), \bfx \in \curS_h\}$  of}
 the latent random field . This mask is applied to the fully connected weight matrix to obtain the sparsely connected weights. Synaptic weights are then randomly assigned, conditional on the latent field realization, introducing variability in neural responses even under fixed input values. Scalability is achieved through the sparsification mechanism incorporated into the stochastic neural topology.
\subsection{Input and output neuron selection from the latent field for graph topology}
\label{Section4.1}
For a fixed value of $\bftheta$, let $\{u^m_\bftheta(\bfx), \bfx \in \curS_h\}$ be the given realization of the random field $\{U^m(\bfx;\theta), \bfx \in \curS_h\}$ introduced above, from which the graph construction proceeds. Let $\{\bfx_\bftheta^j, j=1,\ldots, N\}$ denote the $N$ neurons stochastically located on the manifold $\curS_h$ as described in Section~\ref{Section3}, and depending on $\bftheta$. For each neuron located at $\bfx_\bftheta^j$, the value $u^m_\bftheta(\bfx_\bftheta^j)$ of the realized Gaussian random variable $U^m(\bfx^j;\bftheta)$ is computed.
The subset of $n_\pin$ input neurons is selected as the nodes corresponding to the $n_\pin$ smallest values among $\{u^m_\bftheta(\bfx_\bftheta^j),j=1,\ldots, N\}$. Let $\curJ_\bftheta^\pin = \{i^{\,\bftheta}_1,\ldots,i^{\,\bftheta}_{n_\ppin}\} \subset \{1,\ldots,N\}$ be the index set corresponding to the sorted order,
\begin{equation}\label{eq4.1}
u^m_\bftheta(\bfx_\bftheta^{i^{\,\bftheta}_1}) \leq u^m_\bftheta(\bfx_\bftheta^{i^{\,\bftheta}_2}) \leq \dots \leq u^m_\bftheta(\bfx_\bftheta^{i^{\,\bftheta}_{n_\pppin}}) \, .
\end{equation}
Similarly, the subset of $n_\pout$ output neurons is selected as the nodes corresponding to the $n_\pout$ largest values of $u^m_\bftheta(\bfx^j)$. Let $\curJ_\bftheta^\pout = \{o^{\,\bftheta}_1,\ldots,o^{\,\bftheta}_{n_\ppout}\} \subset \{1,\ldots,N\}$ be the index set such that,
\begin{equation}\label{eq4.2}
u^m_\bftheta(\bfx_\bftheta^{o^{\,\bftheta}_1}) \geq u^m_\bftheta(\bfx_\bftheta^{o^{\,\bftheta}_2}) \geq \ldots \geq u^m_\bftheta(\bfx_\bftheta^{o^{\,\bftheta}_{n_\pppout}})\, .
\end{equation}
Since $N$, $n_\pin$, and $n_\pout$ are given, the indices of the $n_\pint$ internal (hidden) neurons, which depend on $\bftheta$, are grouped in $\curJ_\bftheta^\pint$, and we have
\begin{equation}\label{eq4.3}
N = n_\pin + n_\pint + n_\pout \quad , \quad \curJ_\bftheta^\pin\cup \curJ_\bftheta^\pint\cup\curJ_\bftheta^\pout = \{1,\ldots , N\}
\quad , \quad \curJ_\bftheta^\pin\cap \curJ_\bftheta^\pint\cap\curJ_\bftheta^\pout = \emptyset.
\end{equation}
Note that if $n_\pin$, $n_\pout$, $n_\pint$
{\color{black}  (and thus $N$) are fixed independently of $\bftheta$, the index sets}
 $\curJ_\theta^\pin$, $\curJ_\theta^\pint$, and $\curJ_\bftheta^\pout$ still depend on $\bftheta$.
This selection strategy induces a functional differentiation of neurons based on the intensity landscape of the latent field:
{\color{black} input neurons lie in low-activation regions (the smallest values of the random latent field), whereas output neurons lie in high-activation regions (the largest values of the random latent field).}
This choice introduces a geometric and probabilistic bias into the architecture, governed by the underlying generative process.

The input and output amplitudes of the random neural network are then represented by the vectors $\xx = (\xx_1,\ldots,\xx_{n_\ppin})\in\RR^{n_\ppin}$ and $\yy = (\yy_1,\ldots,\yy_{n_\ppout})\in\RR^{n_\ppout}$, respectively, whose  components correspond to the activation values of the neurons selected by the process described above.
\subsection{Sparsity binary mask construction for neural graph topology}
\label{Section4.2}
Having generated the vertex set $\curV_\bftheta = \{ \bfx_\bftheta^1,\ldots ,\bfx_\bftheta^N\}$ from the Poisson process (Section~\ref{Section3}) and the preliminary fully connected weight matrix $[w^g_\bftheta]$ (Section~\ref{Section4.2.2}-(i)), we construct a set of weighted edges $\curE_\bftheta$ to obtain
{\color{black} a sparse matrix}
$[w_\bftheta^\psp]$. This defines the graph $\curG_\bftheta = (\curV_\bftheta, \curE_\bftheta)$, whose topology reflects both spatial geometry and field-induced connectivity.
{\color{black} This topology captures local geometric structure and manifold connectivity and critically influences subsequent training.}
In what follows, we briefly review standard methodologies for edge construction and sparsification, before presenting the approach adopted in this work.
\subsubsection{Short survey of edge construction and sparsification methodologies}
\label{section4.2.1}
The preliminary weight matrix  $[w^g_\bftheta]$ corresponds to a  fully connected graph on $N$ nodes, resulting in $O(N^2)$ edges, which is often computationally prohibitive. To obtain a sparser and more tractable representation, several well-established methods are commonly used:
\begin{itemize}
    \item $k$-nearest neighbors ($k$-NN) graph. For each node $\bfx_\bftheta^i$, retain only the $k$ edges $(\bfx_\bftheta^i,\bfx_\bftheta^j)$ with the highest weights $[w_\bftheta]_{ij}$, corresponding to the $k$ most strongly connected neighbors. This results in a locally adaptive graph with fixed degree $k$ \cite{Luxburg2007,Maier2009}.

    \item Radius-based graph. Include the edge $(\bfx_\bftheta^i,\bfx_\bftheta^j)$ whenever the approximate geodesic distance satisfies $\hat d_g(\bfx^i_\bftheta,\bfx^j_\bftheta)$ $ \leq r$, where the radius $r$ can be chosen either globally or based on local statistics. This approach induces a variable node degree while preserving spatial locality \cite{Penrose2003}.

    \item Spectral sparsification. Construct a sparse graph with weight matrix $[w^\psp_\bftheta]$ that preserves the spectral (Laplacian) properties of the original fully connected graph. Techniques such as effective resistance sampling are often used to approximate the graph spectrum while reducing the edge count to $O(N \log N)$ \cite{Spielman2011}.

    \item Locally scaled diffusion kernel. Following \cite{Zelnik2004}, define the approximate node-dependent scales $\hat \sigma_{\bftheta,i}$ as the geodesic distance between $\bfx^i_\bftheta$ and its $k$-th nearest neighbor, denoted $\bfx^{i,k}_\bftheta$. These scales reflect local sampling density. Then define the affinity as $[w^g_\bftheta]_{ij} = {\exp}( - \hat d_g(\bfx^i_\bftheta, \bfx^j_\bftheta)^2 /(\hat\sigma_{\bftheta,i} \, \hat\sigma_{\bftheta,j}))$ and
        $[w^g_\bftheta]_{ii} = 0$, which adaptively adjusts the kernel bandwidth according to the density around each node, yielding a smoother and more robust similarity structure.
\end{itemize}
In this work, we have selected the locally scaled diffusion kernel method.

\subsubsection{Self-tuning diffusion kernel with percentile-based sparsification}
\label{Section4.2.2}
We construct a neuron–neuron affinity graph by combining intrinsic geodesic distances computed on the mesh with a self-tuning diffusion kernel \cite{Zelnik2004}, followed by percentile-based thresholding to enforce a desired sparsity level.
While the preliminary affinity weights are defined by Eq.\eqref{eq4.4}, the sparsification procedure considers only geometric information. Specifically, we use a temporary kernel based solely on intrinsic geodesic distances, omitting neuron feature values, which will be introduced
{\color{black} in Section~\ref{Section5} for}
defining the random weights of the neural architecture.
This kernel is used to derive a binary sparsity mask $[M^\psp_\bftheta] \in \{0, 1\}^{N\times N}$, which determines the connectivity pattern of the final graph. The original affinity weights from Eq.~\eqref{eq4.4} are then assigned to the selected edges.\\

\noindent\textit{(i) Computation of edge weights for constructing the binary sparsity mask}.
Consider two neurons located at positions $\bfx^i_\bftheta$ and $\bfx^j_\bftheta$ on the manifold $\curS_h$. We define the geodesic distance $d_g(\bfx^i_\bftheta, \bfx^j_\bftheta)$ as the length of the shortest path connecting these two points on $\curS_h$.
A finite element approximation $\hat{d}_g(\bfx^i_\bftheta, \bfx^j_\bftheta)$ of this geodesic distance is given by Eq.~\eqref{eq4.7} in Section~\ref{Section4.2.2}-(ii).
Using this geodesic metric, we construct the preliminary (fully connected) edge weight matrix $[w^g_\bftheta]$, whose entries are defined by
\begin{equation} \label{eq4.4}
   [w^g_\bftheta]_{ij} = \exp\!\left(-\frac{d_g(\bfx^i_\bftheta , \bfx^j_\bftheta)^2}{\sigma_{\bftheta,i}\,\sigma_{\bftheta,j}}\right)
   \quad , \quad  [w^g_\bftheta]_{ii}=0 \, .
\end{equation}
Here, the parameter $\sigma_{\bftheta,i}$ (resp. $\sigma_{\bftheta,j}$) is local geometric bandwidths, defined by the geodesic distance from neuron location $\bfx_\bftheta^i$ (resp. $\bfx_\bftheta^j$) to its $k$-th nearest neighbor (estimated using Eq.~\eqref{eq4.9}). These bandwidths act as local scaling factors that adaptively normalize distances according to neuron density around each point, ensuring robustness to spatially varying sampling rates.
The exponential term encodes adaptive geometric proximity, generalizing ideas from \cite{Belkin2003, Coifman2006}, and ensures that neurons close in geodesic distance and located in denser regions are assigned stronger weights. Consequently, a pair of neurons will have a large edge weight only if they are spatially proximal, while distant pairs will have exponentially small weights.
{\color{black} As previously explained, since a fully connected graph} is computationally prohibitive for large neuron counts, a sparsification procedure (discussed in the next section) is applied to obtain a computationally tractable network while preserving the essential geometric and probabilistic structure of the neuron topology.\\

\noindent\textit{Remark 2}. In Section~\ref{Section5}, we will consider a more general expression than Eq.~\eqref{eq4.4} for constructing the random weight matrix $[\bfW_\bftheta]$ and its sparse approximation $[\bfW^\psp_\bftheta]$ for the random neural architecture. These will be constructed using the binary sparsity mask introduced here. This more general formulation of the weights will impose additional constraints by promoting strong connections between neurons with similar latent field values. For now, we use the simpler expression in Eq.~\eqref{eq4.4} to construct the binary mask matrix that yields the sparse connectivity pattern of the random neuron system.\\

\noindent\textit{(ii) Intrinsic geodesic distances.}
{\color{black} Recall that the manifold $\curS_h \subset \RR^3$ denotes the finite-element mesh of $\curS$, consisting}
of $n_o$ nodes (see Section~\ref{Section2.1}-(vi)). Denote the coordinates of the $p$-th finite element node by
$\bfx^p_o = (x^p_{o,1}, x^p_{o,2}, x^p_{o,3}) \in \curS \subset \RR^3$ for $p \in\{ 1, \ldots, n_o\}$,
and define the set of all mesh node coordinates as $\{\bfx^p_o  , p=1,\ldots , n_o\}$. Note that the finite element nodes belong to $\curS_h$ but also to $\curS$.
We now consider the triangular connectivity of the mesh (i.e., the list of triangle elements defined over the nodes). Using this connectivity, we construct the adjacency matrix $[A] \in \MM_{n_o}$, where each entry represents the Euclidean distance between directly connected mesh nodes,
\begin{equation}\label{eq4.5}
A_{pq} = \Vert \bfx^p_o - \bfx^p_o \Vert \,\,  \text{if} \,\, (p,q) \,\, \text{is an edge of the mesh}, \,\, \text{and}\, =  0 \,\, \text{otherwise} \, .
\end{equation}
To compute intrinsic geodesic distances between the nodes of the mesh, we first compute the shortest-path geodesic distances between all pairs of mesh nodes using the adjacency matrix $[A]$.
\begin{equation}\label{eq4.6}
[D_g]_{pq} = \min_{\xi \, \in \, \Xi_{pq}} \sum_{(p',\, q')\, \in \, \xi} A_{p'q'} \quad \, \quad \forall\, (p,q) \in \{1,\ldots,n_o\} \times \{1,\ldots,n_o\}\, ,
\end{equation}
where $\Xi_{pq}$ is the set of all paths in the graph induced by the mesh connectivity that connect nodes $p$ and $q$. This shortest-path computation can be implemented efficiently using algorithms such as Dijkstra’s or Fast Marching methods on graphs \cite{Kimmel1998}.
Intrinsic geodesic distances accurately reflect manifold-based separation and avoid artificial shortcuts through Euclidean space.
Since the neuron locations $\{\bfx_\bftheta^i, i=1,\ldots, N\}$ are not located at the mesh nodes $\{\bfx^p_o, p=1,\ldots , n_o\}$, but rather lie in the interior of finite elements, the geodesic distance between two neurons must be interpolated from the mesh-based distances $[D_g]_{pq}$.
Let $\{\varphi_p(\bfx), p=1,\ldots, n_o\}$ denote the finite element basis functions associated with the mesh nodes (see Section~\ref{Section2.1}-(vi)).
Then, for two arbitrary points $\bfx$ and $\bfy$ in $\curS_h$, an interpolated approximation $\hat{d}_g(\bfx,\bfy)$ of their true geodesic distance $d_g(\bfx,\bfy)$, is defined by
\begin{equation} \label{eq4.7}
\hat{d}_g(\bfx,\bfy) = \sum_{p=1}^{n_o} \sum_{q=1}^{n_o} \varphi_p(\bfx)\, [D_g]_{pq}\, \varphi_q(\bfy) \, .
\end{equation}
This bilinear interpolation ensures that $\hat{d}_g(\bfx^p_o, \bfx^q_o) = [D_g]_{pq}$ at the mesh nodes, is independent of $\bftheta$, and extends geodesic distances to all of $\curS_h$.
Applying Eq.~\eqref{eq4.7} to the neuron coordinates $\{\bfx^i_\bftheta , i=1,\ldots ,N\}$ yields the geodesic distance matrix between neurons,
\begin{equation} \label{eq4.8}
[D_{g,\bftheta}^{(\mathrm{neu})}]_{ij} = \hat{d}_g(\bfx^i_\bftheta, \bfx^j_\bftheta) = \sum_{p=1}^{n_o} \sum_{q=1}^{n_o} \varphi_p(\bfx^i_\bftheta)\, [D_g]_{pq}\, \varphi_q(\bfx^j_\bftheta) \quad , \quad i,j \in \{1,\ldots, N\} \, .
\end{equation}
%

\noindent\textit{(iii) Local bandwidth estimation $\hat \sigma_{\bftheta,i}$}.
To obtain locally adaptive scales, for each neuron $i \in \{1, \ldots, N\}$, we use the geodesic distances
$\{ [D_{g,\bftheta}^{(\mathrm{neu})}]_{ij} , j \neq i\}$ to all other neurons $j$. These distances are sorted in increasing order:
$\{ [D_{g,\bftheta}^{(\mathrm{neu})}]_{i,(1)} < [D_{g,\bftheta}^{(\mathrm{neu})}]_{i,(2)} < \ldots < [D_{g,\bftheta}^{(\mathrm{neu})}]_{i,(M)} \}$, and the approximate local bandwidth $\hat \sigma_{\bftheta,i}$ is defined as the distance to the $(k+1)$th nearest neighbor,
{\color{black}
\begin{equation}\label{eq4.9}
  \hat \sigma_{\bftheta,i} =[D_{g,\bftheta}^{(\mathrm{neu})}]_{i,(k+1)}  \quad ,  \quad k = \min(\sqrt{N},\,N-1) \, .
\end{equation}
}
{\color{black} This choice of $k$ trades off}
fine-scale local geometry (for small $k$) and robustness to noise or outliers (for larger $k$). This local scale $\hat \sigma_{\bftheta,i}$ adapts to the sampling density of the neuron distribution: it is smaller in dense regions and larger in sparse regions.\\

\noindent\textit{(iv) Self-tuning diffusion kernel}.
Using the locally adaptive scales $\{\hat\sigma_{\bftheta,i} , i=1,\ldots , N\}$ defined in Eq.~\eqref{eq4.9}, we construct a symmetric affinity matrix $[w^g_\bftheta] \in \MM_N$ (corresponding to the estimation of Eq.~\eqref{eq4.4}), whose entries encode both proximity and local density adaptation. The self-tuning diffusion kernel is defined as
\begin{equation}\label{eq4.10}
[w^g_\bftheta]_{ij} = \exp\left(- \frac{[D_{g,\bftheta}^{(\mathrm{neu})}]_{ij}^2}{\hat\sigma_{\bftheta,i}\,\hat\sigma_{\bftheta,j}} \right) \quad , \quad [w^g_\bftheta]_{ii} = 0 \, , \quad  i,j \in \{1, \ldots, N\} \, .
\end{equation}
This kernel follows the approach proposed in \cite{Zelnik2004}, where each node adapts its own local bandwidth based on its neighborhood density. The resulting graph affinity is stronger between neurons that are both close in geodesic distance and embedded in dense regions (i.e., when $\hat\sigma_{\bftheta,i}\, \hat\sigma_{\bftheta,j}$ is small), and weaker for pairs in sparser regions. This self-tuning property eliminates the need to manually specify a global scale and ensures that both local and global geometric features are faithfully encoded in the resulting weighted graph.\\

\noindent\textit{(v) Percentile-based sparsification.}
To control the sparsity of the graph, we introduce a global percentile threshold parameter $\tau_\pprc \in [0,100]$. Let $a_\pprc^\bftheta$ denote the $\tau_\pprc$th percentile of the upper-triangular (off-diagonal) entries of the symmetric matrix $[w^g_\bftheta]$,
\begin{equation}\label{eq4.11}
a_\pprc^\bftheta = \texttt{prctile}\left( \{ [w^g_\bftheta]_{ij} \,:\, 1 \le i < j \le N \} ,\, \tau_\pprc \right) \, .
\end{equation}
Here, $\texttt{prctile}(z,\tau_\pprc)$ denotes the value below which a fraction $\tau_\pprc/100$ of the elements in the
set $\curZ = \{ [w^g_\bftheta]_{ij} \,:\, 1 \le i < j \le N \}$ fall.
We define a binary sparsity edge-mask $[M^\psp_\bftheta]\in \{0,1\}^{N \times N}$ that encodes the sparsity pattern induced by the thresholded self-tuning kernel $[w^g_\bftheta]$ as follows,
{\color{black}
\begin{equation}\label{eq4.12}
[M^\psp_\bftheta]_{ij} = 1 \,\,\text{if} \,\, [w^g_\bftheta]_{ij} \ge a_\pprc^\bftheta \,\, \text{and}\,\, i \ne j\, , \,\, \text{and}\,\,
 = 0 \,\, \text{otherwise} \, .
\end{equation}
}
The final sparse, undirected neural affinity matrix is denoted by $[w^\psp_\bftheta] \in \MM_N$, which corresponds to the sparse field-informed weights matrix as the Hadamard (elementwise) product,
\begin{equation}\label{eq4.13}
[w^\psp_\bftheta]_{ij} = [M^\psp_\bftheta]_{ij} \, [w^g_\bftheta]_{ij} \, ,
\end{equation}
where $[w^g_\bftheta]$ is
{\color{black} defined by Eq.~\eqref{eq4.10}.}
This construction guarantees that sparsity is determined purely from geometry (via $[w^g_\bftheta]$).
In addition, this procedure preserves only the top $(100 - \tau_{\mathrm{prc}})\%$ strongest affinities, thereby enforcing global control over graph density. For instance, the choice such as $\tau_{\mathrm{prc}} = 75\%$ (i.e., keeping the top $25\%$ of edge weights) yields a favorable trade-off between computational efficiency and preservation of the manifold's local geometry.
The number of undirected edges in the sparse graph is then given by
\begin{equation}\label{eq4.14}
\vert\, \curE_\bftheta\, \vert = \frac{1}{2} \sum_{i=1}^{N} \sum_{\substack{j=1 \\ j \ne i}}^{N} [M^\psp_\bftheta]_{ij} \, .
\end{equation}
%

\noindent\textit{(vi) Connectivity validation of the random graph}.
For each fixed $\bftheta$, once the construction of the binary sparsity edge-mask $[M^\psp_\bftheta]$,
{\color{black} defined by Eq.~\eqref{eq4.12},}
is achieved, a test is applied to ensure that $[M^\psp_\bftheta]$ is a symmetric, zero-diagonal adjacency matrix corresponding to an undirected simple graph in which all nodes are connected through at least one path—meaning the graph is connected and contains no isolated nodes or disconnected components.

\section{Random weight construction for sparsified connectivity in the geometric-probabilistic neural
graph given the hyperparameter}
\label{Section5}
%
\noindent\textit{(i) Random weight construction}.
As previously explained, the supervised input-output training data is used to identify the hyperparameter $\bftheta$ via a random sparse weight matrix that depends on $\bftheta$.
Let the real-valued random field $\{ U^m(\bfx;\bftheta), \bfx \in \curS_h \}$ be the Gaussian random field, indexed by the manifold $\curS_h$, as defined by Eq.~\eqref{eq2.18}.
{\color{black} Following Eq.~\eqref{eq4.10},}
let $S^\bftheta_j$ be the Gaussian real-valued random variable defined by
\begin{equation} \label{eq5.1}
  S^\bftheta_j = U^m(\bfx_\bftheta^j;\bftheta)  \quad , \quad j \in \{1, \ldots, N\} \, ,
\end{equation}
where $\{\bfx^j_\bftheta , j=1,\ldots,N\}$ is defined by Eq.~\eqref{eq3.4}.
Let $\sigma_{\bfS_\bftheta}^2 = N^{-1}\operatorname{tr}[C_{\bfS_\bftheta}]$, where $[C_{\bfS_\bftheta}]$ is the covariance matrix of the $\RR^N$-valued random variable $\bfS_\bftheta = (S_{\bftheta,1}, \ldots, S_{\bftheta,N})$.
We define the random  matrix $[\bfW_\bftheta]$ with values in $\MM_N$, which combines both deterministic geometric  and  random field-based similarities,
\begin{equation} \label{eq5.2}
  [\bfW_\bftheta]_{ij} = [w^g_\bftheta]_{ij}\,
    \exp\!\left(-\frac{(\Delta S^{\bftheta}_{ij})^2}{2\zeta_s^2\,\sigma_{\bfS_\bftheta}^2}\right), \quad [\bfW_\bftheta]_{ii} = 0 \, ,
\end{equation}
where $[w^g_\bftheta]_{ij}$ is
{\color{black} defined by Eq.~\eqref{eq4.10} and}
$\Delta S^\bftheta_{ij} = S^\bftheta_i - S^\bftheta_j$, with the following interpretations:
\begin{itemize}
\item The random difference $\Delta S_{\bftheta,ij}$ is defined as $\Delta S_{\bftheta,ij} = S_{\bftheta,i} - S_{\bftheta,j}$, where $S_{\bftheta,i}$ and $S_{\bftheta,j}$ are the values of the
    {\color{black} random field $U^m(\cdot\, ;\bftheta)$}
     neurons $i$ and $j$, as defined in Eq.~\eqref{eq5.1}.
\item The parameter $\zeta_s > 0$ controls the relative sensitivity of the random edge weights to differences
     {\color{black} $\Delta S_{\bftheta,ij}$ normalized}
      by $\sigma_{\bfS_\bftheta}$. This normalization ensures that the influence of dissimilarity is invariant to changes in the field scale. Larger values of $\zeta_s$ result in weaker penalization of normalized field differences, while smaller values enforce stricter similarity for strong connections.
\end{itemize}
As previously explained, the first exponential term in Eq.~\eqref{eq5.2} encodes adaptive geometric proximity, and ensures that neurons closer in geodesic distance and located in denser regions are assigned stronger weights. The second exponential term imposes additional constraints by promoting strong connections between neurons with similar latent field values.
To enforce sparsity while preserving the most significant connections, we apply the binary mask matrix $[M^\psp_\bftheta] \in \MM_N$, {\color{black} defined in Eq.~\eqref{eq4.12}.}
The final sparse random matrix is then given by the Hadamard (elementwise) product,
\begin{equation} \label{eq5.3}
  [\bfW^\psp_\bftheta]_{ij} = [M^\psp_\bftheta]_{ij} \, [\bfW_\bftheta]_{ij} \quad ,\quad i,j\in\{1,\ldots, N\}\, .
\end{equation}
The resulting sparse random matrix $[\bfW^\psp_\bftheta]$  with values in $\MM_N$, encodes the geometry- and field-informed stochastic connectivity of the neural graph at hyperparameter value $\bftheta$.\\

\noindent\textit{Remark 3}. It should be noted that the weights  $[\bfW_\bftheta]_{ij}$ defined by Eq.~\eqref{eq5.2} are  positive-valued random variables. Note that, even with positive weights and real biases, a feedforward ANN with nonlinear activations functions (such as ReLU, sigmoid, etc.), can still act as a  universal approximator (see, for instance, \cite{Sonoda2017}). However, we do not consider a classical feedforward ANN but a parameter-conditionned stochastic Graph Neural Network  (stochastic GNN). \\

\noindent\textit{(ii) Generation of independent realizations of the random weight matrix $[\bfW^\psp_\bftheta]$}.
For a fixed value of $\bftheta$, $n_\psim$ independent realizations $\{[\bfw^{\psp,\ell}_\bftheta], \ell=1,\ldots , n_\psim\}$ of $[\bfW^\psp_\bftheta]$  can be generated using $n_\psim$ independent realizations $s^{\bftheta,\ell}_i$ of the Gaussian random variable $S^{\bftheta}_i $, such that
\begin{equation}\label{eq5.4}
s^{\bftheta,\ell}_i = U^m(\bfx_\bftheta^j; \bftheta;\omega_\ell) \quad , \quad \omega_\ell\in\Omega \, ,
\end{equation}
where $\{U^m(\bfx\, ; \bftheta;\omega_\ell),\bfx\in\curS_h\}$ is the realization $\omega_\ell$ of the random field
$\{U^m(\bfx\, ; \bftheta),\bfx\in\curS_h\}$ defined by Eq.~\eqref{eq2.18} with Eq.~\eqref{eq2.19}, which are rewritten as
\begin{equation}\label{eq5.5}
U^m(\bfx; \bftheta;\omega_\ell)=  \langle \, \bfpsi^m(\bfx;\bftheta) , \bfeta^\ell)\, \rangle \quad , \quad  \forall\, \bfx\in\curS_h\, ,
\end{equation}
where $\{\bfeta^\ell = \bfH(\omega_\ell), \ell=1,\ldots,n_\psim\}$ are $n_\psim$ independent realizations of the normalized Gaussian $\RR^m$-valued random variable, and where
\begin{equation} \label{eq5.6}
\bfpsi^m(\bfx;\bftheta) = [\lambda^m(\bftheta)]^{1/2} [\Phi^m(\bftheta)]^T\, \bfvarphi(\bfx)\, .
\end{equation}
%

\noindent\textit{(iii) Final definition of hyperparameter $\bftheta$}.
We can now define the hyperparameter $\bftheta$ and it admissible set $\curC_\bftheta$,
{\color{black}
\begin{equation}\label{eq5.7}
\bftheta = (h_1\, ,h_2\, ,\zeta_s , \bfbeta^{\,(1)}_{n_h}\!, \bfbeta^{\,(2)}_{n_h}) \in \curC_\bftheta\subset \RR^{3 +2 n_h}\quad, \quad
         \curC_\bftheta =\{ h_1 > 0 \, , h_2 > 0 \, , \zeta_s > 0 \, , \bfbeta^{\,(1)}_{n_h} \in \curC_{\adp,1} \, ,
         \bfbeta^{\,(2)}_{n_h} \in \curC_{\adp,2}\}\, ,
\end{equation}
}
where for $k\in\{1,2\}$, the admissible set $\curC_{\adp,k} \subset \RR^{n_h}$ is defined (see Eq.~\eqref{eq2.5}) by
{\color{black}
\begin{equation}\label{eq5.8}
\curC_{\adp,k} = \{ \bfbeta^{\,(k)}_{n_h} \in\RR^{n_h} \, , \,
0 < c_\star^{(k)} \,\leq \, h^{(k)}_{n_h}(\bfx) = h_k + \sum_{j=1}^{n_h} \beta_j^{\,(k)}\,\hh^{(k,j)}(\bfx) \, \leq \,  c^{(k)\star} \, , \, \forall \, \bfx \in \curS
   \}\, .
\end{equation}
}
\section{Low-rank representation of neuron biases in large neural networks}
\label{Section6}

When constructing the probabilistic model of the neural network, the bias vector of the neurons may also be modeled. Unlike the weight parameters are random, the bias vector acts only as an offset and does not exhibit any localized or spatially structured behavior. For a large number $N$ of neurons, the bias vector $\bfb\in\RR^N$ can be reduced in dimensionality by expressing it as
\begin{equation}\label{eq6.1}
\bfb = \sum_{j=1}^{n_b} \beta_j \bfh^j \quad , \quad n_b \ll N\, ,
\end{equation}
using $n_b$ algebraically independent vectors  $\{\bfh^j\}_{j=1}^{n_b}$ in $\RR^{N}$, where $\{\beta_j\}_{j=1}^{n_b}$ are hyperparameters.  However, in many applications, no low-rank reduction is applied. There are several strategies to construct the basis vectors $\{\bfh^j\}_j$, either by using a predefined basis, typically fixed, orthonormal, and not learned, or by learning the basis during training. Nevertheless, these classical aspects are not the focus of this work, and we refer the reader, for instance, to \cite{Xu2019,Idelbayev2020,Huh2021,Galanti2023,Shen2023,Park2024} for further details.
\section{Implementation of the Stochastic Graph Neural Network}
\label{Section7}
For each value of the hyperparameter vector $\bftheta$, a realization of the random neural architecture is generated, as explained at the beginning of Section~\ref{Section4}, yielding a new configuration of the network connectivity and weights. This allows the model to explore a diverse set over possible neural architectures.\\

\noindent \textit{(i) Computation of neural activations in the general graph neural network}.
For $\bftheta$ fixed in $\curC_\bftheta$ (see Eq.\eqref{eq5.7}), the general neural network that we constructed in Sections~\ref{Section3} to \ref{Section6} consists of $N$ neurons, labeled from $1$ to $N$, interconnected by a directed graph. For each given  $\bftheta$, this network does not follow a layered or feedforward architecture. Instead, it is characterized by the following components.
\begin{itemize}
   \item As explained in Section~\ref{Section4.1}, a subset $\curJ_\bftheta^\pin = \{i^{\,\bftheta}_1,\ldots,i^{\,\bftheta}_{n_\ppin}\}$ of $n_\pin$ input neurons and a subset $\curJ_\bftheta^\pout = \{o^{\,\bftheta}_1,\ldots,o^{\,\bftheta}_{n_\ppout}\}$ of $n_\pout$ output neurons are selected. These subsets depend on $\bftheta$ and satisfy $\curJ_\bftheta^\pin \cap \curJ_\bftheta^\pout = \emptyset$.
    \item The binary connectivity matrix $[M^\psp_\bftheta]\in \{0,1\}^{N \times N}$,
     {\color{black} defined in Eq.~\eqref{eq4.12},}
      where $[M^\psp_\bftheta]_{ij} = 1$ indicates a directed connection from neuron $j$ to neuron $i$.
    \item The $\MM_N$-valued random weight matrix $[\bfW_\bftheta]$, defined by Eq.~\eqref{eq5.2}, encodes  full connectivity among all $N$ neurons, with entries given by positive-valued random variables. A total of $n_\psim$ independent realizations
        $\{[\bfw^\ell_\bftheta], \ell=1,\ldots , n_\psim\}$ are generated, as described  in Section~\ref{Section5}-(ii). From Eqs.~\eqref{eq4.10} and \eqref{eq5.2}, each $[\bfw^\ell_\bftheta]\in\MM_N$ is symmetric, with zero diagonal $[\bfw^\ell_\bftheta]_{ii} = 0$, and strictly positive off-diagonals $0 < [\bfw^\ell_\bftheta]_{ij} < 1$ for $i \neq j$.
    \item The sparse random weights matrix $[\bfW^\psp_\bftheta]$, associated with the binary connectivity matrix, is defined as     $[\bfW^\psp_\bftheta]_{ij} = [M^\psp_\bftheta]_{ij} \, [\bfW_\bftheta]_{ij}$. The $n_\psim$ independent realizations
         $\{[\bfw^{\psp,\, \ell}_\bftheta], \ell=1,\ldots , n_\psim\}$ are then computed by $[\bfw^{\psp,\, \ell}_\bftheta]_{ij} = [M^\psp_\bftheta]_{ij} \, [\bfw^\ell_\bftheta]_{ij}$.
    \item The deterministic bias vector $\bfb = \sum_{j=1}^{n_b} \beta_j \bfh^j$, defined by Eq.~\eqref{eq6.1}), provides  a low-rank representation of the bias, where  $\bfbeta=(\beta_1,\ldots,\beta_{n_b})$ is a hyperparameter vector. However, in many applications, $n_b$ is simply taken to be $N$, thus foregoing dimensionality reduction.
    \item An activation function $f: \RR \to ]-1 , 1[$, applied element-wise, is chosen as $f(x) = \tanh(x)$ whose derivative at origin is $1$.
        {\color{black} Note that}
         $f(x) = x/(1+\vert x\vert)$ could also be used. However, $f(x) = \max(0,x)$ is not suitable in the present formulation, as the gradient is computed numerically using central finite differences. More importantly, we require the range of $f$ to be a bounded set, see paragraph~(iv)).
\end{itemize}
Using the notations introduced at the end of Section~\ref{Section4.1} for the input-output amplitudes, for fixed $\bftheta\in\curC_\bftheta$, given an input vector $\xx = (\xx_1,\ldots,\xx_\pin)\in\RR^{n_\ppin}$ applied at the neurons indexed by $\curJ_\bftheta^\pin$, and a given realization $[\bfw^{\psp,\, \ell}_\bftheta]$ of the sparse weight matrix $[\bfW^\psp_\bftheta]$, we aim to compute the corresponding
{\color{black} realization $\yy^\ell(\bftheta_t \mid \xx) \in \RR^{n_\ppout}$ of}
the random output
{\color{black} vector $\YY(\bftheta_t \, \vert \, \xx)$ at}
the neurons indexed by $\curJ_\bftheta^\pout$. This output consists of the activation values at the output neurons, after the influence of the input has propagated through the network. Unlike in feedforward neural networks, this propagation is defined implicitly due to the recurrent and potentially cyclic structure of the graph.\\

\noindent \textit{(ii) Random neural network equation}.
The notation $\bfz = f(\bfxi)$, where $\bfz$ and $\bfxi$ belong to  $\RR^N$, means that $z_i = f(\xi_i)$ for $i=1,\ldots, N$, that is,  the function $f$ is applied element-wise.
{\color{black} Let $\bfA_{\bftheta_t} =(A_{\bftheta_t,1},\ldots , A_{\bftheta_t,N})$ denote the full $\RR^N$-valued random activation vector of all neurons. Let $\bfa^\ell_{\bftheta_t} =(a^\ell_{\bftheta_t,1},\ldots , a^\ell_{\bftheta_t,N}) \in \RR^N$ represent any realization of $\bfA_{\bftheta_t}$.
}
This random activation vector collects all the individual neuron activations and depends on the input $\bfx$ (not explicitly shown),
{\color{black} which must be a realization of the probability measure $P_\XX(d\xx)$).
Each component $A_{\bftheta_t ,i}$
}
represents the random activation of neuron $i$, for $i=1,\ldots, N$. The random activations must satisfy the following constrained random equation,
{\color{black}
\begin{equation} \label{eq7.1}
    \bfA_{\bftheta_t} = \bff\,( \,[\bfW^{\psp}_\bftheta] \,\bfA_{\bftheta_t} \, +\, \bfb\, ) \, ,
\end{equation}
}
subject to the constraint induced by the prescribed input value $\xx\in\RR^{n_\ppin}$ applied to the deterministic input neurons.
{\color{black}
It should be noted that the random solution $\bfA_{\bftheta_t}$ of Eq.~\eqref{eq7.1} effectively depends on $\bftheta_t= (\bftheta,\bfbeta)$, because the vector $\bfb$ in the equation depends on $\bfbeta$ (see Eq.~\eqref{eq6.1}).
}
\\

\noindent \textit{(iii) Reformulation for solving the random neural network equation}.
We partition the $\RR^N$-valued random vector $\bfA_{\bftheta_t}$ into two disjoint random vectors: an $\RR^{\hat n}$-valued random vector
 $\hat \bfA_{\bftheta_t}$ and the  deterministic input $\xx\in\RR^{n_\ppin}$, which is  independent of $\bftheta_t$, such that  $\hat n + n_\pin = N$, and
{\color{black}
\begin{equation} \label{eq7.2}
  \hat \bfA_{\bftheta_t}  = [\hat O_\bftheta]\, \bfA_{\bftheta_t} \quad , \quad \xx = [O^{\,\pin}_\bftheta]\, \bfA_{\bftheta_t} \, ,
\end{equation}
}
where $[\hat O_\bftheta]\in\MM_{\hat n,N}$ and $[O^{\,\pin}_\bftheta]\in\MM_{n_\ppin , N}$ are binary matrices (containing only $0$ and $1$) that depend on $\bftheta$, and where the indices corresponding to the components of $\hat \bfA_{\bftheta_t}$ and $\xx$ form a complementary partition of the indices of $\bfA_{\bftheta_t}$. Each column of $[\hat O_\bftheta]$ and  $[O^{\,\pin}_\bftheta]$ exactly contains one $1$ (and the rest of the entries is $0$).
This partition is constructed using the set $\widehat\curJ_\bftheta$ of the $\hat n = N - n_\pin$ free (non-input) neurons, defined by
$\widehat\curJ_\bftheta = \{1,\ldots, N\}\,  \backslash\,  \curJ_\bftheta^\pin  = \{j_1^{\,\bftheta},\ldots , j_{\hat n}^{\,\bftheta} \}$, where the set $\curJ_\bftheta^\pin$ is defined in Section~\ref{Section4.1}.
Since the matrix $[\PP_\bftheta]\in\MM_N$,
{\color{black} defined by $[\PP_\bftheta] \, \bfA_{\bftheta_t} = (\hat\bfA_{\bftheta_t} ,\xx) = ([\hat O_\bftheta]\, \bfA_{\bftheta_t}  , [O^{\,\pin}_\bftheta]\, \bfA_{\bftheta_t})\in \RR^N$,}
is a permutation matrix, it follows that $[\PP_\bftheta]$ is orthogonal, that is, $[\PP_\bftheta]^T\, [\PP_\bftheta] =[I_N]$. Consequently,  we have $[\PP_\bftheta]^T = [\, [\hat O_\bftheta]^T \,  [O_\bftheta^{\,\pin}]^T\,]$ and
{\color{black} $\bfA_{\bftheta_t} = [\PP_\bftheta]^T \, ( \hat\bfA_{\bftheta_t}  , \xx)$}, which proves that
{\color{black}
\begin{equation}\label{eq7.3}
\bfA_{\bftheta_t}  = [\hat O_\bftheta]^T\, \hat \bfA_{\bftheta_t}  + [O^{\,\pin}_\bftheta]^T\, \xx  \, .
\end{equation}
}
Using Eqs.~\eqref{eq7.2} and \eqref{eq7.3}, the constrained Eq.~\eqref{eq7.1} leads us to  the following random neural network equation,
{\color{black}
\begin{equation}\label{eq7.4}
\hat\bfA_{\bftheta_t}  = \hat\bff ( \, [\hat\bfW_\bftheta]\, \hat\bfA_{\bftheta_t}  + \hat \bb_{\bftheta_t} ) \, ,
\end{equation}
}
where $\hat\bff$ is defined analogously to $\bff$, but acting on $\RR^{\hat n}$,
$[\hat\bfW_\bftheta]$ is the random matrix with values in $\MM_{\hat n}$, and $\hat\bb_{\bftheta_t}$ is an $\RR^{\hat n}$-valued random variable defined as,
\begin{equation}\label{eq7.5}
[\hat\bfW_\bftheta] = [\hat O_\bftheta]\, [\bfW^{\psp}_\bftheta]\, [\hat O_\bftheta]^T \quad , \quad
\hat \bb_{\bftheta_t} = [\hat O_\bftheta]\,[\bfW^{\psp}_\bftheta]\,[O^{\,\pin}_\bftheta]^T\, \xx + \hat \bfb_\bftheta \quad , \quad
\hat \bfb_\bftheta = [\hat h_\bftheta]\, \bfbeta
\quad , \quad [\hat h_\bftheta] = [\hat O_\bftheta]\,[h] \, ,
\end{equation}
where the hyperparameter $\bfbeta\in\RR^{n_b}$ (a vector-valued
{\color{black} component of $\bftheta_t$) and}
$[h]\in\MM_{N,\, n_b}$ are defined in Eq.~\eqref{eq6.1}, which expresses the low-rank representation   $\bfb = [h]\,\bfbeta$.
The random
{\color{black} output $\YY(\bftheta_t \, \vert \, \xx)$ of}
the random neural network is the $\RR^{n_\ppout}$-valued random variable given by
{\color{black}
\begin{equation} \label{eq7.6}
\YY(\bftheta_t \, \vert \, \xx) = [O_\bftheta^{\,\pout}]\, \hat\bfA_{\bftheta_t} \quad , \quad [O_\bftheta^{\,\pout}]\in\MM_{n_\ppout,\hat n} \, ,
\end{equation}
}
where each column of the $\bftheta$-dependent deterministic matrix $[O_\bftheta^\pout]$  exactly contains one $1$ (and other entries are $0$).\\

\noindent\textit{(iv) Existence, uniqueness, and numerical solution}.
Each realization $\ell$ of Eq.~\eqref{eq7.4} can be written as
{\color{black}
\begin{equation}\label{eq7.7}
\hat\bfa^\ell_{\bftheta_t}  = \hat\bff ( \, [\hat\bfw^\ell_\bftheta]\, \hat\bfa^\ell_{\bftheta_t}  + \hat \bb^\ell_{\bftheta_t} ) \, .
\end{equation}
}
The nonlinear system defined by Eq.~\eqref{eq7.7} is solved using a fixed-point iteration algorithm with under-relaxation to ensure convergence. An alternative approach could be  the Newton–Raphson method, which would require computation of the Jacobian.
This system can be reformulated as the fixed‐point problem
\begin{equation}\label{eq.7.8}
\hat\bfa = \bfF (\hat\bfa) \quad , \quad \bfF:\RR^{\hat n}\to\RR^{\hat n} \, .
\end{equation}
Here, the indices $\ell$ and $\bftheta_t$ have been suppressed
{\color{black} in $\hat\bfa^\ell_{\bftheta_t}$}
 and $\bfF$, and we define
\begin{equation}\label{eq7.9}
\bfF(\hat\bfa)= \hat\bff\,(\, [\hat\bfw^{\ell}_\bftheta] \, \hat\bfa \, + \, \hat\bb^\ell_{\bftheta_t}\,  )
\quad , \quad \hat\bfa \in\RR^{\hat n}\, .
\end{equation}
It is easy to verify that the matrix $[\hat\bfw^\ell_\bftheta]\in\MM_{\hat n}$ is  symmetric, has zero diagonal entries ($[\hat\bfw^\ell_\bftheta]_{ii} = 0$), and strictly positive off-diagonal entries ($0 < [\hat\bfw^\ell_\bftheta]_{ij} < 1$ for $i \neq j$).
The analysis of Eqs.~\eqref{eq7.4} and \eqref{eq7.5} shows that the parameter $n_b$ must satisfy $n_b \leq \hat n = N - n_\pin$.\\

\noindent\textit{(iv-1) Existence of at least one solution}.
Let $\curD = [-1,1]^{\hat n} \subset \RR^{\hat n}$.  For any $\hat\bfa\in\RR^{\hat n}$, each component of $\bfF(\hat\bfa)$ lies in  the open interval  $]-1,1[\subset [-1,1]$, so $\bfF$ defines a continuous map $F: \curD \to \curD$.  Since $\curD$ is nonempty, compact, and convex, the Brouwer fixed‐point theorem \cite{Deimling1980} guarantees the existence of
{\color{black} at least one $\hat\bfa^\ell_{\bftheta_t}\in \curD$ with
 $\bfF(\hat\bfa^\ell_{\bftheta_t})=\hat\bfa^\ell_{\bftheta_t}$}.\\

\noindent\textit{(iv-2) Fixed‐point iteration (under‐relaxation) for constructing a solution}.
Although uniqueness  does not generally hold, a scalable iterative scheme can be used. Given an initial guess  $\hat\bfa^{(0)}\in\RR^{\hat n}$ and a relaxation parameter $\alpha\in ]0,1]$,
iterate for $i=0,1,2,\dots$,
\begin{equation}\label{eq7.10}
\hat\bfa^{(i+1)} = (1-\alpha)\,\hat\bfa^{(i)} + \alpha \, \bfF(\hat\bfa^{(i)}) \, ,
\end{equation}
and stop then when $\Vert \hat\bfa^{(i+1)}-\hat\bfa^{(i)}\bigr \Vert < \varepsilon$
for a prescribed tolerance $\varepsilon > 0$.  Upon convergence, set
$\hat\bfa^\ell_\bftheta\approx\hat\bfa^{(i+1)}$, form
{\color{black} $\hat\bfa^\ell_\bftheta =(\hat a^\ell_{\bftheta,1},\ldots , \hat a^\ell_{\bftheta,\hat n})$},
and extract
{\color{black} the output  $\yy^\ell(\bftheta_t \, \vert \, \bfx)$ via}
 Eqs.~\eqref{eq7.6} and \eqref{eq7.7}.
In practice, for a relatively small number of neurons (typically a few hundred), numerical experiments indicate that the fixed-point iteration is efficient in terms of computational cost. This observation is limited to the specific test cases and computational setup considered, and should not be generalized to larger network sizes or different hardware configurations.

\section{Training the vector valued hyperparameter from Network Supervision}
\label{Section8}
Using the low-rank representation of the bias vector, which involves the hyperparameter $\bfbeta \in \RR^{n_b}$ with $n_b \ll N$ (see Section~\ref{Section6}), the total hyperparameter $\bftheta_t$ and its admissible set are defined as
\begin{equation} \label{eq8.1}
\bftheta_t = (\bftheta, \bfbeta) \in \curC_{\bftheta_t} = \curC_\bftheta \times \RR^{n_b} \, ,
\end{equation}
where $\bftheta$ and $\curC_\bftheta$ are defined in Eq.~\eqref{eq5.7}.\\

\noindent\textit{(i) Fixing a realization of the random germs generating the Gaussian latent random field and the inhomogeneous Poisson process}.
As we explained in Section~\ref{Section1}, training the neural network consists of estimating an optimal value of the hyperparameter $\bftheta_t = (\bftheta,\bfbeta)$ by minimizing the cost function $\curJ(\bftheta_t)$, defined as the negative log-likelihood function and evaluated on the input-outputs of the training dataset. For the optimization problem, based on the projected gradient descent with Adam updates and trial-based initialization, to be well-posed, we draw a single realization of each of the random vectors $\bfH$ (see Eq.~\eqref{eq2.18}) and $\bfcurU_\pelem$, $\bfcurU_a$, $\bfcurU_b$, and $\bfcurU_\pPoisson$ before starting the optimization algorithm. Thus, these realizations are not modified
during optimization and are independent of $\bftheta$ (and also of $\bfbeta$). This does not mean that the architecture of the neural network is deterministic; it remains random and changes for each value of $\bftheta$. Indeed, when $\bftheta$ is modified, the spectrum of
the covariance matrix of $\bfU^m$ changes, due to changes of
{\color{black} $h_1$, $h_2$, $\bfbeta^{\,(1)}_{n_h}$, and $\bfbeta^{\,(2)}_{n_h}$}
within $\bftheta$ (see Eq.~\eqref{eq3.1}), and the Poisson process is also modified, which implies a change in the location of the $N$ neurons. Moreover, Section~\ref{Section4} shows that the location of the inputs-outputs on the one hand, and the binary sparsity mask $[M_\bftheta^\psp]$ defined by Eq.~\eqref{eq4.12} and the edge weight matrix $[w_\bftheta^g]$
{\color{black} defined by Eq.~\eqref{eq4.10}},
on the other part, both depend on $\bftheta$.
However, as we explained in Section~\ref{Section5}-(ii), for each value of $\bftheta$ (and in particular of $\zeta_s$), we generate $n_\psim$ realizations $\{\bfeta^\ell,\ell=1,\ldots , n_\psim\}$ of $\bfH$. This allows the realizations $s_i^{\bftheta,\ell}$ of $S_i^\bftheta$ (see Eqs.~\eqref{eq5.4} and \eqref{eq5.5}) to be generated, and thus the  realizations of the random matrix $[\bfW_\bftheta^\psp]$ (see Eqs.~\eqref{eq5.2} and \eqref{eq5.3}) can be deduced.  These realizations are used to estimate the conditional probability density function that appears in the loss function.\\

\noindent\textit{(ii) Definition of the cost function}.
In the random neural architecture with random weights and deterministic bias, that we have proposed, the model output is an $\RR^{n_\ppout}$-valued random variable $\YY(\bftheta_t \vert \xx)$ depending on a given $\RR^{n_\ppout}$-valued input $\xx$  and the  hyperparameter $\bftheta_t$.  For a fixed hyperparameter value $\bftheta_t\in\curC_{\bftheta_t}$ and a given input $\xx$, the random network produces $n_\psim$ independent realizations $\{\yy^\ell(\bftheta_t \vert \xx), \ell=1,\ldots, n_\psim\}$ of $\YY(\bftheta_t \vert\xx)$.
Given a training dataset $\{(\xx^i, \yy^i), i =1,\ldots, n_d\}$, the goal is to estimate the optimal hyperparameter $\bftheta_t$ by fitting the probability distribution of the model outputs $\YY(\bftheta_t \, \vert \,\xx^i)$ to the observed outputs $\yy^i$.
A statistical efficient and widely used cost function in this setting is the negative log-likelihood (NLL), defined as,
\begin{equation}\label{eq8.2}
\curJ(\bftheta_t) = -\sum_{i=1}^{n_d} \log p^\ANNpp(\yy^i \,\vert \,\xx^i; \bftheta_t) \quad , \quad \bftheta_t\in\curC_{\bftheta_t} \, ,
\end{equation}
where  $\yy\mapsto p^\ANNpp(\yy \,\vert \,\xx; \bftheta_t)$  denotes the conditional probability density function of $\YY(\bftheta_t \,\vert \,\xx)$.
Since the network random output $\YY(\bftheta_t \,\vert\,\xx)$ is only available through sampling, its conditional probability density function can be approximated using Gaussian Kernel Density Estimation (GKDE) \cite{Bowman1997,Gentle2009,Givens2013} as
\begin{equation}\label{eq8.3}
p^\ANNpp(\yy \,\vert \,\xx; \bftheta_t) = {\color{black} \frac{1}{n_\psim}\sum_{\ell=1}^{n_\psim} }
                     \frac{1}{ (\sqrt{2\pi}\, s_\pSB)^{\,n_\ppout}\,\sigma_{\yy,1}(\bftheta_t \vert \xx)\times\ldots\times\sigma_{\yy,n_\ppout}(\bftheta_t \vert \xx) } \,
                     \exp \Bigg ( -\frac{1}{2\, s_\pSB^2} \sum_{k=1}^{n_\ppout}
                       \Bigg (\frac{ \yy_k - \yy_k^\ell(\bftheta_t \vert \xx) } { \sigma_{\yy,k}(\bftheta_t \vert \xx)} \Bigg )^2 \Bigg ) \, ,
\end{equation}
where $s_\pSB = ( 4 /( n_\psim \,(2 + n_\ppout) )  )^{ 1/(4+n_\ppout)}$ and where $\sigma_{\yy,k}(\bftheta_t \vert \xx)$ is the empirical standard deviation estimated from $\{\yy^\ell(\bftheta_t \vert \xx), \ell=1,\ldots, n_\psim\}$.
Evaluating the kernel involves computing exponential functions, which can be computationally intensive. Nevertheless, GKDE remains an accurate and widely used method in practice. Several techniques have been developed to reduce its computational cost; see, for example, \cite{Silverman1986,Scott2015} for standard references on kernel density estimation theory and practice, and \cite{Yang2003,Gray2003,Raykar2010,Obrien2016} for acceleration methods and computational considerations. These acceleration methods preserve the statistical robustness of GKDE while significantly improving its computational efficiency. One such approximation involves approximating Eq.~\eqref{eq8.3} using the product of marginals, that is, the product of univariate KDEs for each component, with each marginal using its own optimal 1D Silverman bandwidth.\\

\noindent\textit{(iii) Optimization problem}.
The estimation of an optimal value  $\bftheta_t^\optp =(\bftheta^\optp,\bfbeta^\optp)$  of the hyperparameter $\bftheta_t=(\bftheta,\bfbeta)$ is obtained by solving the optimization problem,
\begin{equation}\label{eq8.4}
\bftheta_t^\optp = \arg \min_{\bftheta_t\in \curC_{\theta_t}} \curJ(\bftheta_t) \, ,
\end{equation}
where $\curJ(\bftheta_t)$ is defined in Eq.~\eqref{eq8.2} and the admissible set $\curC_{\theta_t}$ in Eq.~\eqref{eq5.7}.
This nonconvex optimization problem cannot  be solved directly using projected gradient descent, as the optimal point generally corresponds to a local minimum that depends on the initial conditions. Consequently, hybrid optimization strategies are necessary, such as  multi-start gradient methods, basin-hopping algorithms, or Bayesian optimization. In this work, we adopt a projected gradient descent scheme with Adam updates and trial-based initialization \cite{Kingma2015,Conn2009}.
The main steps of the proposed algorithm are as follows:\\

\noindent\textit{Step 1: Definition of a grid-based admissible set for the hyperparameters $h_1$, $h_2$, and $\zeta_s$}.
To prepare the trial-based method for global optimization, we introduce the first three components of $\bftheta$, namely   $\bftheta_\ptrial = (h_1, h_2, \zeta_s)$, and define its admissible set $\curC_\ptrial$ as
\begin{equation} \label{eq8.5}
\curC_\ptrial = \{\bftheta_\ptrial\in [c_1, d_1]\times [c_2, d_2] \times [c_3, d_3]\subset \RR^3 \,\,\, \vert \,\,\,  0 < c_i < d_i < +\infty  \, , \, i=1,2,3\} \, .
\end{equation}
The constants $c_i$ and $d_i$ are prescribed.  We then define a grid $\curC_\pgrid\subset \curC_\ptrial$ consisting of $n_\ppgrid$ nodes $\bftheta_\ptrial^{\, j} = (h_1^j, h_2^j, \zeta_s^{\, j})$ for $j =1,\ldots , n_\pgrid$.\\

\noindent\textit{Step 2: Trial-based method for global optimization, used to initialize local optimization with projected gradient descent and Adam updates}. In this trial-based step, we set  $n_h =0$, so that  $\bfbeta^{\,(1)}_{n_h}$ and $\bfbeta^{\,(2)}_{n_h}$ are not involved. Therefore, this step estimates an optimal value $\bftheta_\ptrial^\optp$ for the trial hyperparameter $\bftheta_\ptrial = (h_1,h_2,\zeta_s)$.
For each point $\bftheta_\ptrial^{\, j} \in \curC_\pgrid$, we estimate a corresponding value $\bfbeta^{\, j} \in \RR^{n_b}$, as detailed below. The trial-based global optimization problem then consists of solving the discrete minimization problem,
\begin{equation}\label{eq8.6}
\bftheta_\ptrial^\optp =  \arg \min_{\bftheta_\ppptrial\in \curC_\ppgrid } \curJ(\bftheta_\ptrial) \quad , \quad
  \bftheta_\ptrial^\optp  = ( h^\optp_{1,\ptrial} \, , h^\optp_{2,\ptrial} \, , \zeta^\optp_{s,\pptrial})   \, .
\end{equation}
The value of $\bfbeta$ corresponding to the optimal value $\bftheta_\ptrial^\optp$ is denoted by $\bfbeta_{\ptrial}^\optp$.\\

\noindent\textit{Estimation of $\bfbeta^{\, j}$ for $\bftheta^{\, j}\in \curC_\pgrid$}.
Let $\{\xx^i,i=1,\ldots, n_d\}$ denote the set of input vectors from the training dataset. For notational simplicity, we write $\bftheta^{\, j}$ instead of $\bftheta_\ptrial^{\, j}$.
For each $i\in\{1,\ldots, n_d\}$, we consider the random nonlinear neural equation (see Eqs.~\eqref{eq7.4} and \eqref{eq7.5}),
\begin{equation} \label{eq8.7}
\hat\bfA^{i,j} = \hat\bff (\, [\hat\bfW^j]\, \hat\bfA^{i,j} + [\hat\bfW^{\pin,j}] \, \xx^i + \hat\bfb^j)
\quad , \quad \bftheta^{\, j}\in \curC_\pgrid
\quad , \quad j \in\{1,\ldots , n_\pgrid \}\, ,
\end{equation}
where  $[\hat\bfW^j]= [ \hat O_{ \bftheta^{\, j} } ] \,[\bfW^{\psp}_{ \bftheta^{\, j} }]\, [\hat O_{\bftheta^{\, j}}]^T$,
$[\hat\bfW^{\pin,j}] = [\hat O_{\bftheta^{\, j}}]\,[\bfW^{\psp}_{\bftheta^{\, j}}]\,[O^\pin_{\bftheta^{\, j}}]^T$,
and
{\color{black} $\hat\bfb^j =[\hat O_{\bftheta^{\, j}}]\, [h]\, \bfbeta^j$}.
For each $i$ and $j$, we first solve Eq.~\eqref{eq8.7} for $\hat\bfA^{i,j}$ by temporarily setting $\hat\bfb^j = \bfzero_{\hat n}$
and  computing $n_\psim$ independent realizations $\{ [\bfw^{\psp,\ell}_{ \bftheta^{\, j} } ] , \ell =1,\ldots, n_\psim\}$ of the random matrix $[\bfW^{\psp}_{ \bftheta^{\,j} }]$. This yields $n_\psim$ independent realizations
$\{ \hat\bfa^{i,j,\ell}_\star , \ell=1,\ldots, n_\psim \}$ of the random vector $\hat\bfA^{i,j}$, such that
\begin{equation} \label{eq8.8}
\hat\bfa^{i,j,\ell}_\star = \hat\bff (\, [\hat\bfw^{j,\ell}]\, \hat\bfa^{i,j,\ell}_\star + [\hat\bfw^{\pin,j,\ell}] \, \xx^i) \, .
\end{equation}
{\color{black} Since the vector $\hat\bfb^j \in\RR^{\hat n}$, computed from Eq.~\eqref{eq8.10}, can be expressed as $\hat\bfb^j = [\hat h_{\bftheta^{\, j}} ]\, \bfbeta^j$
where $[\hat h_{\bftheta^{\, j}} ] = [\hat O_{\bftheta^{\, j}}]\, [h]$},
we now seek the vector $\bfbeta^j\in\RR^{n_b}$ for each $j\in\{1,\ldots , n_\pgrid\}$ such that the equation
\begin{equation} \label{eq8.9}
\hat\bfa^{i,j,\ell}_\star = \hat\bff (\, [\hat\bfw^{j,\ell}]\, \hat\bfa^{i,j,\ell}_\star + [\hat\bfw^{\pin,j,\ell}] \, \xx^i +
[\hat h_{\bftheta^{\, j}} ]\, \bfbeta^{\, j} ) \, ,
\end{equation}
is satisfied as accurately as possible.  We thus formulate the following least-squares inverse problem,
\begin{equation} \label{eq8.10}
\bfbeta^{\, j} = \arg \, \min_{\bfbeta \, \in \, \RR^{n_b}} \sum_{i=1}^{n_d} \sum_{\ell=1}^{n_\psim}
\,\, \Vert \, \hat\bfa^{i,j,\ell}_\star - \hat\bff ( \bfz^{i,j,\ell} + [\hat h_{\bftheta^{\, j}} ]\, \bfbeta ) \, \Vert^2 \, ,
\end{equation}
where $\bfz^{i,j,\ell} = [\hat\bfw^{j,\ell}]\, \hat\bfa^{i,j,\ell}_\star + [\hat\bfw^{\pin,j,\ell}] \, \xx^i$. This optimization problem can be efficiently solved using a gradient-based optimizer, such as L-BFGS.
Let  $j^\optp$ be the index such that $\bftheta_\ptrial^\optp = \bftheta^{\, j^{\optpp}}$. Then
\begin{equation} \label{eq8.11}
\bfbeta_{\ptrial}^\optp = \bfbeta^{\, j^\optpp} \, .
\end{equation}
%

\noindent\textit{Step 3: Local optimization using projected gradient descent with  Adam updates and trial-based initialization}.
{\color{black}
For  the local optimization, $h_1$, $h_2$ and $\zeta_s$ are set to the values $h^\optp_{1,\ptrial}$, $h^\optp_{2,\ptrial}$, and
$\zeta^\optp_{s,\pptrial}$, respectively, as defined in Eq.~\eqref{eq8.6}. Consequently,
\begin{equation}\label{eq8.12}
\bftheta_t = ( h^\optp_{1,\ptrial}, h^\optp_{2,\ptrial}, \zeta^\optp_{s,\pptrial}, \bfbeta^{\,(1)}_{n_h}\! , \bfbeta^{\,(2)}_{n_h} \!,\bfbeta) \, .
\end{equation}
In this step, the first three components of $\bftheta_t$ are fixed, and the initialization $\bftheta_0$ of  $\bftheta_t$ for the  projected gradient descent algorithm with Adam updates for solving the optimization problem defined by Eq.~\eqref{eq8.4} is
\begin{equation}\label{eq8.13}
\bftheta_0 = ( h^\optp_{1,\ptrial}, h^\optp_{2,\ptrial}, \zeta^\optp_{s,\pptrial}, \bfzero_{n_h},\bfzero_{n_h},
 \bfbeta_{\ptrial}^\optp) \, .
\end{equation}
In practice, the gradient of the cost function $\curJ$ (defined in Eq.~\eqref{eq8.2}) with respect to $\bftheta_t$ is computed only with respect to the free hyperparameters $\bfbeta^{\,(1)}_{n_h}\! , \bfbeta^{\,(2)}_{n_h} \!,\bfbeta$.
}
Differentiability is thus necessary. Examination of Eqs.~\eqref{eq7.4} to \eqref{eq7.6} shows that, for $n_h \geq 1$,  the only source of non-differentiability lies in the matrix-valued functions  $[\hat O_\bftheta]$, $[O_\bftheta^\pin]$, and  $[O_\bftheta^\pout]$, which are not differentiable with respect to $\bfbeta^{\,(1)}_{n_h}$ and $\bfbeta^{\,(2)}_{n_h}$, due to their dependence on neuron selection (see Section~\ref{Section10.4}).
{\color{black} To ensure differentiability, in Step 3 we fix the input-output neuron selection to match that obtained at the optimal trial-based value, i.e., $\bftheta_0$.
}
This means that, to solve the nonlinear random neural network equation defined by Eqs.~\eqref{eq7.4} to \eqref{eq7.6}, we treat the matrices $[\hat O_{\bftheta_0}]$, $[O_{\bftheta_0}^\pin]$, and  $[O_{\bftheta_0}^\pout]$  as constants dependent on $\bftheta_0$ but independent of $\bftheta_t$.
However, all other quantities of the random neural network, such as the latent Gaussian field, the Poisson field, neuron locations on the manifold, connectivity structure, sparse random weight matrix, and biases, remain dependent on $\bftheta_t$.
{\color{black} As previously explained, we apply projected gradient descent using the Adam optimizer, which maintains exponentially decaying averages of first- and second-order  moments of the gradients.}
The parameters  $\beta_1^{\,\adam}$, $\beta_2^{\,\adam}$, and $\varepsilon^\adam$ are set to $0.9$, $0.998$, and $10^{-8}$, respectively, in accordance with \cite{Kingma2015}.
At each iteration, the algorithm evaluates the  loss function $\curJ(\bftheta_t)$ and estimates its gradient via central finite differences. Adam updates are applied, followed by a projection step to enforce the constraints defined by $\curC_{\theta_t}$ (see paragraph (iv) below).
The procedure  continues until convergence, monitored via stabilization of both the cost function and parameter updates. All iterates and their corresponding loss values are stored for post-processing and diagnostic purposes.
When the optimize reaches $\bftheta_t = \bftheta_t^\optp$, the cost function  attains its  minimal value among  the explored domain. We define the optimal negative log-likelihood as
\begin{equation}\label{eq8.14}
\NLL_\ANNp^\opt = -\sum_{i=1}^{n_d} \log p^\ANNpp(\yy^i \,\vert \,\xx^i; \bftheta_t^\optp) \quad , \quad \bftheta_t^\optp\in\curC_{\bftheta_t} \, .
\end{equation}

\noindent\textit{Recording the convergence and stabilization criteria}.
The iteration algorithm of the projected gradient descent with Adam updates,
{\color{black} starting from $\bftheta_0$ defined by Eq.~\eqref{eq8.13},}
generates a sequence $\{\bftheta_t^{\,(n)}\}_{n\geq 1}$ for the hyperparameter $\bftheta_t$, where $n$ denotes the iteration number. To analyze the convergence of the iteration algorithm at each iteration $n$ of the optimization loop, we introduce the loss function,
\begin{equation}\label{eq8.15}
n\mapsto \NLL_\ANNp (n) = \curJ(\bftheta_t^{\,(n)}) \quad , \quad \bftheta_t^{\,(n)}\in\curC_{\bftheta_t} \, .
\end{equation}
where $\curJ$ is defined by Eq.~\eqref{eq8.2}, and is such that $\NLL_\ANNp^\opt = \curJ(\bftheta_t^\optp)$.
During the iteration loop, it is also important to monitor the stability of the parameters $\bftheta_t$. To this end, we introduce
the normalized maximum change in the sequence $\{\bftheta_t^{\,(n)}\}_{n\geq 1}$ over a sliding window of the most recent iterations. This monitoring is also used as a convergence criterion.
At iteration $n$, consider the most recent $n_\wind$ vectors,
$\bftheta_t^{\,(n-n_\wind + 1)} , \bftheta_t^{\,(n-n_\wind + 2)} ,\ldots , \bftheta_t^{\,(n)}$.
The normalized maximum change function is then defined as,
\begin{equation}\label{eq8.16}
n \mapsto \Delta_\max (n) =  \frac{ \max_{k\in\curK} \, \Vert \, \bftheta_t^{\,(k+1)} - \bftheta_t^{\,(k)}\Vert }
                                  { \max ( 1, \Vert\, \bftheta_t^{\,(n)} \Vert ) }
                                  \quad , \quad n > n_\wind
                                  \quad , \quad \curK = \{ n - n_\wind + 1 , \ldots , n-1 \} \, .
\end{equation}
%

\noindent\textit{(iv) Specification of the constraints applied to $h_1$, $h_2$, $\bfbeta^{\,(1)}_{n_h}$, and $\bfbeta^{\,(2)}_{n_h}$ when $n_h > 0$, used in Step~3}.
For $k\in\{1,2\}$, the admissible set for the coefficient vector is $\curC_{\adp,k}$ as defined in Eq.~\eqref{eq5.8}.
{\color{black}
The hyperparameters $h_1 > 0$ and $h_2 > 0$ are thus fixed to the values $h_1^\star = h^\optp_{1,\ptrial}$ and $h_2^\star = h^\optp_{2,\ptrial}$.
}
{\color{black}
After each unconstrained update, we project $\bfbeta^{\,(k)}_{n_h}$ onto an inner approximation of $\curC_{\adp,k}$ that guarantees the constraints in Eq.~\eqref{eq5.8}. In practice, we take this inner set to be the hypercube $[-B_k,B_k]^{n_h}$:
\begin{equation}\label{eq8.17}
\widehat{\curC}_{\adp,k} =  [-B_k,B_k]^{n_h} \, \subseteq\, \curC_{\adp,k} \, ,
\end{equation}
where $B_k>0$ is chosen as
\begin{equation}\label{eq8.18}
B_k \,=\, \min_{\bfx\,\in\,\curS} \,
\frac{ \max \!\big( 0\, , \min( h_k^\star - c_\star^{(k)} \, , \, c^{(k)\star} - h_k^\star ) \big )}
     { \sum_{j=1}^{n_h} \big| \, \hh^{(k,j)}(\bfx) \, \big| } \, .
\end{equation}
If the denominator vanishes at some $\bfx$, that point imposes no constraint and can be ignored in the minimum.
In the extreme case $c_\star^{(k)} = 0$ and $c^{(k)\star} = +\infty$, Eq.~\eqref{eq8.18} reduces to
\begin{equation}\label{eq8.18a}
B_k \,=\, \min_{\bfx\,\in\,\curS} \,
\frac{ h_k^\star }
     { \sum_{j=1}^{n_h} \big| \, \hh^{(k,j)}(\bfx) \, \big| } \, .
\end{equation}
With this choice, any $\bfbeta^{\,(k)}_{n_h}\in [-B_k,B_k]^{n_h}$ satisfies the inequalities in Eq.~\eqref{eq5.8}
for all $\bfx\in\curS$.}\\

\noindent\textit{(v) Numerical noise and stability of hyperparameter optimization}.
During the optimization of hyperparameters via gradient descent, small fluctuations in the computed loss function were systematically observed, despite smooth and stable convergence of the hyperparameters. Such fluctuations originate primarily from numerical noise inherent in floating-point arithmetic, amplified significantly by evaluations of nonlinear activation function, particularly for $f(z) = tanh(z)$, and to a lesser extend for $f(z) = z/(1+\vert z\vert)$) in their saturation regions. The parallel implementation on multiple CPU cores and the utilization of GPUs exacerbate these variations due to differences in the order of floating-point arithmetic operations across computational threads or GPU kernels. Such non-deterministic ordering introduces hardware-dependent stochastic perturbations that manifest as numerical noise.
Remarkably, even though the instantaneous loss values exhibit small fluctuations, the hyperparameters themselves converge smoothly. This stability arises because the hyperparameter update rule, based on gradient computations obtained via central finite differences, inherently averages out these independent numerical perturbations. The central difference approximation, symmetric by construction, effectively reduces noise in the gradient estimate, ensuring unbiasedness of gradient computations on average, provided hyperparameter updates are sufficiently incremental.
It should be emphasized that this explanation relies upon the assumptions that the numerical noise is unbiased, independent across iterations, and small relative to the scale of hyperparameter updates. Under these conditions, gradient-based hyperparameter optimization acts implicitly as a denoising mechanism, resulting in robust convergence despite observed fluctuations in the loss function.
\section{Procedure for testing  and validating the neural network with random neuronal architecture}
\label{Section9}
%
\subsection{Estimation of $\NLL_\ANNp$ from the random $\ANN$ on the test dataset}
\label{Section9.1}
Let $\{ ( \xx^i_\testp , \yy^i_\testp) , i=1,\ldots,n_\testp\}$  be the test dataset.
For the optimal value $\bftheta_t^\optp$ estimated in Section~\ref{Section8}, we evaluate the NLL,
{\color{black} denoted by $\NLL_\ANNp^\test$, as}
\begin{equation} \label{eq9.1}
\NLL_\ANNp^\test = - \sum_{i=1}^{n_\testpp} \log \, p^\ANNpp(\yy^i_\testp \, \vert \, \xx^i_\testp \, ; \bftheta_t^\optp) \, ,
\end{equation}
where the conditional density function $p^\ANNpp$ is defined by Eq.~\eqref{eq8.3} for the optimal random neural network applied to the test dataset.
\subsection{Criterion $\criter_o$ for assessing overfitting in random neural networks}
\label{Section9.2}
Traditional statistical learning theory connects overfitting to a trade-off between model complexity and data availability.
In classical models, overfitting typically occurs when a model
{\color{black} captures spurious fluctuations}
 or "noise" in the training data, leading to poor generalization.
However, in the present framework, the situation is fundamentally different. The randomness is not observational noise, but rather an inherent structural feature of the model, resulting from the stochastic generation of network configurations. Consequently, overfitting must be interpreted through a new perspective: not as fitting noise, but as selecting hyperparameters that induce overly specialized, improbable network structures.
The optimal negative log-likelihood, $\NLL_\ANNp^\opt$, computed using Eq.~~\eqref{eq8.14} on the training dataset of size $n_d$, and the negative log-likelihood, $\NLL_\ANNp^\test$, computed on an independent test dataset of size $n_\testp$ using Eq.~~\eqref{eq9.1}, provide fundamental insights into the predictive performance and potential overfitting of the random neural network.
Given the high complexity of the stochastic neural architecture generated by hyperparameter $\bftheta$ in $\bftheta_t = (\bftheta,\bfbeta)$, each configuration can approximate complex functions (see Section~\ref{Section10}), resulting in a potentially significant discrepancy between training and testing performance due to over-specialized random architectures. To quantify this form of overfitting, distinct from classical parameter-driven overfitting, we refine the criterion $\criter_o$ as follows:
\begin{equation} \label{eq9.2}
\criter_o = \frac{1}{ \vert \NLL_\ANNp^\test \vert } \times \Bigg\vert \frac{n_\testp}{n_d} \,
\NLL_\ANNp^\opt - \NLL_\ANNp^\test \Bigg\vert \, .
\end{equation}
This criterion measures the relative discrepancy between the scaled training and test negative log-likelihood values. Because the random network has no intrinsic regularization and exhibits architectural flexibility, $\criter_o$ serves as a statistical test to evaluate whether the hyperparameter selection $\bftheta_t^\optp$ leads to over-specialized random architectures.
Under the assumption that the training and test datasets are composed of independent and identically distributed samples (as is the case here), and with a sufficiently large test set (for instance, $n_\testp = n_d / 5$), the criterion $\criter_o$ provides robust insight into generalization. A value $\criter_o \ll 1$ indicates no significant overfitting: the network generalizes well beyond the training data. In contrast, a large value of $\criter_o$ reveals overfitting due to the selection of hyperparameters that induce random network realizations overly tailored to the training set.
\subsection{Criteria  for testing the accuracy}
\label{Section9.3}
We propose testing the accuracy of the random ANN model by comparing the confidence intervals with those estimated using the training dataset, for a given probability level $p_c$. Let $\YY(\bftheta_t^\opt \,\vert\,\xx)$ be the $\RR^{n_\ppout}$-valued random output of the random $\ANN$ model for a given input $\xx\in \RR^{n_\ppin}$. For each $k\in\{1, \ldots , n_\ppout\}$, let $\CI^\ANNpp_k(\xx)$ denote the $p_c$-confidence interval of the $k$-th component $\YY_k(\bftheta_t^\opt \,\vert\,\xx)$, estimated from $n_\psim$ independent realizations $\{\yy_k^\ell(\bftheta_t^\optp \vert \,\xx), \ell=1,\ldots, n_\psim\}$. Let $\CI^\trainingpp_k(\xx)$ denote the $p_c$-confidence interval estimated using the training dataset $\{(\xx^i,\yy^i) , i=1,\ldots, n_d\}$.\\

\noindent\textit{(i) Estimation of  the $p_c$-confidence interval $\CI^\ANNpp_k(\xx)$}.
Let $\yy_k^{(1)}(\xx) \leq \ldots \leq  \yy_k^{(n_\psim)}(\xx)$ be the sorted realization  $\{ \yy_k^\ell(\bftheta_t^\optp \vert \,\xx), \ell=1,\ldots, n_\psim\}$.  The symmetric $p_c$-confidence interval $\CI^\ANNpp_k(\xx)$ around the median is given by
\begin{equation}\label{eq9.3}
\CI^\ANNpp_k(\xx) = [\alpha^\ANNpp_k(\xx) \, , \beta^{\,\ANNpp}_k(\xx)] \, ,
\end{equation}
where  $\alpha^\ANNpp_k = \yy_k^{(j_\min)}(\xx)$ with $j_\min =  \max(1,\lceil n_\psim(1-p_c)/2 \rceil)$ and  $\beta^{\,\ANNpp}_k = \yy_k^{(j_\max)}(\xx)$  with $ j_\max = \min(n_\psim,\lfloor n_\psim(1+p_c)/2\rfloor)$. This corresponds to the empirical quantile method using simulated samples from the random $\ANN$ model. For constructing the criteria $\criter_a$ in paragraph (iii) below, we introduce the vectors $\bfalpha^\ANNpp(\xx) = (\alpha^\ANNpp_1(\xx),\ldots , \alpha^\ANNpp_{n_\ppout}(\xx))\in\RR^{n_\ppout}$ and
$\bfbeta^{\,\ANNpp}(\xx) = (\beta^{\,\ANNpp}_1(\xx),\ldots , \beta^{\,\ANNpp}_{n_\ppout}(\xx))\in\RR^{n_\ppout}$.\\

\noindent\textit{(ii) Estimation of  the $p_c$-confidence interval $\CI^\trainingpp_k(\xx)$}.
The conditional probability density function $p^\trainingpp_k(\yy_k \, \vert \, \xx)$ of the real-valued random  $k$-th output component, for a given input $\xx\in \RR^{n_\ppin}$, is estimated using GKDE based on the training dataset $\{(\xx^i,\yy^i) , i=1,\ldots, n_d\}$. Let $F^\trainingpp_k(\yy_k\vert\,\xx) = \int_{-\infty}^{\yy_k}  p^\trainingpp_k(z \, \vert \, \xx)\, dz$ be the cumulative distribution function. Therefore, the symmetric $p_c$-confidence interval $\CI^\trainingpp_k(\xx)$ around the median is given by
\begin{equation}\label{eq9.4}
\CI^\trainingpp_k(\xx) = [\alpha^\trainingpp_k(\xx) \, , \beta^{\,\trainingpp}_k(\xx)] \, ,
\end{equation}
where $\alpha^\trainingpp_k(\xx)=  (F^\trainingpp_k)^{-1}((1-p_c)/2  \vert \,\xx)$ and
$\beta^{\,\trainingpp}_k(\xx) = (F^\trainingpp_k)^{-1}((1+p_c)/2  \vert \,\xx)$,
where $(F^\trainingpp_k)^{-1}$ is the inverse (quantile) function of the cumulative distribution function $F^\trainingpp_k$.
Similarly to (i) above,  we introduce the vectors
$\bfalpha^\trainingpp(\xx) = (\alpha^\trainingpp_1(\xx),\ldots , \alpha^\trainingpp_{n_\ppout}(\xx))\in\RR^{n_\ppout}$ and
$\bfbeta^{\,\trainingpp}(\xx) = (\beta^{\,\trainingpp}_1(\xx),\ldots , \beta^{\,\trainingpp}_{n_\ppout}(\xx))\in\RR^{n_\ppout}$.\\

\noindent\textit{(iii) Criteria $\criter_{\bfalpha}$ and $\criter_{\bfbeta}$}.
The criterion  is based on the relative norm error on the lower and upper bounds of the $p_c$-confidence intervals for
the random $\ANN$ model and the training dataset, evaluated over the test dataset. For the lower bound, the criterion
$\criter_{\bfalpha}$ is defined by
\begin{equation}\label{eq9.5}
\criter_{\bfalpha} = \frac{1}{n_\testp} \sum_{i=1}^{n_\testpp} \criter_{\bfalpha}(\xx^i_\testp) \quad , \quad
\criter_{\bfalpha}(\xx) =  2\,\frac{\Vert \bfalpha^\ANNpp(\xx) - \bfalpha^\trainingpp(\xx)\Vert}{\Vert \bfalpha^\ANNpp(\xx) + \bfalpha^\trainingpp(\xx)} \, ,
\end{equation}
and for the upper bound, the criterion
$\criter_{\bfbeta}$ is defined by
\begin{equation}\label{eq9.6}
\criter_{\bfbeta} = \frac{1}{n_\testp} \sum_{i=1}^{n_\testpp} \criter_{\bfbeta}(\xx^i_\testp) \quad , \quad
\criter_{\bfbeta}(\xx) =  2\,\frac{\Vert \bfbeta^{\,\ANNpp}(\xx) - \bfbeta^{\,\trainingpp}(\xx)\Vert}{\Vert \bfbeta^{\,\ANNpp}(\xx) + \bfbeta^{\,\trainingpp}(\xx)} \, .
\end{equation}
%
%
\subsection{Continuous Ranked Probability Score (CRPS)}
\label{Section9.4}
The Continuous Ranked Probability Score (CRPS) \cite{Matheson1976,Hersbach2000,Gneiting2007} is a strictly proper scoring rule that quantifies the accuracy of a predictive cumulative distribution function (CDF) by comparing it to an observed outcome. Unlike the NLL criterion, which evaluates the pointwise likelihood, the CRPS captures the entire shape of the predictive distribution, making it particularly effective for assessing non-Gaussian outputs. It complements the NLL and confidence interval-based criteria by measuring how well the full predictive distribution aligns with the observed data.
For a given input $\xx \in \RR^{n_\ppin}$ and a selected output component $k\in\{1,\ldots,n_\ppout\}$, let $\YY_k(\bftheta_t^\opt \,\vert\,\xx)$ be the $k$-th component of the random output of the $\ANN$ model and let $F_k^\ANNpp(\cdot \, \vert \, \xx)$ be its cumulative distribution function (CDF). Given an observed value $\yy_k^\obs \in \RR$, the CRPS is defined by
\begin{equation}\label{eq9.7}
\CRPS^\ANNpp_k(\xx) = \int_{-\infty}^{+\infty} \left( F_k^\ANNpp(z \, \vert \, \xx) - \Unit_{]-\infty, \yy_k^\obspp]}(z) \right)^2 dz \, ,
\end{equation}
where $\Unit_{]-\infty, y]}$ denotes the Heaviside function, equal to $1$ is $z \leq y$ and $0$ otherwise.
Since $F_k^\ANNpp$ is not available in closed form but can be sampled via $n_\psim$ independent realizations $\{\yy_k^\ell(\bftheta_t^\optp \, \vert \, \xx)\}_{\ell=1}^{n_\psim}$, we use the standard Monte Carlo approximation
\begin{equation}\label{eq9.8}
\CRPS^\ANNpp_k(\xx) \approx \frac{1}{n_\psim} \sum_{\ell=1}^{n_\psim} \left| \yy_k^\ell - \yy_k^\obs \right| - \frac{1}{2 n_\psim^2} \sum_{\ell=1}^{n_\psim} \sum_{\ell'=1}^{n_\psim} \vert \yy_k^\ell - \yy_k^{\ell'} \vert \, ,
\end{equation}
where $\yy_k^\ell = \yy_k^\ell(\bftheta_t^\optp \, \vert \, \xx)$ and $\yy_k^\obs$ is the $k$-th component of the observed output $\yy^\obs$ for the given input $\xx$.
Let $\{(\xx^i_\testp,\yy^i_\testp)\}_{i=1}^{n_\testpp}$ be the test dataset. For each test input $\xx^i_\testp$, let $\yy^{i,\obs}_k = (\yy^i_\testp)_k$ be the observed $k$-th component. Then, the average CRPS over the test dataset is defined as
\begin{equation}\label{eq9.9}
\CRPS^\ANNpp_\testp = \frac{1}{n_\testp n_\ppout} \sum_{i=1}^{n_\testpp} \sum_{k=1}^{n_\ppout} \CRPS^\ANNpp_k(\xx^i_\testp) \, .
\end{equation}
The scalar value $\CRPS^\ANNpp_\testp$ provides a global measure of the distributional accuracy of the random $\ANN$ model on the test dataset. Since $\CRPS^\ANNpp_\testp$ has no absolute reference scale, it must be interpreted relatively by comparing it to the CRPS of a nonparametric cumulative distribution function $F_k^\trainingpp(\cdot \, \vert \, \xx)$, estimated from the training dataset, and then denoted as $\CRPS^\trainingpp_\testp$.
To quantify this comparison, we introduce the reference cumulative distribution function  $F_k^\trainingpp(\cdot \, \vert \, \xx)$ defined in Section~\ref{Section9.3}-(ii), and compute the criterion
\begin{equation}\label{eq9.10}
\criter_\CRPSpp = \frac{1}{n_\testp n_\ppout} \sum_{i=1}^{n_\testpp} \sum_{k=1}^{n_\ppout}
2 \, \frac{ \left| \CRPS^\ANNpp_k(\xx^i_\testp) - \CRPS^\trainingpp_k(\xx^i_\testp) \right| }
{ \CRPS^\ANNpp_k(\xx^i_\testp) + \CRPS^\trainingpp_k(\xx^i_\testp) } \, ,
\end{equation}
which evaluates the relative discrepancy between the CRPS computed from the random $\ANN$ model and that obtained from the GKDE-based baseline.
It should be emphasized that in Eq.~\eqref{eq9.7}, the cumulative distribution function $F_k^\ANNpp(z \, \vert \, \xx)$ is not estimated using GKDE, but corresponds to the empirical distribution derived from simulated realizations of the random $\ANN$ model for a fixed value of $\bftheta_t^\optp$. This justifies the use of the closed-form estimation in Eq.~\eqref{eq9.8}. In contrast, the cumulative distribution function $F_k^\trainingpp(\cdot \, \vert \, \xx)$ used in Eq.~\eqref{eq9.10} corresponds to an unknown conditional distribution inferred from observed data and is therefore estimated via GKDE based on the training dataset.
This methodological asymmetry arises from the structural difference between the two distributions: the random $\ANN$ model defines an explicit generative mechanism for $\YY_k(\bftheta_t^\opt \,\vert \,\xx)$, allowing for direct simulation and empirical estimation, whereas the data-driven distribution $p^\trainingpp_k$ must be approximated nonparametrically due to the absence of a generative model.
It is important to emphasize that a small difference between the global scores $\CRPS^\ANNpp_\testp$ and $\CRPS^\trainingpp_\testp$ does not imply a small value of the relative discrepancy criterion $\CRPS^\ANNpp_\testp$. The latter evaluates local divergences at each test sample and output component, independently of possible cancellations at the global scale. Thus, a relatively large value of $\CRPS^\ANNpp_\testp$ indicates the presence of pointwise prediction inconsistencies, even if the global distributional accuracy of the random ANN model appears satisfactory.

\subsection{Validation using the comparison of conditional probability density functions}
\label{Section9.5}
As a complement to the criteria based on $p_c$-confidence intervals for all points of the test dataset, and since the ouuputs are not Gaussian, for a selected input $\xx^i_\testp \in\RR^{n_\ppin}$ chosen in the test dataset and for a selected  $k$-th component output  chosen among all the output components, we
compare the graph of the conditional probability density function
$y\mapsto  p^\ANNpp_k(y \, \vert \, \xx^i_\testp \, ; \bftheta_t^\optp)$
with $y\mapsto p^\trainingpp_k(y \, \vert \, \xx^i_\testp)$.
The density $p^\ANNpp_k(\cdot \, \vert \, \xx^i_\testp \, ; \bftheta_t^\optp)$ on $\RR$ is the conditional probability density function
of the $k$-th component $\YY_k(\bftheta_t^\opt \,\vert\,\xx^i_\testp)$ of the random output, estimated from $n_\psim$ independent realizations $\{\yy_k^\ell(\bftheta_t^\optp \vert \,\xx^i_\testp), \ell=1,\ldots, n_\psim\}$ generated with the random $\ANN$ for $\theta_t=\bftheta_t^\optp$.
The density $p^\trainingpp_k(\cdot \, \vert \, \xx^i_\testp)$ on $\RR$ is the conditional probability density function
of the $k$-th component  of the random output, estimated using GKDE based on the training dataset $\{(\xx^i,\yy^i) , i=1,\ldots, n_d\}$.
\section{Mathematical analysis of the  proposed random neural network}
\label{Section10}
The present section introduces the first mathematical elements that characterize the proposed stochastic neural model from a theoretical perspective. The emphasis is not on deriving exhaustive results or closed-form theorems but on establishing foundational properties that substantiate the internal coherence, expressive capacity, and well-posedness of the architecture as a probabilistic function generator. In particular, we identify and prove structural properties related to the existence and uniqueness of the latent field, the measurability of the random architecture generator, and the statistical consistency of the induced stochastic mappings. These results already support a {\color{black} preliminary characterization, in a nonparametric sense, of the expressive variability and functional diversity of the model.}
Several deeper questions, such as the full characterization of the induced function class, identifiability of the generative hyperparameters, and convergence properties of the supervised learning estimator, remain open. These are mathematically subtle and require further development. Nevertheless, the analysis presented here provides a rigorous starting point for formalizing the proposed model as a geometric and stochastic learning framework grounded in applied mathematics.
\subsection{Expressive stochastic capacity of the random neural network model}
\label{Section10.1}
We consider the supervised learning of an unknown random mapping
\begin{equation*}
\xx\mapsto  \FF(\xx):  \RR^{n_\ppin} \to L^2_\curP(\Omega, \RR^{n_\ppout})\, ,
\end{equation*}
defined on the probability space $(\Omega,\curT,\curP)$, for which only a finite training dataset $\{(\xx^i, \yy^i)\}_{i=1}^{n_d}$ is available, where each $\yy^i = \FF(\xx^i; \omega_i)$ is a realization associated with  $\omega_i \in\Omega$.
{\color{black} The random mapping}
$\FF$ itself is not known and cannot be queried outside of the given training dataset.
In Sections~\ref{Section2} to \ref{Section7}, we have constructed a family of random neural network surrogates defined by random mappings
\begin{equation*}
\left \{ \xx \mapsto \FF^\ANNpp (\xx ;\bftheta_t) : \RR^{n_\ppin} \to {\color{black} L^2_\curP(\Omega, \RR^{n_\ppout} )} \,\, , \,\, \bftheta_t = (\bftheta,\bfbeta) \in \curC_{\bftheta_t}\subset\RR^{n_{\bftheta_t}} \right \}\, ,
\end{equation*}
structured via latent random fields on a compact, boundaryless, multiply-connected manifold. Our goal is to define a notion of \textit{expressive stochastic capacity} for this model class that is independent of any specific $\FF$ and that characterizes the expressive capacity of the architecture class.\\

\noindent To this end, we define the expressive stochastic capacity of the model class
$\{ \FF^\ANNpp ( \cdot \, ; \bftheta_t) : \bftheta_t \in \curC_{\bftheta_t} \}$
 as the \textit{density} of this set in the Hilbert space,
\begin{equation*}
\HH = L^2_{P_\XX}( \RR^{n_\ppin}\, ; \, L^2_\curP(\Omega, \RR^{n_\ppout}))\, ,
\end{equation*}
where $P_\XX$ is the probability measure of the random input $\XX$. Vector space $\HH$ is the Hilbert space of measurable functions on $\RR^{n_\ppin}$ with values in $L^2_\curP(\Omega, \RR^{n_\ppout})$, square integrable with respect to $P_\XX$. Therefore, the model class is said to possess expressive stochastic capacity if, for every target mapping $\FF \in \HH$ and every $\varepsilon > 0$, there exists a parameter vector $\bftheta^\optp \in \curC_{\bftheta_t}$ such that
\begin{equation*}
\int_{\RR^{n_\ppin}} E \left \{ \, \Vert \,\FF^\ANNpp(\xx; \bftheta_t^\optp) - \FF(\xx) \, \Vert^2 \right\} \,  P_\XX(d\xx)  < \varepsilon^2.
\end{equation*}
This capacity condition characterizes the approximation power of the model class independently of any specific training dataset or optimization method. In practice, the expressive stochastic capacity of the model could be ensured by constructing $\FF^\ANNpp(\xx; \bftheta_t)$ via a chaos expansion of the form
{\color{black} $\FF^\ANNpp(\xx; \bftheta_t) = \sum_{k=1}^{+\infty} \aaeclair_k(\xx\,; \bftheta_t) \,\xi_k$,}
 where $\{\xi_k\}_k$ is a Hilbert basis in $L^2_\curP(\Omega)$, and where the coefficient functions
{\color{black} $\aaeclair_k(\cdot\, ; \bftheta_t)$ with values in $\RR^{n_\ppout}$}
 are drawn from a class of expressive stochastic function generators.
Under suitable conditions, the resulting model class forms a dense subset of $\HH$, thereby exhibiting strong representational capacity of the random neural network architecture. \\

\noindent Basing a proof on such a chaos expansion seems difficult, and since there are three main ingredients in the proposed random neural network model which, by composition, give $\FF^\ANNpp(\cdot; \bftheta_t)$, we propose to prove separately the following four properties:
{\color{black}
\begin{itemize}
\item The expressive stochastic approximation of  centered, second-order, Gaussian, real-valued random fields indexed by $\curS$, and differentiability of the associated SPDE solutions with respect to the hyperparameter vector.
\item The expressive stochastic encoding of the probability measure of the random weight matrix.
\item The differentiability of the solution of the constrained random neural network equation.
\end{itemize}
}
Each of the following three sections addresses these three points.

\subsection{Expressive stochastic approximation of centered second-order Gaussian real-valued random fields indexed by $\curS$ and  differentiability of the associated SPDE solutions with respect to the hyperparameter vector}
\label{Section10.2}
We consider the  smooth, compact, oriented, boundaryless 2-dimensional Riemannian manifold $\curS$ defined in Section~\ref{Section2.1}-(i). Let $L^2_0(\curS)$ denote the Hilbert space of square-integrable functions with zero mean on $\curS$,
\begin{equation}\label{eq10.1}
L^2_0(\curS) = \left \{ f \in L^2(\curS) ~\middle |~ \int_{\curS} f(\bfx) \, d\sigma(\bfx) = 0 \right\} \, .
\end{equation}
Let $\HH_S$ be the set of self-adjoint, positive-definite, Hilbert-Schmidt operators on $L^2_0(\curS)$, corresponding to covariance operators of centered, second-order, Gaussian,  real-valued random fields indexed by $\curS$.
For each integer $n_h \geq 1$, we consider the stochastic partial differential equation defined by Eq.~\eqref{eq2.2},
\begin{equation}\label{eq10.2}
\curL_{n_h} U_{n_h}(\bfx) = \curB(\bfx) \quad , \quad  \bfx\in\curS \, ,
\end{equation}
where $\curL_{n_h} = \tau_0 - \langle \nabla , [K_{n_h}(\bfx)] \nabla \rangle $, with $\tau_0 > 0$, $\curB$ the Gaussian white generalized random field defined in Section~\ref{Section2.1}-(ii), and $\{[K_{n_h}(\bfx),\bfx\in\curS \}$ the symmetric positive-definite tensor field on $T_\bfx \curS$, defined in Sections~\ref{Section2.1}-(iii)-(iv) and such that, for $k\in\{1,2\}$, the functions
$\{\hh^{(k,j)} , j=1,\ldots, n_h\}$ are orthonormal in $L^2_0(\curS)$, and $h_k > 0$.
Therefore, the covariance operator of random field $U_{n_h}$ can be written as
$C_{U_{n_h}} = (\curL_{n_h}^{-1})^2$.
\begin{proposition}[Expressive stochastic representation of the Gaussian latent field model within the diagonal model] \label{Proposition1}
{\color{black}
For each $k\in\{1,2\}$, let $\{ \hh^{(k,j)} \}_{j \,\geq \,1}$ be a Hilbert basis of $L^2_0(\curS)$ and let $\bfbeta^{\,(k)}_{n_h}\in\curC_{\adp ,k}$ be such that Eq.~\eqref{eq2.5} and thus  Eq.~\eqref{eq2.3} hold. Hence, the diagonal tensor fields $[K_{n_h}]=\mathrm{diag} \big(h^{(1)}_{n_h},h^{(2)}_{n_h}\big)$ are uniformly elliptic on $\curS$.
Let $[K_\star] = \mathrm{diag}\big(h^{(1)}_\star,h^{(2)}_\star\big)$ be any uniformly elliptic diagonal tensor field on $\curS$ with $h^{(k)}_\star$ belonging to the $L^\infty(\curS)$-closure of $\mathrm{span}\{\hh^{(k,j)}\}_{j\ge 1}$, and let $C_\star\in\HH_S$ denote the covariance operator of the SPDE solution associated with $[K_\star]$. Then there exist two sequences of parameter vectors $\{\bfbeta^{\,(1)}_{n_h}\}_{n_h}$ and $\{\bfbeta^{\,(2)}_{n_h}\}_{n_h}$ (defining $[K_{n_h}]$ via Eqs.~\eqref{eq2.4} and \eqref{eq2.5}) such that the corresponding covariance operators $\{C_{U_{n_h}}\}_{n_h}\subset\HH_S$ satisfy
}
\begin{equation}\label{eq10.4}
\lim_{n_h \to \infty} \Vert \, C_\star - C_{U_{n_h}}  \Vert_{{}_{\HH_S}} = 0.
\end{equation}
In particular, the model class generated by diagonal tensor fields of the
{\color{black} form \eqref{eq2.4} with \eqref{eq2.5} is}
 dense in $\HH_S$ within the diagonal subclass.
\end{proposition}
\begin{proof}
{\color{black}
Since $\curS$ is compact and without boundary, $\tau_0 > 0$, and $[K_{n_h}]$ and $[K_\star]$ are symmetric and uniformly positive-definite, the operators
$\curL_{n_h} = \tau_0 - \langle \nabla , [K_{n_h}(\bfx)] \nabla \rangle $
and
$\curL_\star = \tau_0 - \langle \nabla , [K_\star(\bfx)] \nabla \rangle $
are strongly elliptic, self-adjoint, and coercive \cite{Dautray2013}.
It follows that $\curL_{n_h}^{-1}$ and $\curL_\star^{-1}$ are compact, self-adjoint, and Hilbert-Schmidt on $L^2_0(\curS)$ \cite{Guelfand1964}, due to the compact embedding $\HH^1(\curS)$ into $L^2(\curS)$ \cite{Adams2003}.
Denote the covariances by $C_{U_{n_h}}=(\curL_{n_h}^{-1})^2$ and $C_\star=(\curL_{\star}^{-1})^2$.
By assumption, for $k\in\{1,2\}$ there exist truncated expansions $h^{(k)}_{n_h}\in\mathrm{span}\{\hh^{(k,j)}\}_{j\ge 1}$ such that $h^{(k)}_{n_h}\to h^{(k)}_\star$ in $L^\infty(\curS)$ while preserving the uniform ellipticity bounds (this is enforced by
$\bfbeta^{\, (k)}_{n_h}\in\curC_{\adp,k}$). Hence $[K_{n_h}]\to [K_\star]$ in $L^\infty(\curS)$.
Consider the bilinear forms
\begin{equation*}
a_{n_h}(u,v) = \tau_0\!\int_{\curS}u\, v \,d\sigma +\!\int_{\curS}\langle [K_{n_h}]\nabla u\,,\nabla v\rangle\,d\sigma
\quad , \quad
a_{\star}(u,v) = \tau_0\!\int_{\curS}u\, v\,d\sigma +\!\int_{\curS}\langle [K_\star] \nabla u \, ,\nabla v\rangle\,d\sigma\, .
\end{equation*}
Uniform ellipticity and $L^\infty$-convergence of coefficients yield
\begin{equation*}
\big\vert a_{n_h}(u,v)- a_{\star}(u,v)\big \vert\, \leq \, \Vert [K_{n_h}]-[K_\star]\Vert_{L^\infty}\,\Vert\nabla u\Vert_{L^2}
\,\Vert \nabla v\Vert _{L^2}\, .
\end{equation*}
Standard form-perturbation estimates imply the norm-resolvent convergence \cite{Kato1966}, hence
\begin{equation*}
\big\|\,\curL_{n_h}^{-1}-\curL_{\star}^{-1}\,\big\|_{\mathcal{L}(L^2_0,L^2_0)}\;\longrightarrow\;0.
\end{equation*}
Finally, write
$(\curL_{n_h}^{-1})^2-(\curL_{\star}^{-1})^2
=\big(\curL_{n_h}^{-1}-\curL_{\star}^{-1}\big)\curL_{n_h}^{-1}
+\curL_{\star}^{-1}\big(\curL_{n_h}^{-1}-\curL_{\star}^{-1}\big)$.
Since the product of a bounded operator with a Hilbert–Schmidt operator is Hilbert–Schmidt, we get
\begin{equation*}
\big \Vert \,C_{U_{n_h}} - C_\star\,\big \Vert_{\HH_S}\, \leq \,
\big \Vert\,\curL_{n_h}^{-1}-\curL_{\star}^{-1}\,\big \Vert \,
\Big(\,\big\Vert\,\curL_{n_h}^{-1}\,\big\Vert_{\HH_S}
 + \big\Vert\,\curL_{\star}^{-1}\,\big\Vert_{\HH_S}\Big)\, \longrightarrow \, 0.
\end{equation*}
This proves \eqref{eq10.4} within the diagonal subclass and completes the proof.
}
\end{proof}

\noindent\textit{Differentiability of the associated SPDE solutions with respect to parameters}.
The differentiability of the random field $U_{n_h}$ with respect to the parameters $(h_1,h_2,\beta_1^{(1)},\dots,\beta_{n_h}^{(1)},\beta_1^{(2)},\dots,\beta_{n_h}^{(2)})$ (see Eqs.~\eqref{eq2.5}) is guaranteed by the inverse mapping theorem in Banach spaces, due to the smooth dependence of the elliptic operator $\curL_{n_h}$ on these parameters. Indeed, standard elliptic PDE theory ensures that $\curL_{n_h}^{-1}$ is differentiable in the operator norm provided that uniform ellipticity  constraints for $h^{(k)}(\bfx)$ are strictly maintained. Consequently, the differentiability of $U_{n_h}$ directly follows.
\subsection{Expressive stochastic encoding of the probability measure of the random weight matrix}
\label{Section10.3}
As explained in Section~\ref{Section8}-(iii), the training of the hyperparameter $\bftheta_t$ is performed using the projected gradient descent with Adam updates and trial-based initialization, which yields a convergent sequence $\{\bftheta_t^{\,(n)}\}_{n\geq 1}$ for hyperparameter $\bftheta_t$, and thus induces a corresponding  sequence $\{\bftheta^{\,(n)}\}_{n\geq 1}$ for $\bftheta$.
For each fixed $\bftheta^{\,(n)}$, the sparse random matrix of the weights {\color{black} is given by $[\bfW^\psp_{\bftheta^{\,(n)}}]_{ij}= [M^\psp_{\bftheta^{\,(n)}}]_{ij}\,[\bfW_{\bftheta^{\,(n)}}]_{ij}$ as defined} in Eq.~\eqref{eq5.3}, where $[M^\psp_{\bftheta^{\,(n)}}]$ is  specified in Eq.~\eqref{eq4.12}, and $[\bfW_{\bftheta^{\,(n)}}]$ from  Eq.~\eqref{eq5.2}, is rewritten component-wise as    $[\bfW_{\bftheta^{\,(n)}}]_{ij} = [w^g_{\bftheta^{\,(n)}}]_{ij}\, F^{\,(n)}_{ij}$, where $[\bfF^{\,(n)}]$ is a  random symmetric $(N\times N)$ real matrix  with entries defined by
\begin{equation}\label{eq10.5}
F^{\,(n)}_{ij} = \exp\left(-\frac{(\Delta S^{\bftheta^{\,(n)}}_{ij})^2}{2(\zeta_s^{(n)})^2\,\sigma_{\bfS_{\bftheta^{\,(n)}}}^2}\right) \quad , \quad
F^{\,(n)}_{ij} = F^{\,(n)}_{ji} \quad , \quad F^{\,(n)}_{ij} \in \,]0 , 1] \,\, a.s. \, .
\end{equation}
In Eq.~\eqref{eq10.5}, $\zeta_s^{(n)}$ denotes the component of $\bftheta^{\,(n)}$ corresponding to  $\zeta_s$ in $\bftheta$.
Note that it is sufficient to prove the expressive stochastic encoding of the probability measure of the random weight matrix
$[\bfW_{\bftheta^{\,(n)}}]$. This is because, in the limit case of sparsity, that is, when the  global percentile threshold parameter $\tau_\pprc$, defined in Section~\ref{Section4.2.2}-(v), is set to $0$, we have $[\bfW^\psp_{\bftheta^{\,(n)}}] = [\bfW_{\bftheta^{\,(n)}}]$.
Let $[\bfF]$ be any symmetric second-order random matrix of size $(N\times N)$ with entries in $]\, 0 ,1]$ a.s.,
{\color{black}
such that for each $(i,j)$, $F_{ij}\in L^2_\curP(\Omega,\sigma(\Xi_{ij}))$ where $\sigma(\Xi_{ij})$ is the $\sigma$-algebra generated by  $\Xi_{ij}$. Then, we have to demonstrate that the stochastic model defined by Eq.~\eqref{eq10.5} constitutes  an expressive stochastic  encoding of the probability measure of $[\bfF]$.
For notational simplicity, we rewrite $\Delta S^{\bftheta^{\,(n)}}_{ij}$  as $Y_{ij}^{(n)}$.
}
%
\begin{proposition}[Expressive stochastic encoding of the probability measure of the random matrix weights]
\label{Proposition2}
Let $Y_{ij}^{(n)} = U^m(\bfx^i_{\bftheta^{\,(n)}}; \bftheta^{\,(n)} )$ $-  U^m(\bfx^j_{\bftheta^{\,(n)}}; \bftheta^{\,(n)})$ be the centered, second-order,  Gaussian, real-valued random variable (see Section~\ref{Section5}-(i)) and let $\sigma^{(n)}_{ij}$ be its standard deviation. Therefore, the real-valued random variable $\Xi_{ij} = Y_{ij}^{(n)} / \sigma^{(n)}_{ij}$ is centered, second-order, Gaussian, with unit variance, $E\{\Xi_{ij}^2\}=1$,  and the collection $\{\Xi_{ij}\}_{i,j=1}^N$ are jointly dependent.
For each $(i,j)$, the normalized Hermite polynomials $\{\curH_\alpha(\Xi_{ij})\}_{\alpha\geq 0}$ form a complete orthonormal system of $L^2_\curP(\Omega,\sigma(\Xi_{ij}))$.
Let $b^{(n)}_{ij}$ and $a^{(n)}_{ij}$ be positive scalars such that
\begin{equation}\label{eq10.6}
b^{(n)}_{ij} = a^{(n)}_{ij} \, / \, (\zeta_s^{(n)} )^2   > 0 \quad , \quad
a^{(n)}_{ij} =  (\sigma^{(n)}_{ij})^2 \, / \, (2 \,\sigma^2_{\bfS_{\bftheta^{\,(n)}}} )  > 0 \, .
\end{equation}
Let $\curH_\alpha(\xi)$ denote the normalized Hermite polynomials on $\RR$.
Then each random variable $F^{(n)}_{ij}$ defined by Eq.~\eqref{eq10.5}, is written as $F^{(n)}_{ij} = {\exp}(-b^{(n)}_{ij}\, \Xi_{ij}^2)$, and admits the polynomial chaos expansion,
\begin{equation}\label{eq10.7}
F^{\,(n)}_{ij} = \sum_{\alpha=0}^{\infty} f_{ij,\, \alpha}^{(n)} \,\,\curH_\alpha(\Xi_{ij}) \, ,
\end{equation}
where the coefficients are explicitly given by
\begin{equation}\label{eq10.8}
 f_{ij\, ,\, \alpha}^{(n)} = \frac{ \sqrt{ (2\alpha)! } }{\alpha!} \, \frac{ (-b^{\,(n)}_{ij})^\alpha }
                                                         { ( 1+2 b^{\,(n)}_{ij} )^{\,\alpha +1/2} } \, .
\end{equation}
If the sequence $b^{\,(n)}_{ij} > 0$ is such that the set of coefficient sequences $\{f^{(n)}_{ij,\alpha}\}_{\alpha \in \NN}$ indexed by $n \in \NN$ is dense in $\ell^2(\NN)$, i.e.,
\begin{equation}\label{eq10.9}
\forall \bfc = \{c_\alpha \}_{\alpha\in\NN } \in \ell^2(\NN)\, ,\, \forall \varepsilon > 0\, ,\, \exists \, n \in \NN \,\, \text{ such that } \sum_{\alpha=0}^{\infty} ( f^{(n)}_{ij\, ,\,\alpha} - c_\alpha )^2 < \varepsilon^2 \, ,
\end{equation}
then the stochastic model defined by Eq.~\eqref{eq10.5}
{\color{black}
constitutes an expressive stochastic encoding of the probability law within the entrywise subclass $\{[F]:\,F_{ij}\in L^2_\curP(\Omega,\sigma(\Xi_{ij})) \, , \,  \forall i,j\}$.
}
In particular, the sequence  $\{[\bfF^{(n)}]\}_{n\geq 1}$, generated by the stochastic model, is such that
{\color{black}
\begin{equation}\label{eq10.10}
\exists\,\{n_q\}_{q\geq 1}\ \text{strictly increasing such that}\ \
\lim_{q\to+\infty}E\{ \, \Vert \, [\bfF^{(n_q)}] - [\bfF] \,\Vert_F^2 \}= 0 \, .
\end{equation}
}
\end{proposition}
%
\begin{proof}
The proof is developed in three steps.\\

\noindent\textit{Step 1}. Let $b^{\,(n)}_{ij}$ and $\Xi_{ij}$ defined in the Proposition.
Then, Eq.~\eqref{eq10.5} can be rewritten as
\begin{equation}\label{eq10.11}
F^{\,(n)}_{ij} = \exp \left ( - b^{(n)}_{ij} \, \Xi_{ij}^2 \right )  \, .
\end{equation}

\noindent\textit{Step 2}. For $\alpha\in\NN$, let $\curH_\alpha(\xi)$ be the normalized Hermite polynomials on $\RR$, defined by the orthonormality relation,
\begin{equation*}
E\{\curH_\alpha (\Xi_{ij})\, \curH_\beta (\Xi_{ij})\} = \int_\RR \curH_\alpha (\xi)\, \curH_\beta (\xi)
 \, \frac{1}{\sqrt{2\pi}} \exp(-\xi^2/2) \, d\xi  = \delta_{\alpha\beta} \quad , \quad \curH_0(\xi) = 1\, .
\end{equation*}
Since $\Xi_{ij}$ is standard Gaussian, and $F^{(n)}_{ij} \in L^2_\curP(\Omega,\sigma(\Xi_{ij}))$, the Hermite chaos expansion of $F^{(n)}_{ij}$ is given (see for instance \cite{Cameron1947,Xiu2002,Soize2004a}), by Eq.~\eqref{eq10.7}
where the coefficients, calculated by $f_{ij,\, \alpha}^{(n)} = E\{ F^{\,(n)}_{ij}\,\curH_\alpha(\Xi_{ij})\}$, are given by
Eq.~\eqref{eq10.8}.
{\color{black}
Since $0< F^{\,(n)}_{ij}\le 1$ a.s. by \eqref{eq10.5}–\eqref{eq10.11}, we have
}
\begin{equation}\label{eq10.12}
E\{ \, \Vert \, [\bfF^{\,(n)}] \,  \Vert^2_F \, \} = \sum_{i=1}^N \sum_{j=1}^N E \{ ( F^{\,(n)}_{ij})^2\} \,  < +\infty \, ,
\end{equation}
which shows that $[\bfF^{\,(n)}]$ is a second-order random matrix.
The random matrix $[\bfF^{\,(n)}]$ admits the matrix-valued chaos expansion,
\begin{equation}\label{eq10.13}
[\bfF^{\,(n)}] = \sum_{\alpha=0}^{\infty} [f^{(n)}_\alpha] \circ [\bfcurH_\alpha] \, ,
\end{equation}
where $\circ$ denotes the Hadamard (entrywise) product,  $[f^{(n)}_\alpha]\in\MM_N$ has entries
$[f^{(n)}_\alpha]_{ij} =  f_{ij\, , \, \alpha}^{(n)}$, and $[\bfcurH_\alpha]$ is the random matrix with entries
$[\bfcurH_\alpha]_{ij} = \curH_\alpha(\Xi_{ij})$. Due to Eq.~\eqref{eq10.12}, the  series defined by Eq.~\eqref{eq10.13} converges in $L^2_\curP(\Omega, \MM_N)$.\\

\noindent\textit{Step 3}.  If the family of sequences $\{f^{(n)}_{ij\, ,\, \alpha}\}_{n\in\NN}$ is dense in $\ell^2(\NN)$, then any square-summable sequence of Hermite coefficients can be approximated arbitrarily well, which implies,
\begin{equation*}
\forall F_{ij} \in L^2_\curP(\Omega,\sigma(\Xi_{ij})),\, \exists\, n \,\,\text{ such that } \,\,
{\color{black} \Vert F^{(n)}_{ij} - F_{ij} \Vert_{L^2_\curP}} < \varepsilon \, .
\end{equation*}
{\color{black}
Therefore, since for each $(i,j)$ the chaos basis $\{\curH_\alpha(\Xi_{ij})\}_{\alpha\ge 0}$ spans $L^2_\curP(\Omega,\sigma(\Xi_{ij}))$ and the coefficient family is dense by \eqref{eq10.9}, there exists a strictly increasing subsequence $\{n_q\}_q$ such that $[\bfF^{(n_q)}]$ converges in $L^2_\curP$ to any target $[\bfF]$ with $F_{ij}\in L^2_\curP(\Omega,\sigma(\Xi_{ij}))$ and entries in $]\, 0\, ,1]$ a.s., which proves \eqref{eq10.10}.
}
\end{proof}

\noindent\textit{Remark 4. Constructive choice of the sequence $b^{(n)}_{ij}$ supporting expressive stochastic capability}.
In this remark, we clarify a constructive  choice of the sequence $\{ b^{(n)}_{ij} \}_{n\in\NN}$  that supports  the expressive stochastic capability condition of Proposition~\ref{Proposition2}. The objective is to provide an explicit rule ensuring that the coefficient sequences  $\{ f^{(n)}_{ij\,,\,\alpha} \}_{\alpha \in \NN }$, defined
by Eq.~\eqref{eq10.8},
{\color{black}
belong to the Hilbert space $\ell^2(\NN)$ and exhibit fast tail decay.
}
For sufficiently large $n$, a suitable constructive choice for the sequence $b^{(n)}_{ij}$ is
\begin{equation} \label{eq10.14}
b^{(n)}_{ij} \sim \gamma\, n^{-r} \quad , \quad r > 0 \quad , \quad 0 < \gamma < +\infty\, .
\end{equation}
With this explicit definition, the Hermite chaos expansion coefficients from Eq.~\eqref{eq10.8} become
\begin{equation} \label{eq10.15}
f^{(n)}_{ij\, ,\,\alpha} \sim \frac{1}{\sqrt{\gamma}}\frac{\sqrt{(2\alpha)!}}{\alpha!}  \frac{(-1)^\alpha}{n^{r\alpha} \, (1/\gamma + 2/n^r)^{\alpha + 1/2}} \, ,
\end{equation}
Each sequence $\{ f^{(n)}_{ij\, ,\, \alpha} \}_{\alpha\in\NN}$ is clearly square-summable, hence belongs to $\ell^2(\NN)$.
The factorial growth in the numerator
{\color{black}
is dominated by the exponential-type decay in the denominator as $\alpha\to\infty$ (for fixed $n$), which guarantees convergence of $\sum_{\alpha\ge 0} |f^{(n)}_{ij\, ,\, \alpha}|^2$.
Moreover, as $n$ varies, the family $\big\{\{ f^{(n)}_{ij\, ,\, \alpha}\}_{\alpha\in\NN}\big\}_{n\in\NN}$ is nontrivial and provides a flexible range of coefficient profiles (not confined to a single fixed finite-dimensional subspace), which is sufficient for the entrywise $L^2$-approximation scheme used in Proposition~\ref{Proposition2}.
In summary, the choice $b^{(n)}_{ij} = \gamma\,n^{-r}$ supplies a deterministic, explicit rule yielding square-summable Hermite coefficients with fast tail decay and parameter continuity, thereby supporting the expressive stochastic encoding employed in Proposition~\ref{Proposition2}.
}.\\

\noindent\textit{Remark 5. Behavior of the hyperparameter $\zeta_s^{(n)}$}.
We examine a particular asymptotic behavior of $a^{(n)}_{ij}$ based on the Gaussian structure of the random vector $\bfU^{(n)} = (U^{(n)}_1, \ldots, U^{(n)}_N)$, where $U^{(n)}_j = U^m(\bfx^j_{\bftheta^{\,(n)}}; \bftheta^{\,(n)})$. Let $[C_\bfU^{(n)}]$ denote its covariance matrix. Since $\bfU^{(n)}$ is centered, we have
$(\sigma_{ij}^{(n)})^2 = \mathrm{Var}(U_i^{(n)} - U_j^{(n)}) = [C_\bfU^{(n)}]_{ii} + [C_\bfU^{(n)}]_{jj} - 2 \,[C_\bfU^{(n)}]_{ij}$, and
therefore,
\begin{equation*}
a^{(n)}_{ij} = \frac{1}{2\,\sigma^2_{\bfS_{\bftheta^{\,(n)}}}}   [C_\bfU^{(n)}]_{ii} + [C_\bfU^{(n)}]_{jj} - 2 \,[C_\bfU^{(n)}]_{ij}
                    \, .
\end{equation*}
If $[C_\bfU^{(n)}]_{ii} \to C_\infty >$ and $[C_\bfU^{(n)}]_{ij} \to C_\infty$ for all $i \ne j$ as $n \to +\infty$, then $a_{ij}^{(n)} \to 0$ as $n \to +\infty$.
This condition means that the covariance matrix $[C_\bfU^{(n)}]$ tends toward the rank-one matrix $C_\infty\, \mathbf{1}\otimes \mathbf{1}$, where all entries converge to the same positive constant. Equivalently, the differences $U_i^{(n)} - U_j^{(n)}$ converge to zero
{\color{black} in $L^2_\curP$}, indicating a perfectly correlated limit.
If such a condition holds, then
\begin{equation}
(\sigma_{ij}^{(n)})^2 \to 0 \quad \Rightarrow \quad a_{ij}^{(n)} \to 0 \quad \text{as }\,\, n \to +\infty \, .
\end{equation}
Thus, assuming that $a_{ij}^{(n)} \sim n^{-p}$ as $n \to +\infty$ with $p > r$, we obtain
$(\zeta_s^{(n)})^2 \sim \gamma^{-1} \, n^{r-p} \to 0$ as $n \to +\infty$,
which implies that $\zeta_s^{(n)}$ should be a decreasing function of $n$.
This shows that, in the asymptotic regime where the covariance matrix $[C_\bfU^{(n)}]$ degenerates to a constant rank-one structure, the condition for expressive stochastic capability (i.e., the decay of $b_{ij}^{(n)}$) leads naturally to a vanishing $\zeta_s^{(n)}$. The practically observed convergence behavior of the optimization algorithm is consistent with a decreasing sequence $\zeta_s^{(n)}$.
\subsection{Differentiability of the solution of the constrained random neural network equation}
\label{Section10.4}
{\color{black} As explained in Step~3} of Section~\ref{Section8}-(iii), the use of the local optimization based on projected gradient descent requires the differentiability  of the solution of the constrained random nonlinear neural network equation.
Here, we investigate the differentiability, with respect to the parameter vector $\bftheta$, of the random vector $\hat\bfA_\bftheta$, defined implicitly by Eqs.~\eqref{eq7.4} to \eqref{eq7.6}, as the solution of the nonlinear random equation,
\begin{equation*}
\hat\bfA_\bftheta = \hat\bff ( \, [\hat\bfW_\bftheta]\,\hat\bfA_\bftheta + \hat\bb_\bftheta )\, ,
\end{equation*}
where $\hat\bff$ is the element-wise hyperbolic tangent function, $\hat\bff(\bfz)=\tanh(\bfz)$.
As stated
{\color{black} in Step~3 of} Section~\ref{Section8}-(iii), we treat the matrices $[\hat O_{\bftheta_0}]$, $[O_{\bftheta_0}^\pin]$, and  $[O_{\bftheta_0}^\pout]$  as fixed (that is, independent of $\bftheta_t$), where $\bftheta_0$ is defined in Eq.~\eqref{eq8.13}. Therefore, we have,
\begin{equation*}
[\hat\bfW_\bftheta] = [\hat O_{\bftheta_0}]\, [\bfW^{\psp}_\bftheta]\, [\hat O_{\bftheta_0}]^T \quad , \quad
\hat \bb_\bftheta = [\hat O_{\bftheta_0}]\,[\bfW^{\psp}_\bftheta]\,[O^{\,\pin}_{\bftheta_0}]^T\, \xx + \hat \bfb_\bftheta \quad , \quad
\hat \bfb_\bftheta = [\hat h_{\bftheta_0}]\, \bfbeta
\quad , \quad [\hat h_{\bftheta_0}] = [\hat O_{\bftheta_0}]\,[h] \, .
\end{equation*}
Let us introduce the implicit function,
\begin{equation*}
\bfg(\hat\bfA_\bftheta,\bftheta) = \hat\bfA_\bftheta - \hat\bff(\,[\hat\bfW_\bftheta]\,\hat\bfA_\bftheta + \hat\bb_\bftheta ) \, ,
\end{equation*}
which defines the dependence of $\hat\bfA_\bftheta$ on $\bftheta$ through  the condition $\bfg(\hat\bfA_\bftheta,\bftheta) = \bfzero$.
Since $\hat\bff=\tanh$ is infinitely differentiable, the differentiability of $\hat\bfA_\bftheta$ depends entirely on the invertibility of the Jacobian matrix of $\bfg$ with respect to $\hat\bfA_\bftheta$.
The Jacobian $[\bfJ_\bfg]\in\MM_{\hat n}$ is given by
\begin{equation*}
[\bfJ_\bfg] = \frac{\partial\bfg}{\partial\hat\bfA_\bftheta} = [I_{\hat n}] - [\bfD_{\,\hat\bff}]\,[\hat\bfW_\bftheta] \, ,
\end{equation*}
where $[\bfD_{\,\hat\bff}]$ is the diagonal matrix containing the derivatives of $\hat\bff$, defined by
\begin{equation*}
[\bfD_{\,\hat\bff}]_{ii} = 1 - \tanh^2 ( (\, [\hat\bfW_\bftheta]\,\hat\bfA_\bftheta + \hat\bb_\bftheta)_i ), \quad i=1,\dots,\hat n \, .
\end{equation*}
This makes $[\bfD_{\,\hat\bff}]$ a strictly positive diagonal matrix, with entries satisfying $0 < [\bfD_{\,\hat\bff}]_{ii}\leq 1$ for all $i$.
Hence, the Jacobian $[\bfJ_\bfg]$ is invertible if and only if the spectral radius satisfies
\begin{equation*}
\rho (\, [\bfD_{\,\hat\bff}]\,[\hat\bfW_\bftheta]\,)\, < \,1 \, .
\end{equation*}
This condition ensures local uniqueness and differentiability via the implicit function theorem.
Given the structure of $[\hat\bfW_\bftheta]$ (symmetric, with zero diagonal and strictly positive off-diagonal entries), this spectral condition is typically satisfied under mild assumptions on the parameter bounds.
By invoking the implicit function theorem, we conclude that $\hat\bfA_\bftheta$ is locally differentiable with respect to $\bftheta$.
Moreover, the theorem provides the explicit expression for the derivative,
\begin{equation*}
\frac{\partial \hat\bfA_\bftheta}{\partial \bftheta} = -[\bfJ_\bfg]^{-1}\,\frac{\partial \bfg}{\partial \bftheta}
\quad , \quad
\frac{\partial \bfg}{\partial \bftheta} = -[\bfD_{\,\hat\bff}]\, \left(\frac{\partial [\hat\bfW_\bftheta]}{\partial \bftheta}\hat\bfA_\bftheta + \frac{\partial \hat\bb_\bftheta}{\partial \bftheta}\right)\,.
\end{equation*}
Since both $[\hat\bfW_\bftheta]$ and $\hat\bb_\bftheta$ are differentiable with respect to $\bftheta$,
it follows that $\hat\bfA_\bftheta$ is smoothly differentiable with respect to $\bftheta$.\\

\noindent\textit{Remark 6. Probabilistic interpretation and existence of solutions}.
Since $\hat\bfA_\bftheta$ is a random vector, the differentiability with respect to $\bftheta$ described above holds in an almost-sure sense. Consequently, the derivatives obtained via the implicit function theorem are themselves random variables, allowing for the study of statistical quantities.
The differentiability analysis relies on the existence of at least one solution, which is guaranteed by the Brouwer fixed-point theorem, as discussed in Section~\ref{Section7}-(iv-1).\\
%
\section{Numerical Illustration: Random neural architecture on a torus as a compact, boundaryless, multiply-connected  manifold}
\label{Section11}
%
This numerical illustration is not intended to demonstrate superior predictive accuracy over existing models, but to illustrate the expressive flexibility and stochastic behavior induced by the latent generative architecture. The examples are selected to show how topological and anisotropic parameters affect the distributional properties of the outputs, consistent with the model goal to internalize structural uncertainty rather than minimize empirical error.
\subsection{Definition of a torus as the compact, boundaryless, multiply-connected 2D manifold $\curS$ immersed in $\RR^3$}
\label{Section11.1}
Let $\curS \subset \RR^3$ be a torus parameterized by $\bfx(\uc,\vc) = (x_1(\uc,\vc),x_2(\uc,\vc),x_3(\uc,\vc))$, with
\begin{align*}
x_1(\uc,\vc) &= (R + r \cos \vc)\cos \uc \, ,\\
x_2(\uc,\vc) &= (R + r \cos \vc )\sin \uc \, ,\\
x_3(\uc,\vc) &= r \sin \vc  \, ,
\end{align*}
for $\uc$ and $\vc$ in $[0, 2\pi]$, where the major radius is $R = 2$ and the minor radius is $r = 0.7$. The surface $\curS$ is meshed using the finite element method, as explained in Section~\ref{Section2.1}-(v), with $80$ subdivisions along the major circle and $24$ subdivisions along the minor circle. This results in a regular grid with $n_o= 1920$ nodes and $n_\pelem = 3840$ triangular elements with three nodes each, forming a structured mesh $\curS_h$ that approximates the compact, boundaryless, multiply-connected manifold $\curS$, as shown in Fig.~\ref{figure1}.
\begin{figure}[H] 
\centering
\includegraphics[width=9.0cm]{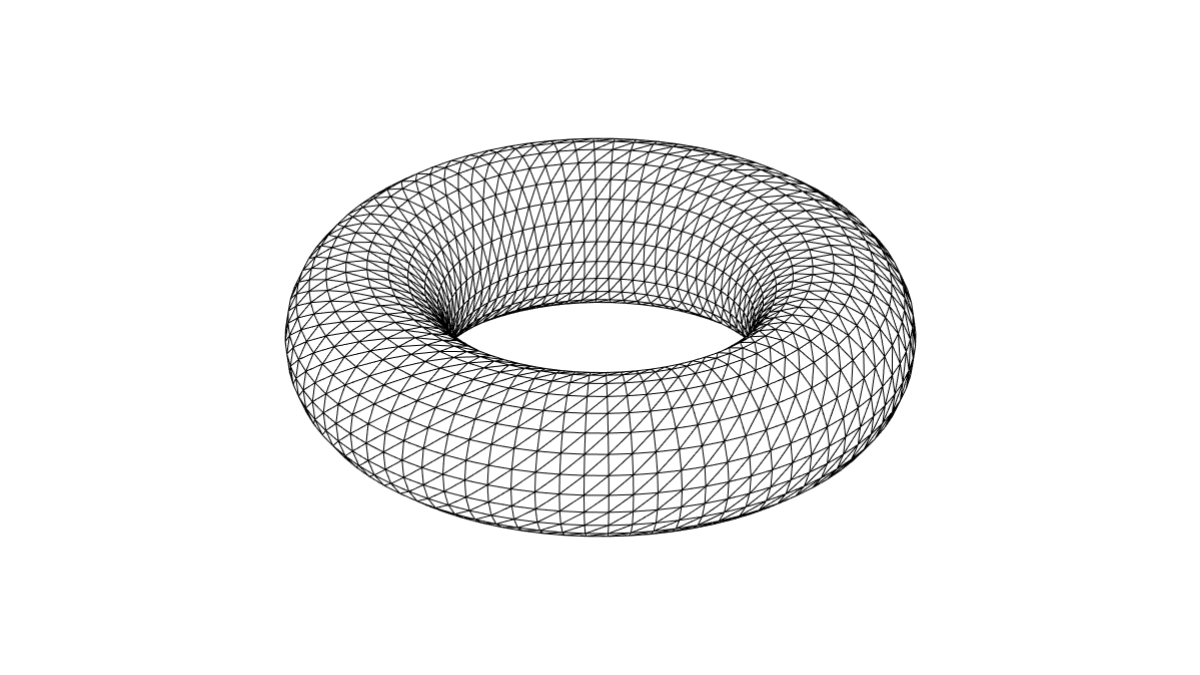}
\caption{Structured mesh $\curS_h$ approximating the compact, boundaryless, multiply-connected manifold $\curS$ (hidden-line removed).}
\label{figure1}
\end{figure}
\subsection{Construction of the functions $\hh^{(k,j)}$ on $\curS$}
\label{Section11.2}
Two models for the Hilbert basis in $L^2(\curS)$ are considered: a non-smooth differentiable model and a smooth differentiable ones.
It should be noted that
{\color{black} from the standpoint of mathematical analysis, it is sufficient that the functions $\bfx\mapsto h_{n_h}^{(1)}(\bfx)$ and
$\bfx\mapsto h_{n_h}^{(2)}(\bfx)$ used in Eq.~\eqref{eq2.4} satisfy Eq.~\eqref{eq2.5}}
(which will be the case  for the non-smooth differentiable functions described below). Therefore, the smooth differentiable model is not strictly necessary.
For the numerical illustration presented here, we adopt the non-smooth differentiable model. Nonetheless,  we briefly compare the results obtained with those obtained using the smooth differentiable model. These two models are defined below.
Let {\color{black} $\{\bfx_o^i , i =1,\ldots, n_o\}$} denote the set of nodes of the triangulated finite element mesh $\curS_h$ of $\curS$, where each node {\color{black} $\bfx_o^i$} corresponds to the parameter values $(\uc_i,\vc_i)$, with $\uc_i, \vc_i \in [0, 2\pi]$ such that {\color{black} $\bfx_o^i = \bfx(\uc_i,\vc_i)$}. For the construction of random neural architectures on $\curS$, the discretization of the functions
$\{\bfx\mapsto \hh^{(k,j)}(\bfx), k=1,2; j=1,\ldots, n_h\}$ (introduced in Section~\ref{Section2.1}-(iv)) at the mesh nodes
{\color{black} $\{ \bfx_o^i , i =1,\ldots, n_o \}$} of $\curS_h$, is represented by the matrices $[\hh_1]$ and $[\hh_2]$ in $\MM_{n_o,n_h}$. These matrices  must satisfy the following properties related to Eq.~\eqref{eq2.6}, for $k\in\{1,2\}$,
\begin{equation}\label{eq11.1}
[\hh_k]^T\, [\hh_k] = [I_{n_h}] \quad , \quad \frac{1}{n_o}\sum_{i=1}^{n_o} [\hh_k]_{ij} = 0 \,\, , \,\, j=1,\ldots, n_h  \, .
\end{equation}

\noindent\textit{(i) Non-smooth differentiable case}.
The procedure is as follows:
Generate two random matrices with independent standard normal entries. Construct a projection that removes the mean from each column, ensuring each column is mean-free. Apply this projection to the matrices, and orthonormalize the resulting mean-free columns using an economy-size QR decomposition. This produces two independent sets of orthonormal, mean-free vectors satisfying Eq.~\eqref{eq11.1}.\\

\noindent\textit{(ii) Smooth differentiable case}.
For the smooth differentiable case, we construct these bases using trigonometric products,
leveraging the natural periodicity inherent to the torus parameterization.
Specifically, for integers $m_\tor, n_\tor \geq 1$, we define the set of functions
\begin{equation*}
\phi_{m_\tor,n_\tor}^{(\ell)}(u,v) =
\begin{cases}
 c\,  \cos(m_\tor\, \uc)\, \cos(n_\tor\, \vc), & \ell = 1, \cr
 c\,  \sin(m_\tor\, \uc)\, \cos(n_\tor\, \vc), & \ell = 2, \cr
 c\,  \cos(m_\tor\, \uc)\, \sin(n_\tor\, \vc), & \ell = 3, \cr
 c\,  \sin(m_\tor\, \uc)\, \sin(n_\tor\, \vc), & \ell = 4, \cr
\end{cases}
\end{equation*}
where the constant $c$ is chosen to satisfy the normalization condition
where $\int_0^{2\pi} \int_0^{2\pi} \vert \phi_{m_\tor,n_\tor}^{(\ell)}(\uc,\vc)\vert^2 d\uc\, d\vc  = 1$, yielding $c = 1/\sqrt{\pi}$.
For each node $i$, we evaluate these functions at $(\uc_i,\vc_i)$ to build columns of the matrices. The number $n_h$ fixes the number of basis functions, and the pairs $(m_\tor,n_\tor)$ are systematically incremented, along with the index $\ell$, until $n_h$ columns are assembled.
To remove the mean (enforce zero average) for each basis function over the mesh nodes, each column is centered by subtracting its empirical mean,
\begin{equation*}
\forall j \in \{1,\dots,n_h\}, \qquad
[\hh_1]_{:,j} \gets [\hh_1]_{:,j} - \frac{1}{n_o}\sum_{i=1}^{n_o} [\hh_1]_{ij} \, .
\end{equation*}
The same procedure is applied for $[\hh_2]$, but starting with a different ordering or phase (i.e., beginning with $\sin(m_\tor\, \uc)\, \cos(n_\tor\, \vc)$).
For $k\in\{1,2\}$, to ensure orthonormality of the columns with respect to the discrete inner product over the mesh nodes, we apply an economy-size QR decomposition,
$[\hh_k] = [Q_k] \, [R_k]$, and retain only the factor $[Q_k]$. Thus, Eq.~\eqref{eq11.1} is satisfied.
\subsection{Values of the fixed parameters for the latent Gaussian random field indexed by the manifold $\curS$}
\label{Section11.3}
The value of $\tau_0$, used in Eq.~\eqref{eq2.2}, is set to $1$.
For training the random neural network, the number $n_\psim$ of independent realizations $\{\bfeta^\ell , \ell=1,\ldots ,n_\psim\}$ of the $\RR^m$-valued random variable $\bfH$ defined by Eqs.~\eqref{eq2.15} and \eqref{eq2.16}, is set to $100$.
The value of $m$, defined in Section~\ref{Section2.1}-(viii), is fixed to $100$.
To illustrate the covariance eigenvalue problem associated with the random field $U^m$, we consider the following value of the hyperparameter  $\bftheta$, as defined in Eq.~\eqref{eq5.7}, with $h_1= 0.087$, $h_2 = 0.058$, $\zeta_s = 0.0668$, and $\bfbeta^{(1)}_{n_h} = \bfbeta^{(2)}_{n_h} = \bfzero$ (which corresponds to $n_h=0$).
Fig.~\ref{figure2a} shows the graph $\alpha\mapsto\lambda_\alpha$ on a $\log_{10}$ scale, representing the eigenvalues of the covariance matrix $[C_\bfU]$ defined by Eq.~\eqref{eq2.13}, while  Fig.~\ref{figure2b} shows  the graph $m\mapsto\err_\pPCA(m)$ defined by Eq.~\eqref{eq2.17}. Regarding the value of $m$ and the error function induced by the truncation, we refer the reader to the analysis presented in Section~\ref{Section2.1}-(viii). In fact, choosing $m=100$ is not an approximation in the context of the proposed
random neural architecture; rather, it is a construction that enables control over statistical fluctuations in the latent random field.
The choice of $m$ must also be correlated with the number $N$ of neurons. For $N = 200$, which is the configuration we detail in this application, the choice $m = 100$ is appropriate. We conducted a parametric study and observed that the performance of supervised learning degrades when $m \ll 100$, but remains stable for $m > 100$.
\begin{figure}[H]
    \centering
    \begin{subfigure}[b]{0.40\textwidth}
        \centering
        \includegraphics[width=\textwidth]{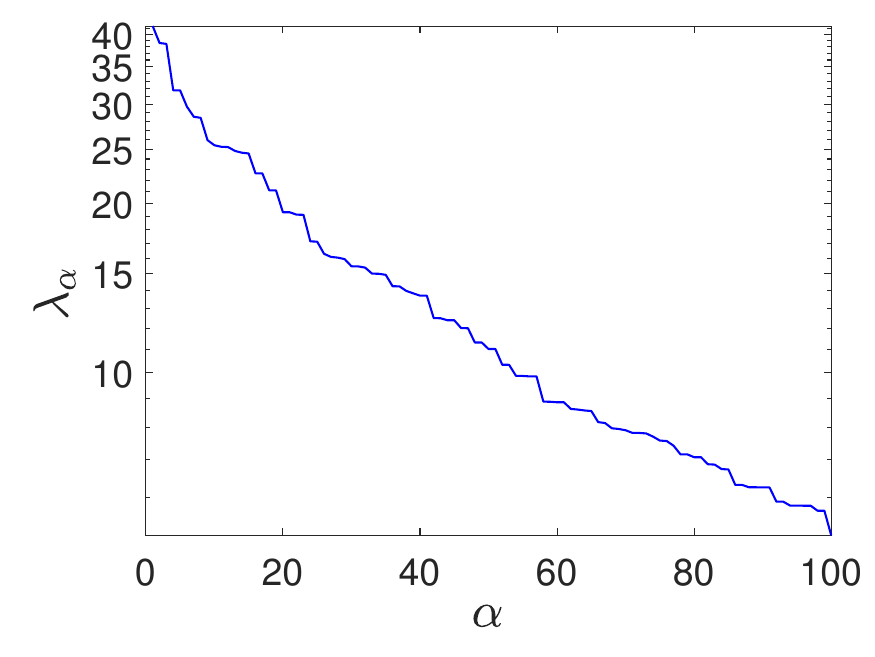}
        \caption{Eigenvalues of matrix $[C_\bfU]$ on a  $\log_{10}$ scale.}
        \label{figure2a}
    \end{subfigure}
    \begin{subfigure}[b]{0.40\textwidth}
        \centering
        \includegraphics[width=\textwidth]{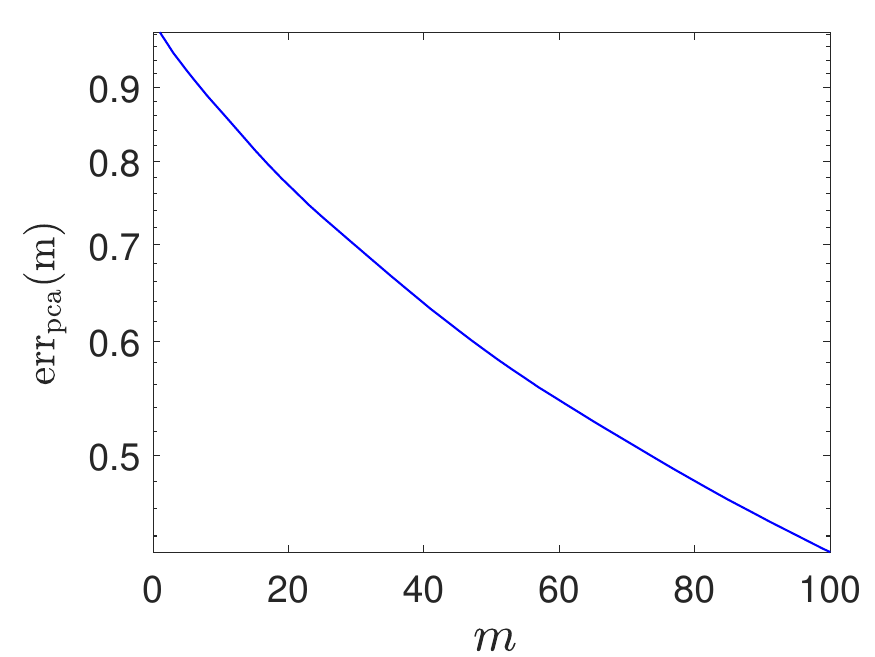}
        \caption{Graph $m\mapsto\perr_\pPCA(m)$ defined by Eq.~\eqref{eq2.17}.}
        \label{figure2b}
    \end{subfigure}
\caption{Covariance eigenvalues and truncated error of the reduced-order representation $U^m$ of random field $U$.}
    \label{figure2}
\end{figure}
\subsection{Parameter values for the generation of the random neural architecture and illustration of a realization}
\label{Section11.4}
The number $N$ of neurons stochastically  located on the manifold $\curS_h$ is set to $200$.
These random  locations result from sampling the inhomogeneous Poisson process introduced in Section~\ref{Section3}, where the number $M$ of candidate points is set to $M=2 n_\pelem = 7680$.
The number $n_\pin$ of input neurons and $n_\pout$ of output neurons, whose random locations are selected as explained in Section~\ref{Section4.1}, based on the extreme values of the latent gaussian random field, are set to $n_\pin= 20$ and $n_\pout = 100$. Thus, the total number $N=200$ neurons is decomposed in $20$ input neurons, $100$ output neurons, and $80$ hidden neurons.
The sparsity binary mask for the random neural graph topology is constructed as explained in Section~\ref{Section4.2}, where the percentile threshold  parameter $\tau_\pprc$ is set to $75\%$, which means that only $25\%$ of inter-neurons connections are retained. Since the total number of possible connections (excluding self-connections) is $39\, 800$, the number of retained unique connections is $4975$.
For each given value of the hyperparameter $\bftheta$, the random weights for the sparsified connectivity in the geometric-probabilistic neural graph are constructed as explained in Section~\ref{Section5}.

To graphically illustrate a given realization of the random neural architecture, we consider the following value of the hyperparameter  $\bftheta$, defined by Eq.~\eqref{eq5.7}, with $h_1= 0.087$, $h_2 = 0.058$, $\zeta_s = 0.0668$, and $\bfbeta^{(1)}_{n_h} = \bfbeta^{(2)}_{n_h} = \bfzero$ (which corresponds to $n_h=0$).
Fig.~\ref{figure3a} shows one realization of the anisotropic latent Gaussian random field $\{ U^m(\bfx),\bfx\in\curS_h\}$, while Fig.~\ref{figure3b} displays the corresponding realization of random neuron locations on $\curS_h$, superimposed on the field realization.
Figs.~\ref{figure4a} and \ref{figure4b} show the realization of the random sparse connectivity and its corresponding density map on the manifold $\curS_h$, associated with the realization of the latent random field $U^m$ and the corresponding realization of random neuron locations on  $\curS_h$, as shown in Fig.~\ref{figure3}. The sparse connectivity corresponds to the realization of the random edge-mask matrix defined by Eq.~\eqref{eq4.12}, but drawn along the geodesics defined by Eq.~\eqref{eq4.8}, which are used to construct the random edge-mask matrix (see Eqs.~\eqref{eq4.10} to \eqref{eq4.13}. It should be noted that the geodesic paths between neurons shown in Fig.~\ref{figure4a} are not perfectly represented due to the triangulation of the manifold and the fact that the neurons do not coincide with the mesh nodes of the surface $\curS_h$. This is simply an artifact, but the generation of the connectivities is fully correct.
\begin{figure}[H]
    \centering
    \begin{subfigure}[b]{0.48\textwidth}
        \centering
        \includegraphics[width=\textwidth]{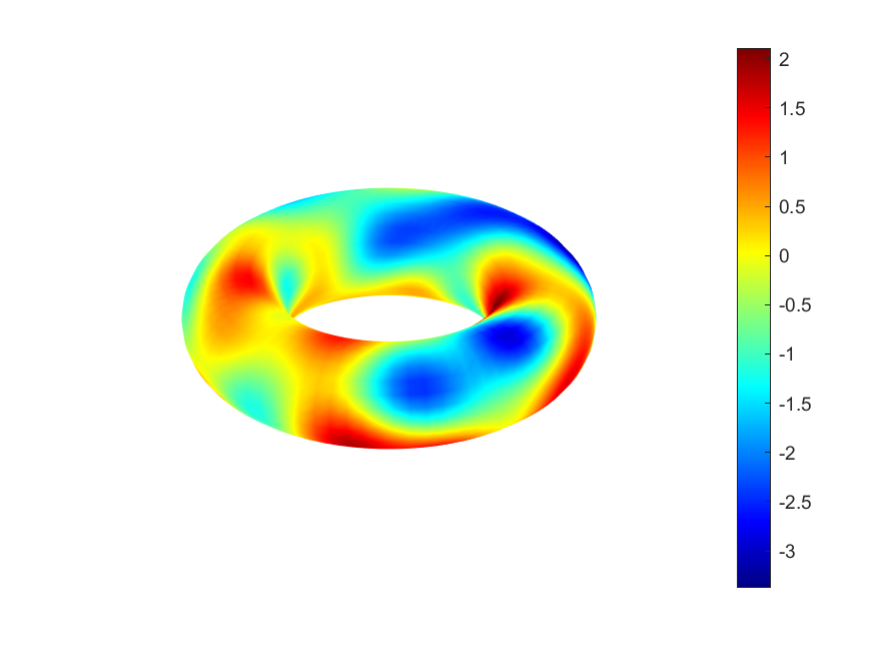}
        \caption{Realization of the anisotropic latent random field $U^m$.}
        \label{figure3a}
    \end{subfigure}
    \begin{subfigure}[b]{0.48\textwidth}
        \centering
        \includegraphics[width=\textwidth]{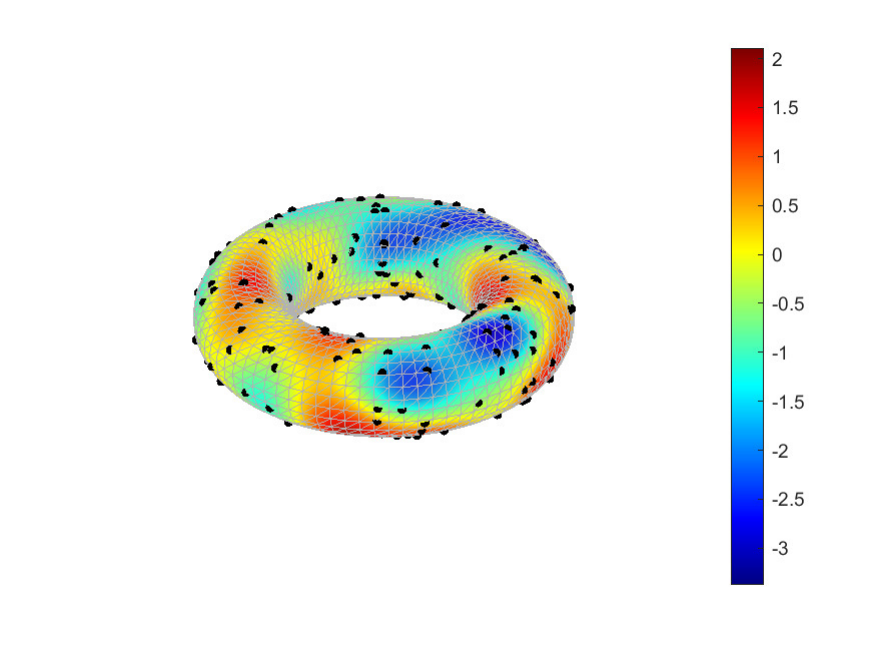}
        \caption{Realization of the random neuron locations on $\curS_h$.}
        \label{figure3b}
    \end{subfigure}
\caption{Realization of the latent random field $U^m$ and the corresponding realization of the random neuron locations on the manifold $\curS_h$, superimposed on the field.}
    \label{figure3}
\end{figure}
\begin{figure}[H]
    \centering
    \begin{subfigure}[b]{0.48\textwidth}
        \centering
        \includegraphics[width=\textwidth]{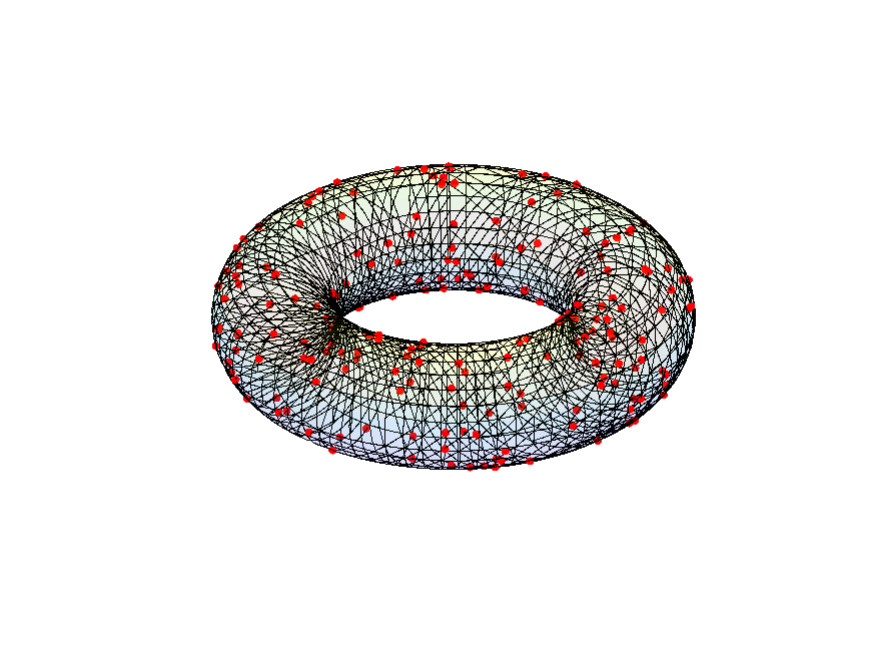}
        \caption{Realization of the random sparse connectivity.}
        \label{figure4a}
    \end{subfigure}
    \begin{subfigure}[b]{0.48\textwidth}
        \centering
        \includegraphics[width=\textwidth]{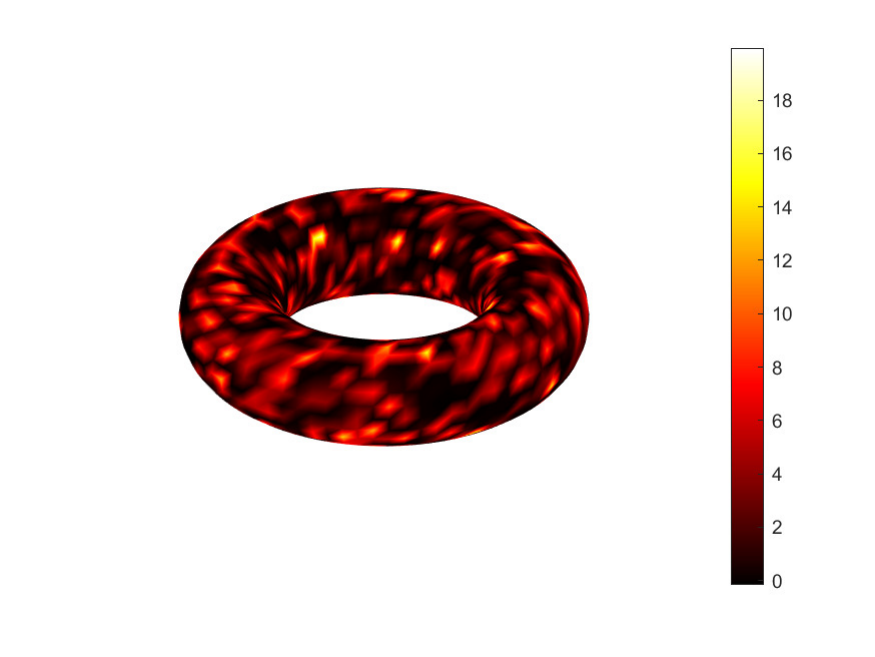}
        \caption{Realization of the sparse connectivity density map.}
        \label{figure4b}
    \end{subfigure}
\caption{Realization of the random sparse connectivity and its corresponding density map on the manifold $\curS_h$, associated with the realization of the latent random field $U^m$ and the corresponding realization of random neuron locations on  $\curS_h$, as shown in Fig.~\ref{figure3}.}
    \label{figure4}
\end{figure}
\subsection{Training the vector-valued hyperparameter from network supervision}
\label{Section11.5}
%
\noindent\textit{(i) Defining the training and test datasets}.
The supervised learning is conducted as described in Section~\ref{Section8}. The training dataset $\{ ( \xx^i , \yy^i) , i=1,\ldots,n_d\}$ consists of $n_d=3000$ samples in $\RR^{20}\times \RR^{100}$, while the test dataset
$\{ ( \xx^i_\testp , \yy^i_\testp) , i=1,\ldots,n_\testp\}$ consists of $n_\testp=600$ samples. For a given input $\xx$ in $\RR^{20}$, the probabilistic model defines the random output $\YY|\XX=\xx$ as an $\RR^{100}$-valued random variable that is highly non-Gaussian.
Both datasets are generated from the same complex probabilistic model using non linear mappings and polynomial chaos expansions. These datasets can be obtained upon request to the author. Note that,  as previously explained, for both datasets the only available information consists, for each $\xx^i$ (or $\xx^i_\testp)$, of a single realization $\yy^i$ (or $\yy^i_\testp$) of the random vector $\YY|\XX=\xx^i$ (or $\YY|\XX=\xx^i_\testp$). This constitutes an additional difficulty for training the neural network model.\\

%
\noindent\textit{(ii) Additional details on model and algorithm choices}.
{\color{black} The dimension of the vector $\bfbeta$ is chosen as} $n_b = \hat n = N - n_\pin = 200 - 20 = 180$.
The fixed-point iteration defined
{\color{black} in Eq.~\eqref{eq7.10} is used}
to solve the random nonlinear neural network equation, with a relaxation parameter $\alpha=0.5$ and a tolerance $\varepsilon = 0.01$. Convergence was consistently achieved within a few dozen iterations. We also testes smaller values of $\varepsilon$, which increased the number of iterations but did not significantly affect {\color{black} the final optimized}  value $\bftheta_t^\optp$ of the hyperparameter $\bftheta_t$. In all cases, the number of iterations remainded below $500$.
In cases $n_h \ge 1$, we  use the smooth differentiable model of the functions $\hh^{(k,j)}$ on $\curS$, as defined in Section~\ref{Section11.2}-(ii).\\

\noindent\textit{(iii) Solving the trial-based global optimization problem}.
The trial-based global optimization is carried out as described in Step~2 of Section~\ref{Section8}-(iii).
We recall that, for this step, the hyperparameter is $\bftheta_\ptrial = (h_1,h_2,\zeta_s)$, with $n_h=0$.
A computational pre-analysis was conducted on the grid $\curC_\ptrial = [0.02\, ,0.10]^3$, defined in Eq.~\eqref{eq8.5},
in order to identify a region where a finer grid is to be used. This refined domain is
$\curC_\ptrial = [0.065\, , 0.120]\times [0.034\, , 0.094]\times [ 0.040\, , 0.107]$.
The finer grid $\curC_\pgrid \subset \curC_\ptrial$ consists of $6^3= 216$ nodes, denoted $\bftheta_\ptrial^{\, j}$, with $n_\ppgrid = 216$.
Note that each grid node in $\curC_\pgrid$ satisfies the positivity constraints $h_1 > 0$, $h_2 > 0$, and $\zeta_s > 0$ (see Eq.~\eqref{eq5.7}), so $\bftheta_\ptrial$ automatically satisfies all required constraints. The optimal value $\bftheta_\ptrial^\optp$ obtained for $\bftheta_\ptrial$ is
\begin{equation}\label{eq11.2}
\bftheta_\ptrial^\optp = (h_{1,\ptrial}^\optp = 0.087 \,\, , \,\, h_{2,\ptrial}^\optp = 0.058 \,\, , \,\, \zeta_{s,\ptrial}^\optp = 0.0668 ) \, .
\end{equation}
Fig.~\ref{figure5a} shows the loss function $j\mapsto \curJ( \bftheta_\ptrial^j )$ for $j=1,\ldots , n_\ppgrid$, where
$\curJ(\bftheta_\ptrial^j)$ is computed using Eqs.~\eqref{eq8.2} with \eqref{eq8.3}.
For $\zeta_s = \zeta_{s,\ptrial}^\optp$, Fig.~\ref{figure5b} shows the graph of the function $(h_1,h_2)\mapsto \curJ( h_1,h_2,\zeta_{s,\ptrial}^\optp)$ over  the domain $[0.076\, , 0.098]\times [0.046\, ,0.07]$, which is centered at the optimal point $(h_{1,\ptrial}^\optp , h_{2,\ptrial}^\optp)$. It can be observed  that the function is smooth in $(h_1,h_2)$ in a neighborhood of the optimal solution.
Fig.~\ref{figure6} displays the graph of the function $k\mapsto \bfbeta_{\ptrial,k}^\optp$ for $k\in\{1,\ldots ,n_b\}$,
{\color{black} with $n_b = \hat n =N - n_\pin = 180$}.
\begin{figure}[H]
    \centering
    \begin{subfigure}[b]{0.40\textwidth}
        \centering
        \includegraphics[width=\textwidth]{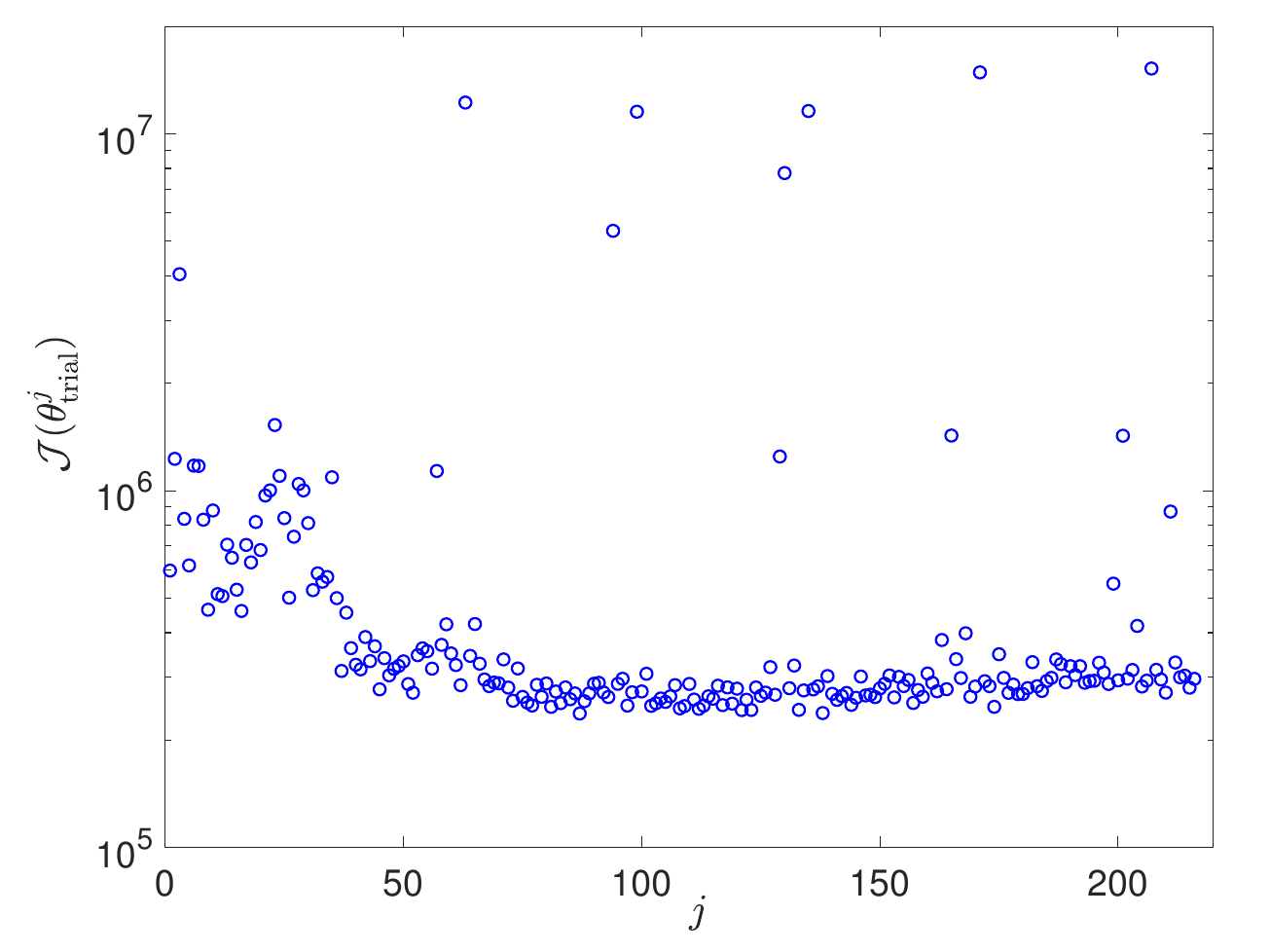}
        \caption{Loss function $j\mapsto \curJ( \bftheta_\ppptrial^j )$.}
        \label{figure5a}
    \end{subfigure}
    \begin{subfigure}[b]{0.40\textwidth}
        \centering
        \includegraphics[width=\textwidth]{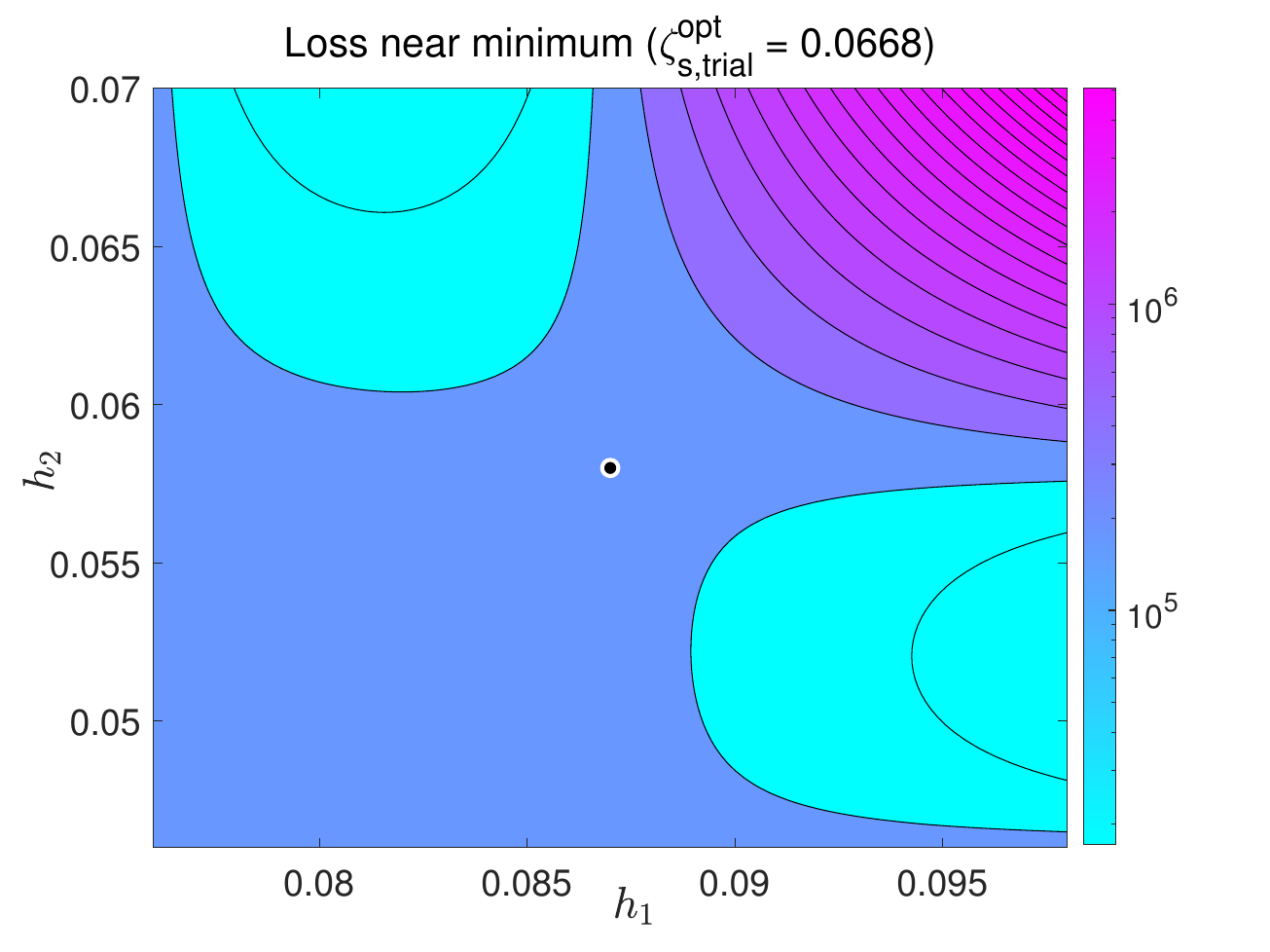}
        \caption{Restricted loss function $(h_1,h_2)\mapsto \curJ( h_1,h_2,\zeta_{s,\ppptrial}^\optpp)$.}
        \label{figure5b}
    \end{subfigure}
\caption{Trial-based global optimization: realization of the latent random field $U^m$ and the corresponding realization of the random neuron locations on the manifold $\curS_h$, superimposed on the field.}
    \label{figure5}
\end{figure}
\begin{figure}[H] 
\centering
\includegraphics[width=7.5cm]{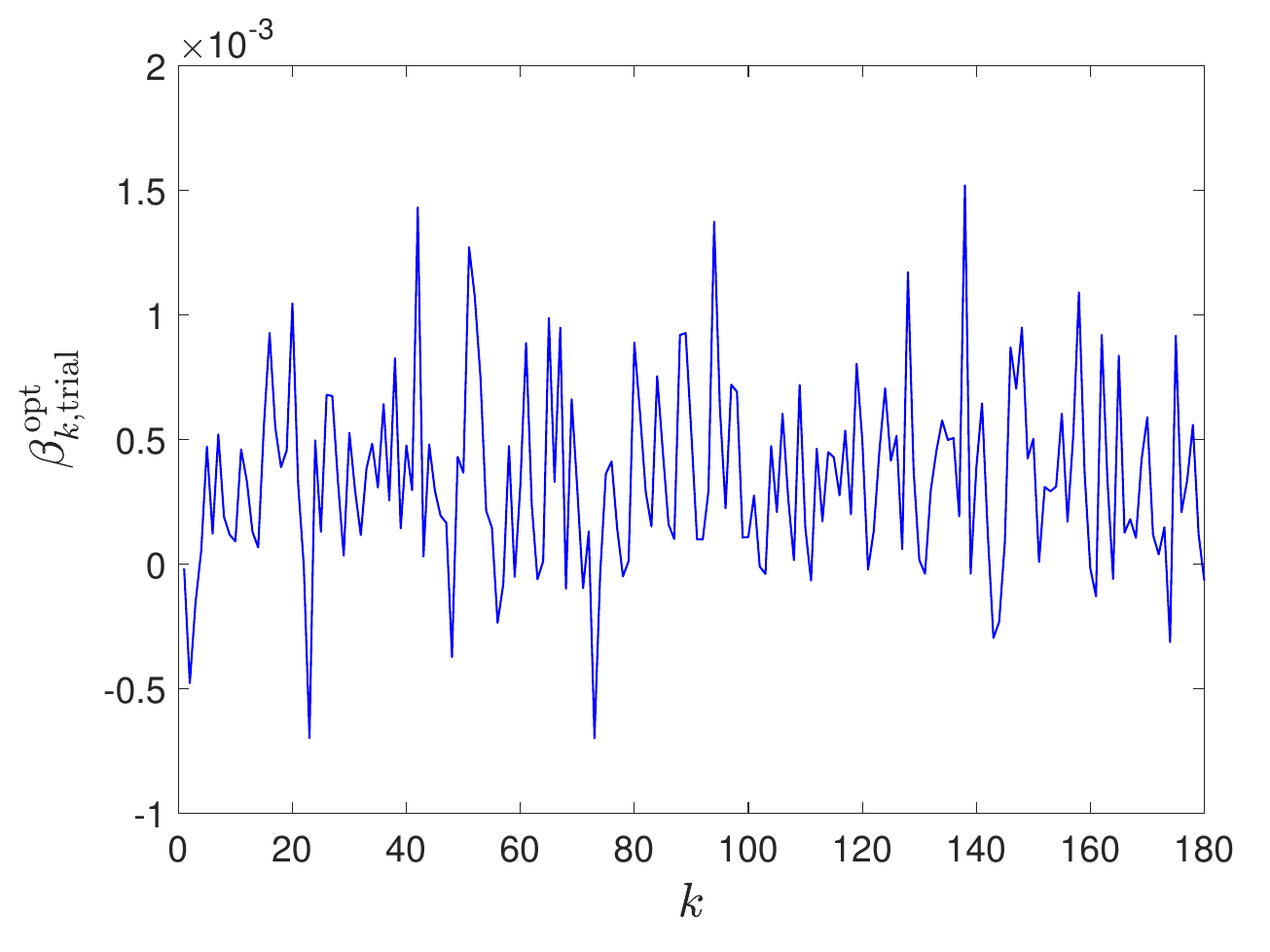}
\caption{Trial-based global optimization: graph of the function $k\mapsto \bfbeta_{\pptrial,k}^\optp$ for $k\in\{1,\ldots ,n_b\}$.}
\label{figure6}
\end{figure}
%

\noindent\textit{(iv) Solving the local optimization problem using the projected gradient descent with Adam updates and trial-based initialization}.
The local optimization is performed as described in Step~3 of Section~\ref{Section8}-(iii).
The hyperparameters $h_1$, $h_2$, and $\zeta_s$ are set to $h_{1,\ptrial}^\optp$, $h_{2,\ptrial}^\optp$, and $\zeta_{s,\ptrial}^\optp$, respectively. We recall that, for this step:

\noindent - when  $n_h = 0$, the updated hyperparameter is {\color{black} $\bfbeta\in\RR^{180}$};

\noindent - when $n_h \geq 1$, the updated hyperparameter vector is {\color{black} $( \bfbeta^{\,(1)}_{n_h}, \bfbeta^{\,(2)}_{n_h} , \bfbeta)\in\RR^{2n_h + 180}$}, subject to the constraints in Eqs.~\eqref{eq8.17}-\eqref{eq8.18},
{\color{black} with $c_\star^{(k)} = 10^{-12}$ and $c^{(k)\star} = +\infty$},
and handled as explained in Section~\ref{Section8}-(iv).

\noindent The optimal value $\bftheta_\ptrial^\optp$, defined in Eq.~\eqref{eq11.2}, is used to construct, the initial
{\color{black} value $\bftheta_0$, as defined in Eq.~\eqref{eq8.13}}, for the projected gradient descent algorithm with Adam updates.
The step size for the Adam algorithm is set to $0.001$ for each component of $\bfbeta^{\,(1)}_{n_h}$ and $\bfbeta^{\,(2)}_{n_h}$, and to $0.01$ for $\bfbeta$. {\color{black} As explained, the gradient} of the function $\bftheta_t\mapsto \curJ(\bftheta_t)$ is computed using a central difference approximation, with a relative directional step size of $10^{-6}$  for the numerical gradient.\\

\noindent - For the case $n_h=0$, Fig.~\ref{figure7a} shows the graph of the loss function $n\mapsto \NLL_\ANNp (n)$, defined in Eq.~\eqref{eq8.15}, and Fig.~\ref{figure7b} shows the graph of the normalized maximum change function $n \mapsto \Delta_\max (n)$, defined in Eq.~\eqref{eq8.16} with $n_\wind = 3$, both as  functions of the iteration number $n$ of the projected gradient descent with Adam updates. The final iteration is denoted by $n^\optp$, and its value can be identified in the figure.
Fig.~\ref{figure8a} shows the graphs  of the components $n\mapsto \beta_k(n)$ for $k=1,\ldots, n_b=180$, representing the evolution of the hyperparameter vector $\bfbeta\in\RR^{180}$ over the iterations. Fig.~\ref{figure8b} shows the graph of $k \mapsto \beta_k^\optp$, representing the components of the optimal value $\bfbeta^\optp \in\RR^{180}$ at iteration $n=25$.
\begin{figure}[H]
    \centering
    \begin{subfigure}[b]{0.40\textwidth}
        \centering
        \includegraphics[width=\textwidth]{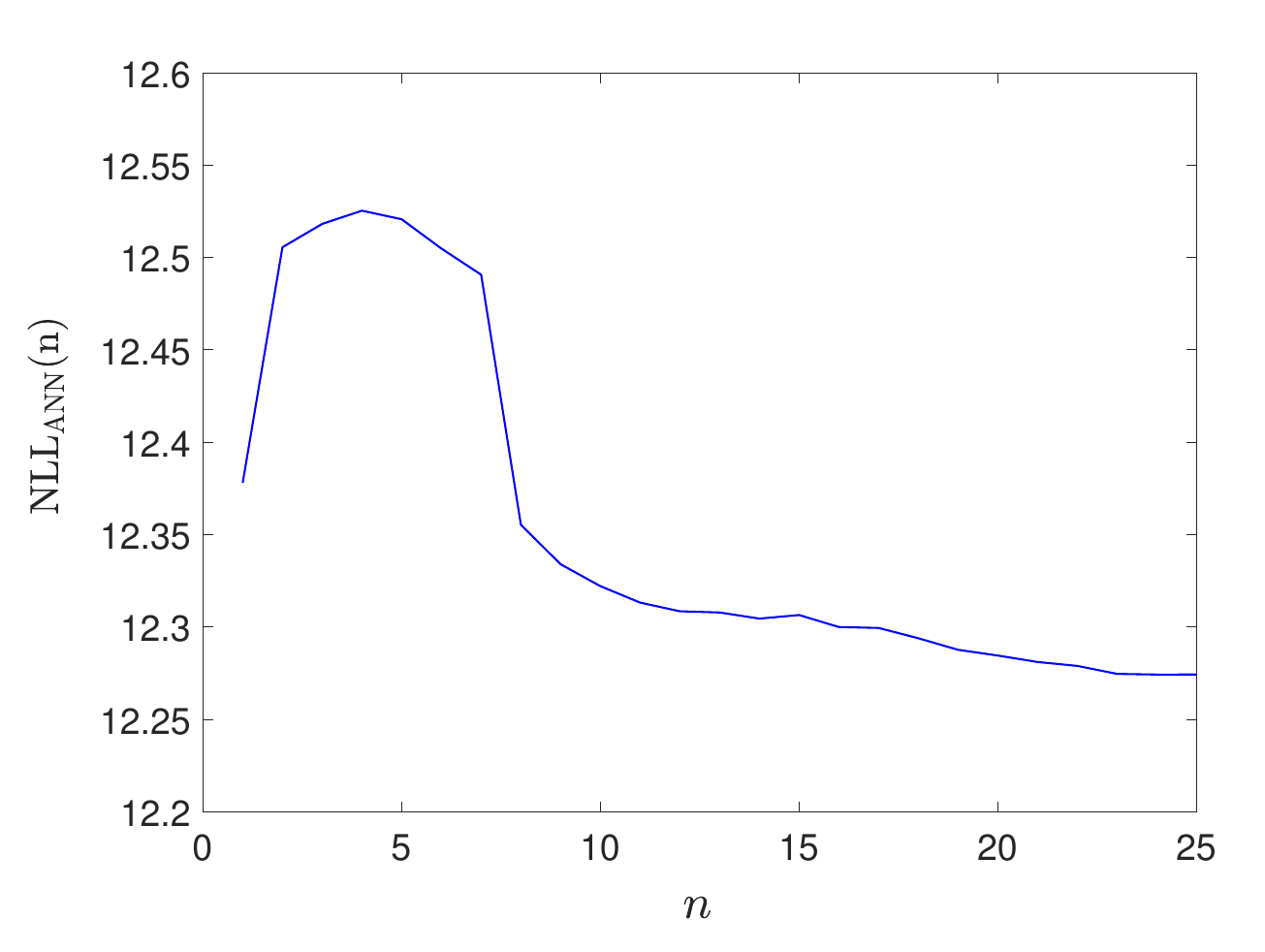}
        \caption{Graph of $n\mapsto \NLLg_\ANNpp (n)$ in $\log_{10} scale$.}
        \label{figure7a}
    \end{subfigure}
    \begin{subfigure}[b]{0.40\textwidth}
        \centering
        \includegraphics[width=\textwidth]{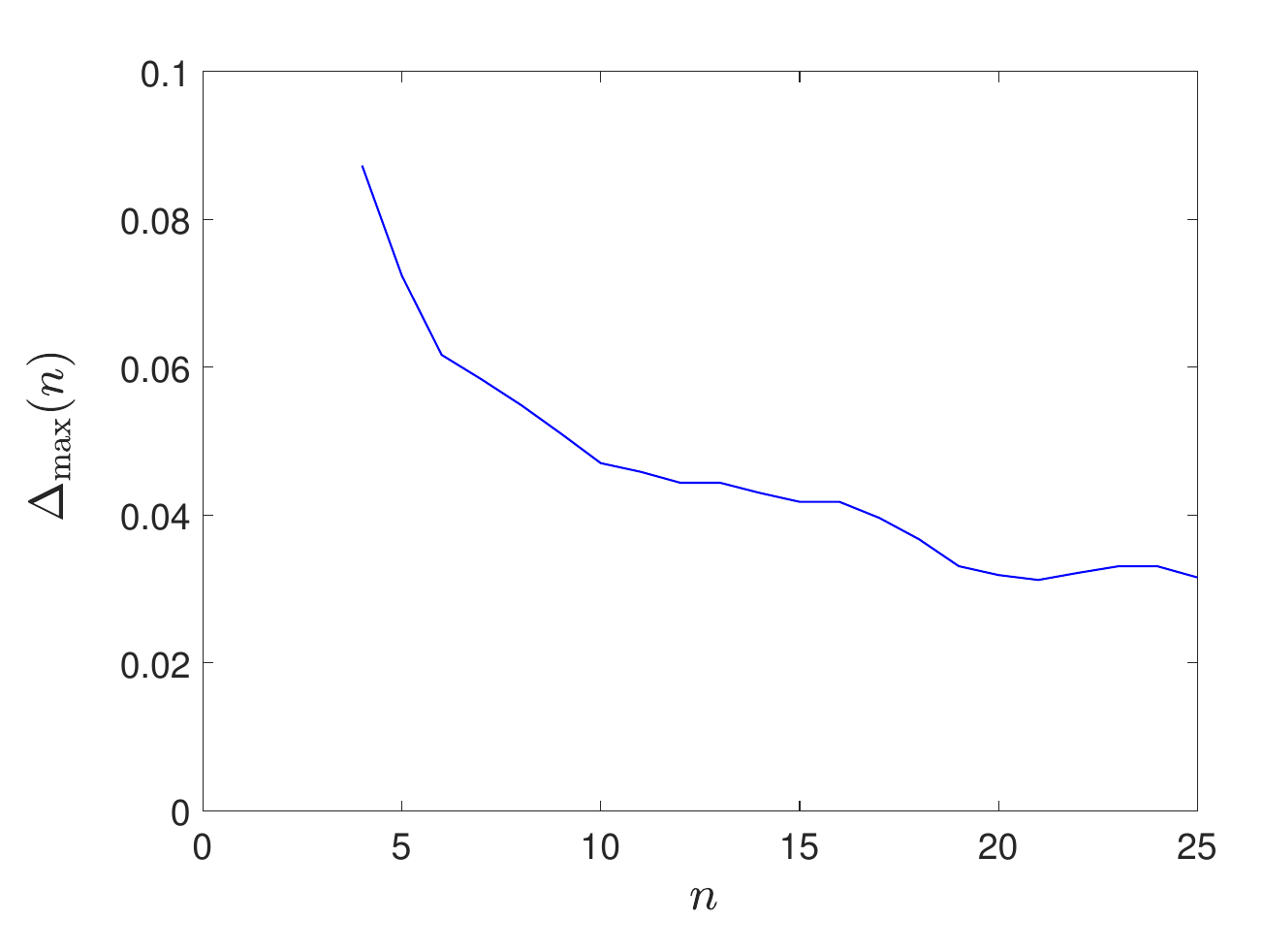}
        \caption{Graph of $n \mapsto \Delta_\max (n)$.}
        \label{figure7b}
    \end{subfigure}
\caption{Local optimization for the case $n_h=0$: (a) loss function $n\mapsto \NLLg_\ANNpp (n)$ defined in Eq.~\eqref{eq8.15}; (b) normalized maximum change function $n \mapsto \Delta_\max (n)$ defined in Eq.~\eqref{eq8.16} with $n_\wind = 3$, as  functions of the iteration number $n$ of the projected gradient descent with Adam updates.}
    \label{figure7}
\end{figure}
\begin{figure}[H]
    \centering
    \begin{subfigure}[b]{0.40\textwidth}
        \centering
        \includegraphics[width=\textwidth]{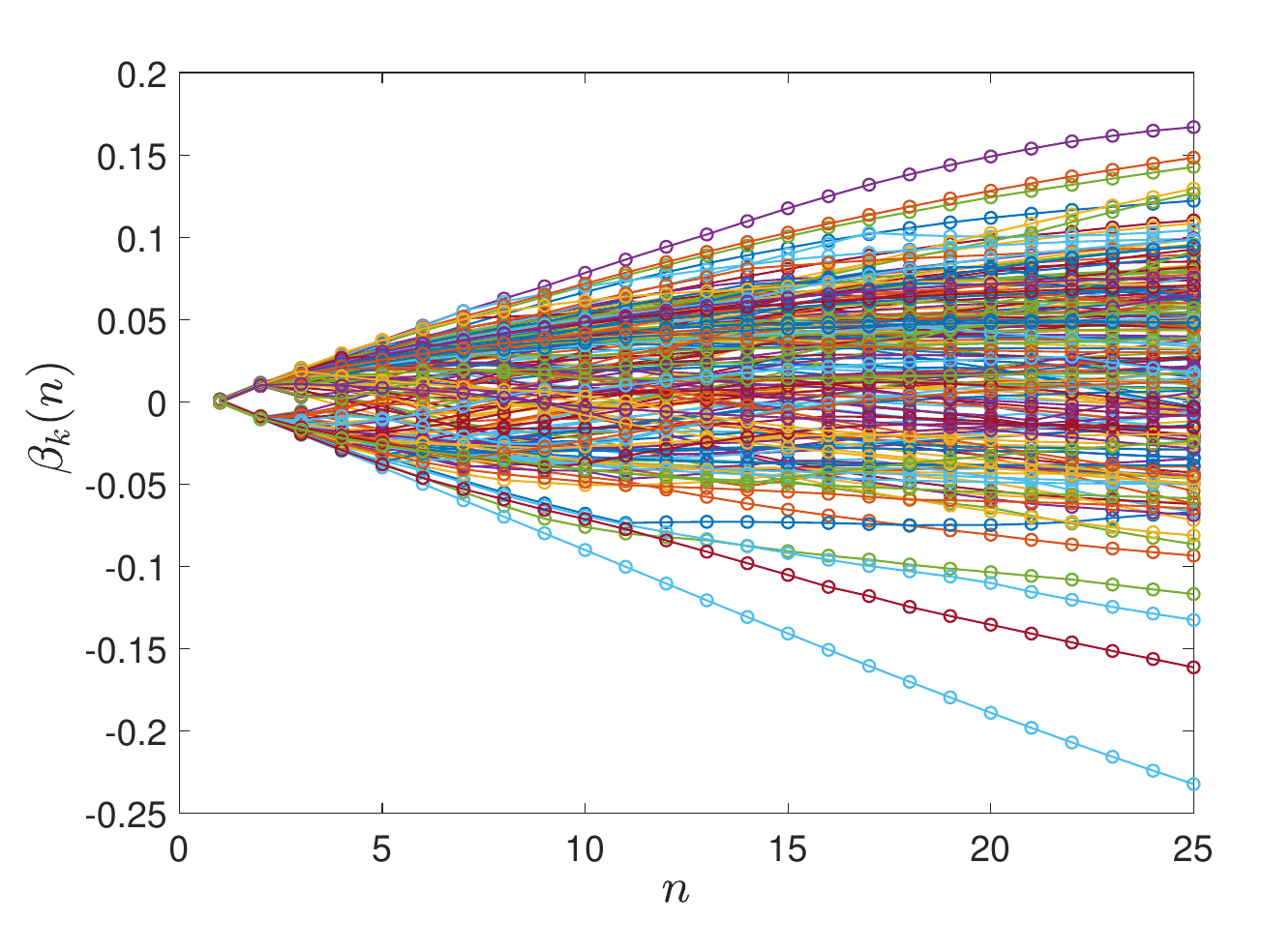}
        \caption{Graphs  of $n\mapsto \beta_k(n)$ for $k=1,\ldots, n_b=180$.}
        \label{figure8a}
    \end{subfigure}
    \begin{subfigure}[b]{0.40\textwidth}
        \centering
        \includegraphics[width=\textwidth]{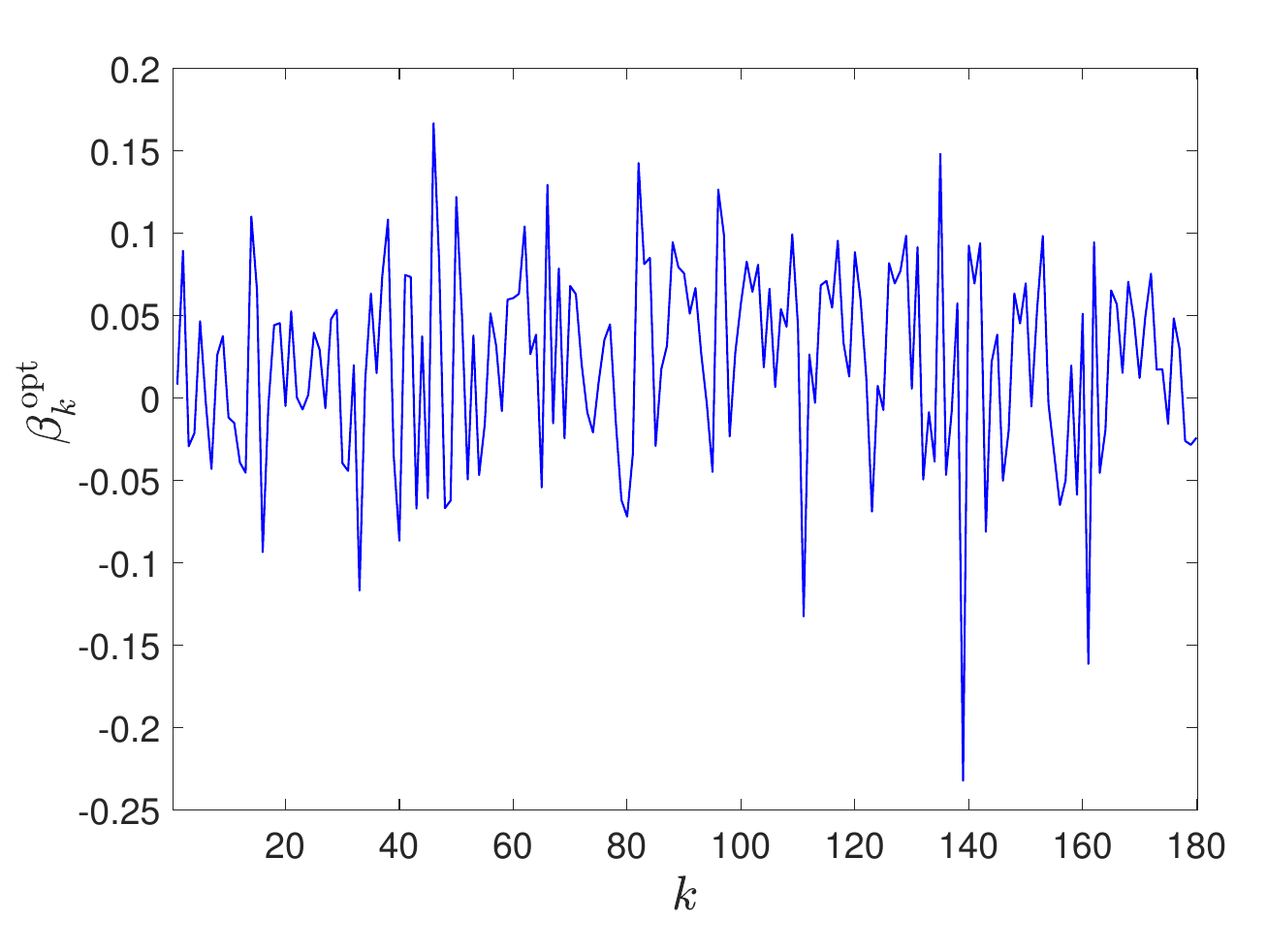}
        \caption{Graph of $k \mapsto \beta_k^\optp$ at iteration $n=25$.}
        \label{figure8b}
    \end{subfigure}
\caption{Local optimization for the case $n_h=0$:  (a) graphs  of the components $n\mapsto \beta_k(n)$ for $k=1,\ldots, n_b=180$ of the hyperparameter $\bfbeta\in\RR^{180}$, as functions of the iteration number $n$ of the projected gradient descent with Adam updates; (b) graph of $k \mapsto \beta_k^\optp$, the components of the optimal value $\bfbeta^\optp \in\RR^{180}$ at iteration $n=25$.}
    \label{figure8}
\end{figure}

\noindent - For the case $n_h=10$, Fig.~\ref{figure9a} shows the graph of the loss function $n\mapsto \NLL_\ANNp (n)$, defined in Eq.~\eqref{eq8.15}, and Fig.~\ref{figure9b} shows the graph of the normalized maximum change function $n \mapsto \Delta_\max (n)$, defined in Eq.~\eqref{eq8.16} with $n_\wind = 3$. Both  are shown as functions of the iteration number $n$ of the projected gradient descent with Adam updates.
Fig.~\ref{figure10a} shows the graphs  of the components $n\mapsto \beta_k(n)$ for $k=1,\ldots, n_b=180$, representing the evolution of the hyperparameter vector $\bfbeta\in\RR^{180}$ over the iterations. Fig.~\ref{figure10b} shows the graph of $k \mapsto \beta_k^\optp$, representing the components of the optimal value $\bfbeta^\optp \in\RR^{180}$ at iteration $n=100$.
Figs.~\ref{figure11a} ($k=1$) and \ref{figure11b} ($k=2$) show the graphs  of the $n_h = 10$ components $n\mapsto \beta_j^{(k)}(n)$ for $j=1,\ldots, n_h=10$ of the
{\color{black} hyperparameter $\bfbeta^{(k)}_{n_h}\in\RR^{10}$  as}
 functions of the iteration number $n$ of the projected gradient descent with Adam updates.

\begin{figure}[H]
    \centering
    \begin{subfigure}[b]{0.40\textwidth}
        \centering
        \includegraphics[width=\textwidth]{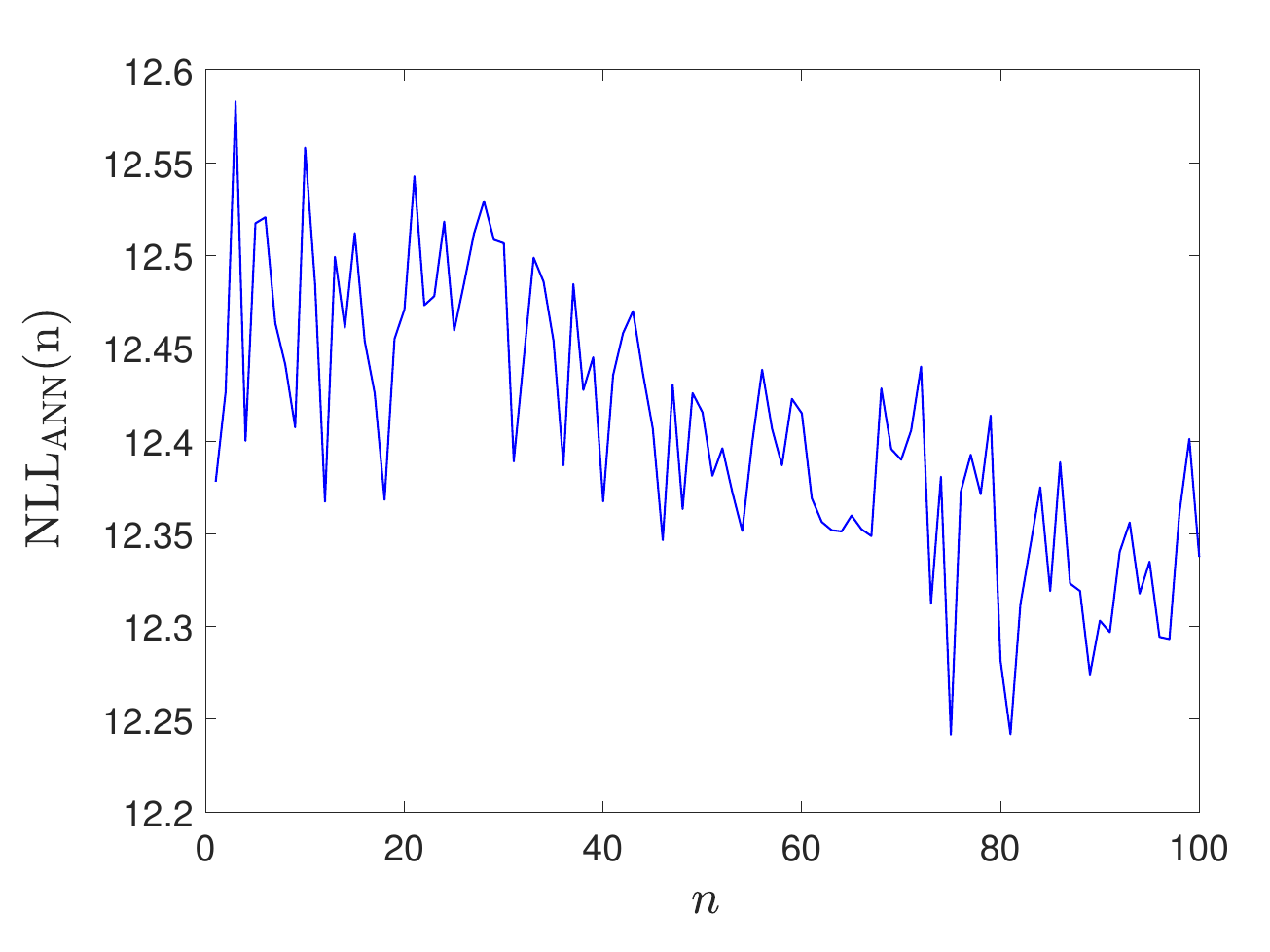}
        \caption{Graph of $n\mapsto \NLLg_\ANNpp (n)$ in $\log_{10} scale$.}
        \label{figure9a}
    \end{subfigure}
    \begin{subfigure}[b]{0.40\textwidth}
        \centering
        \includegraphics[width=\textwidth]{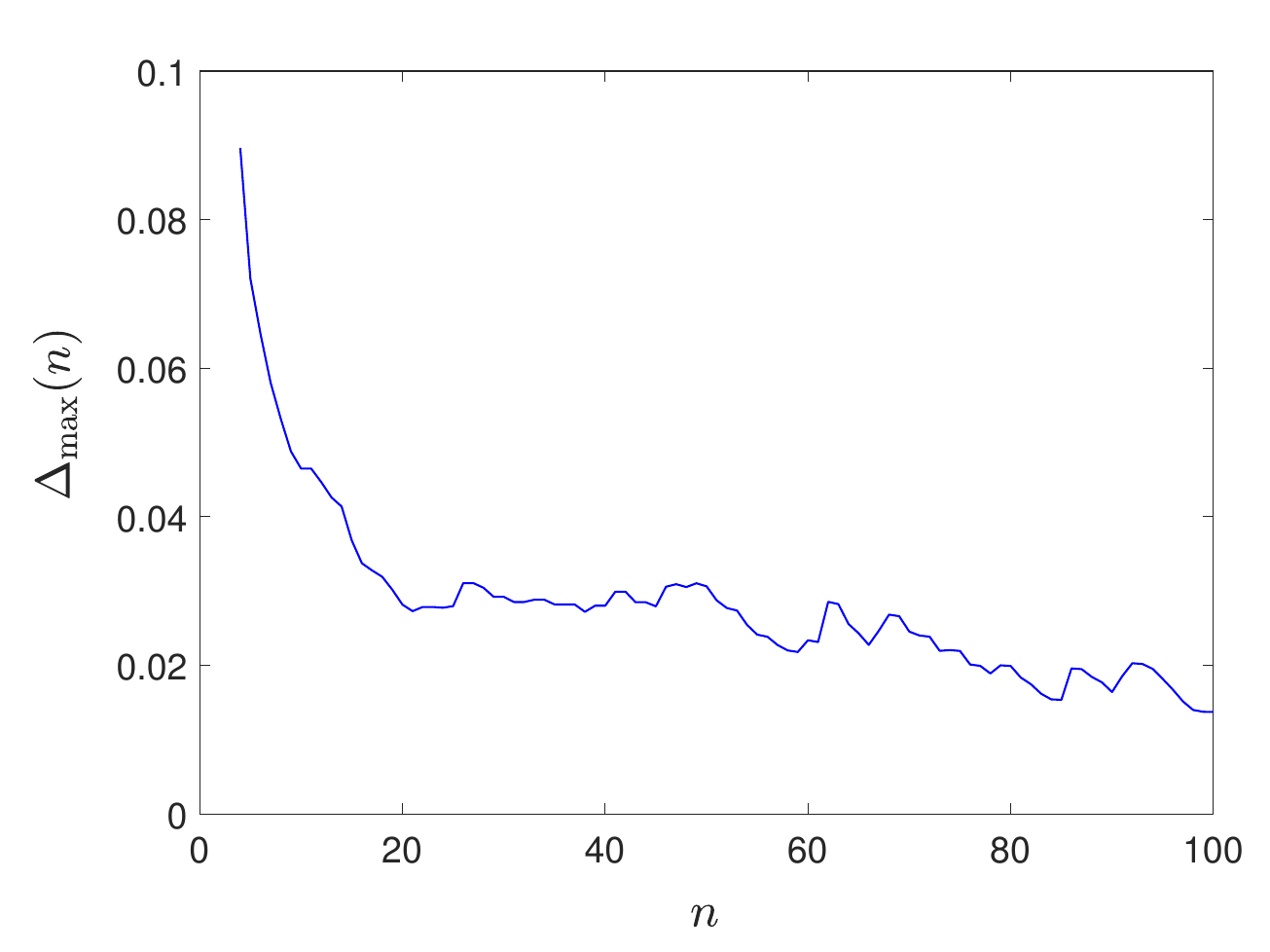}
        \caption{Graph of $n \mapsto \Delta_\max (n)$.}
        \label{figure9b}
    \end{subfigure}
\caption{Local optimization for the case $n_h=10$: (a) loss function $n\mapsto \NLLg_\ANNpp (n)$ defined in Eq.~\eqref{eq8.15}; (b) normalized maximum change function $n \mapsto \Delta_\max (n)$ defined in Eq.~\eqref{eq8.16} with $n_\wind = 3$, as  functions of the iteration number $n$ of the projected gradient descent with Adam updates.}
    \label{figure9}
\end{figure}
\begin{figure}[H]
    \centering
    \begin{subfigure}[b]{0.40\textwidth}
        \centering
        \includegraphics[width=\textwidth]{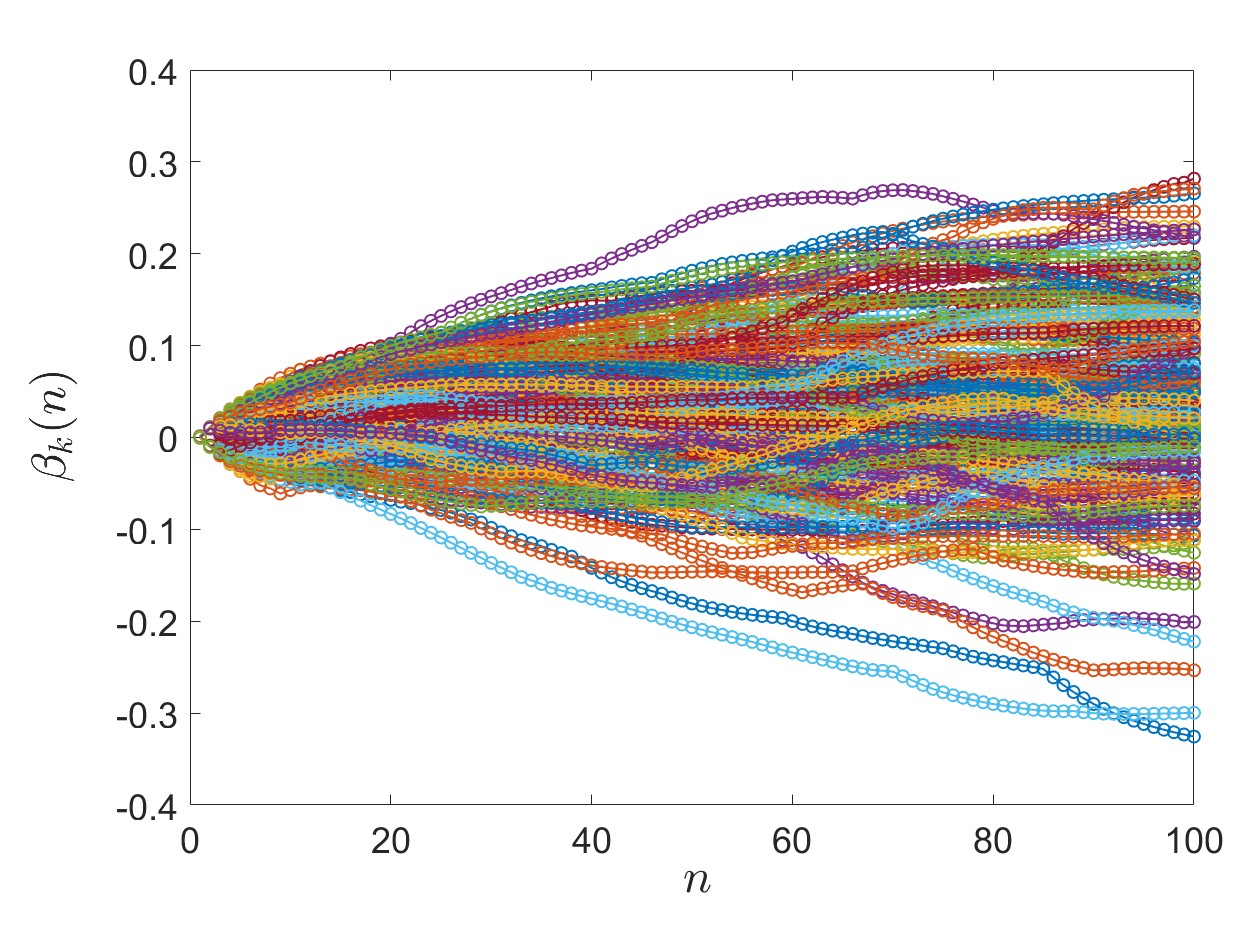}
        \caption{Graphs  of $n\mapsto \beta_k(n)$ for $k=1,\ldots, n_b=180$.}
        \label{figure10a}
    \end{subfigure}
    \begin{subfigure}[b]{0.40\textwidth}
        \centering
        \includegraphics[width=\textwidth]{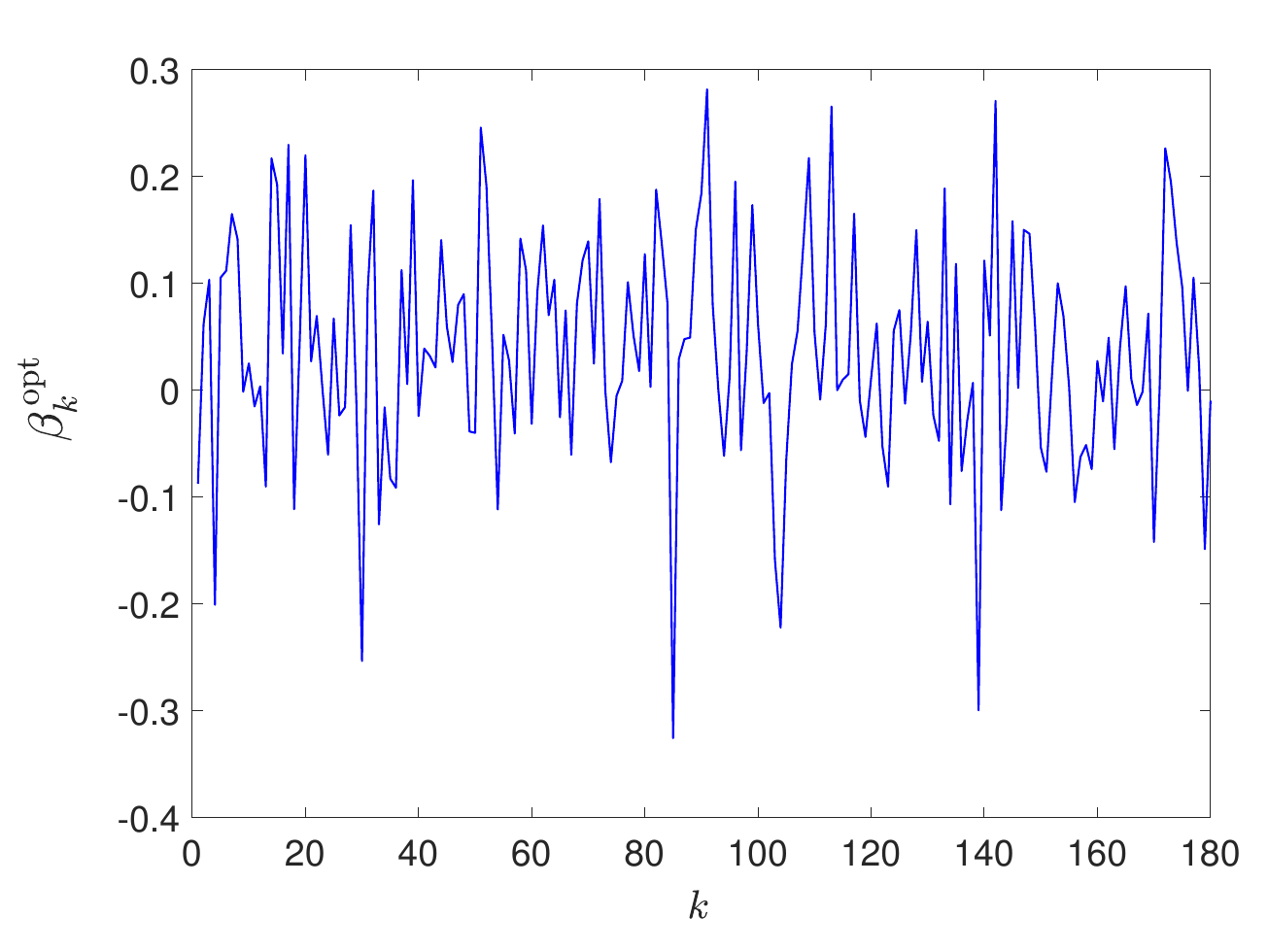}
        \caption{Graph of $k \mapsto \beta_k^\optp$ at iteration $n=25$.}
        \label{figure10b}
    \end{subfigure}
\caption{Local optimization for the case $n_h=10$:  (a) graphs  of the components $n\mapsto \beta_k(n)$ for $k=1,\ldots, n_b=180$ of the hyperparameter $\bfbeta\in\RR^{180}$, as functions of the iteration number $n$ of the projected gradient descent with Adam updates; (b) graph of $k \mapsto \beta_k^\optp$, the components of the optimal value $\bfbeta^\optp \in\RR^{180}$ at iteration $n=25$.}
    \label{figure10}
\end{figure}
\begin{figure}[H]
    \centering
    \begin{subfigure}[b]{0.40\textwidth}
        \centering
        \includegraphics[width=\textwidth]{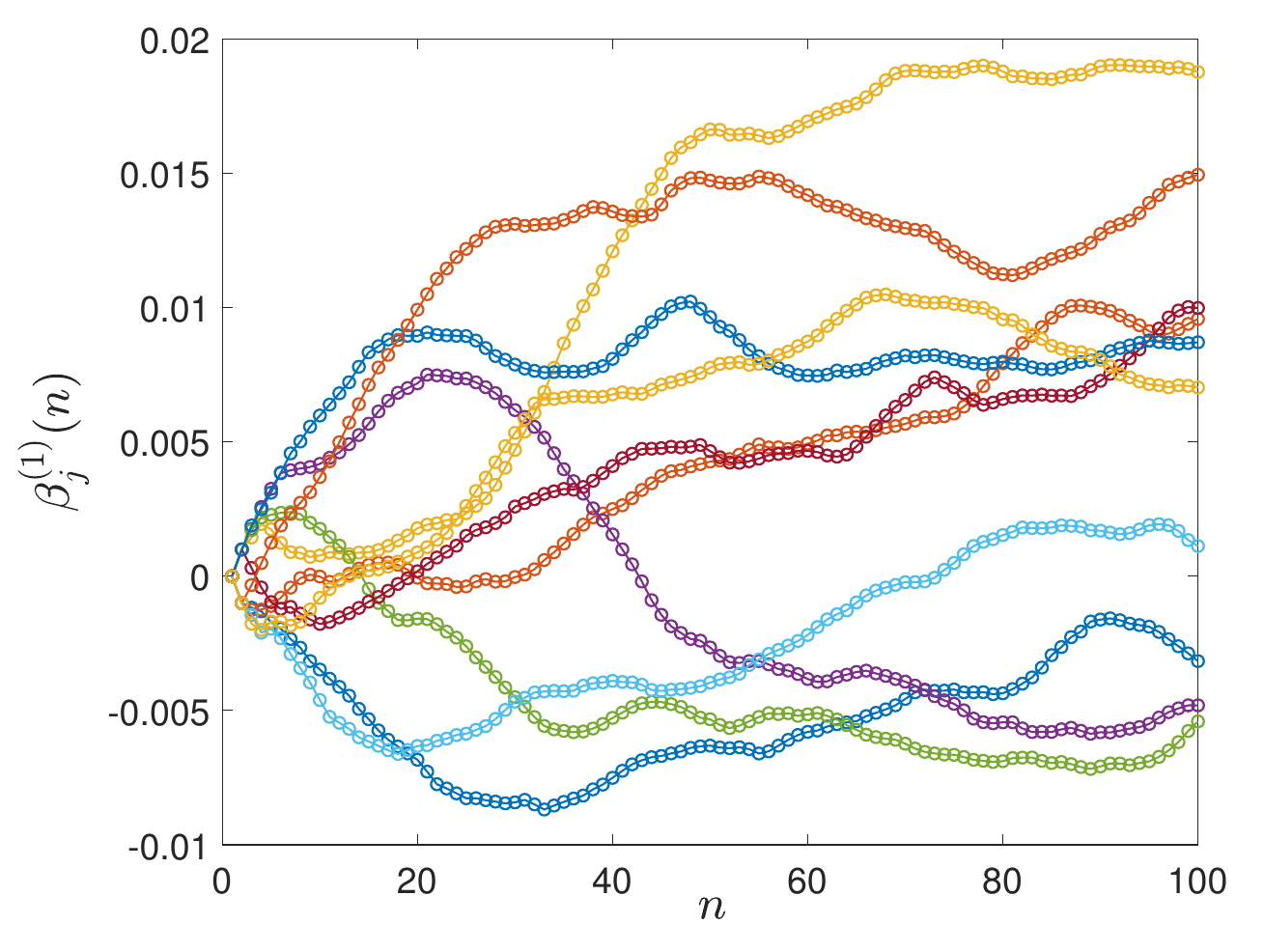}
        \caption{Graphs  of $n\mapsto \beta_j^{(1)}(n)$ for $j=1,\ldots, 10$.}
        \label{figure11a}
    \end{subfigure}
    \begin{subfigure}[b]{0.40\textwidth}
        \centering
        \includegraphics[width=\textwidth]{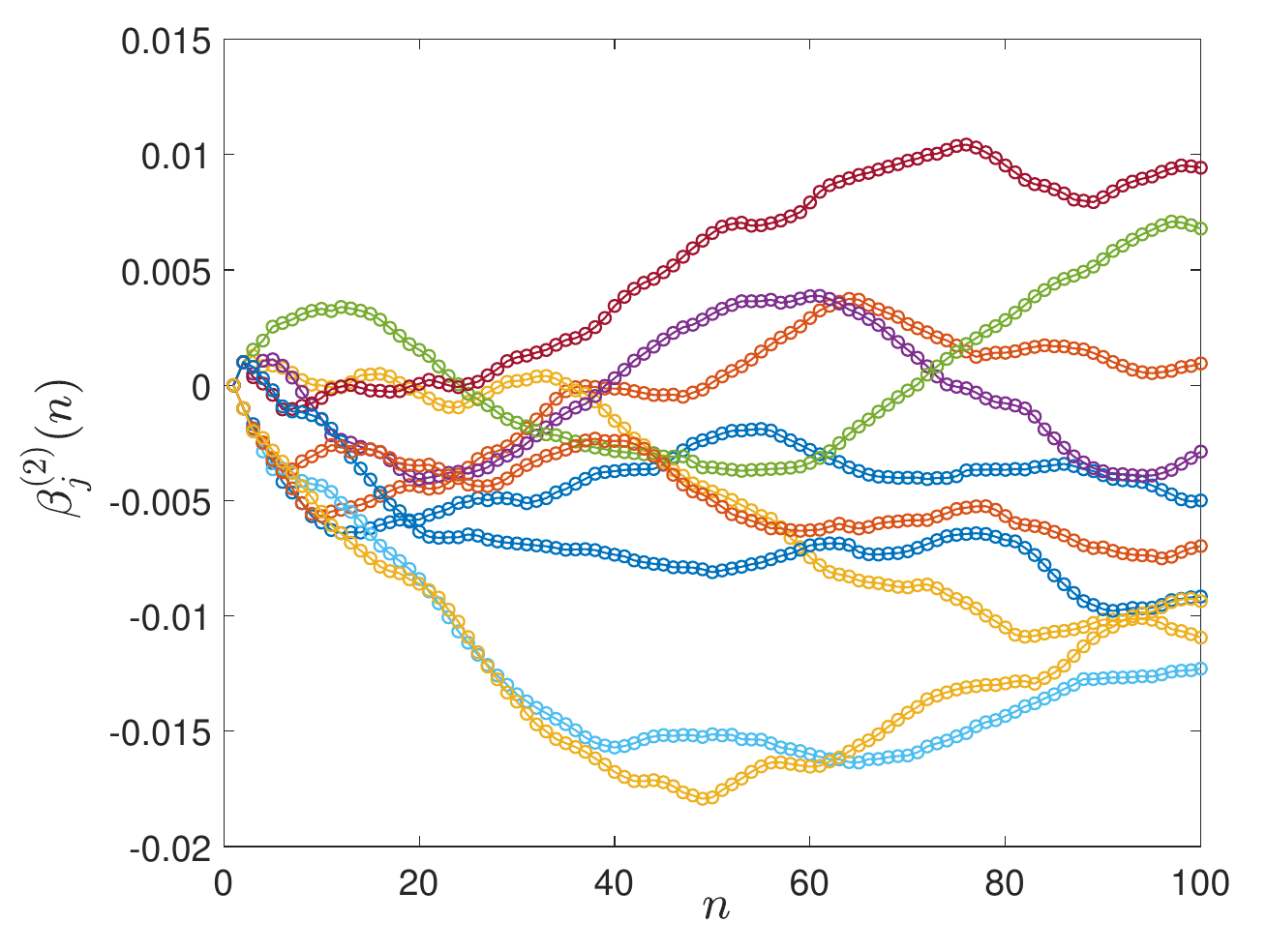}
        \caption{Graphs  of $n\mapsto \beta_j^{(1)}(n)$ for $j=1,\ldots, 10$.}
        \label{figure11b}
    \end{subfigure}
\caption{Local optimization for the case $n_h=10$:  For $k\in\{1,2\}$, graphs  of the components $n\mapsto \beta_j^{(k)}(n)$ for $j=1,\ldots, n_h$ of the hyperparameter $\bfbeta^{\,(k)}_{n_h}\in\RR^{n_h}$, as functions of the iteration number $n$ of the projected gradient descent with Adam updates.}
    \label{figure10}
\end{figure}

For the two cases, $n_h=0$ and $n_h=10$, the values of all the criteria defined in Section~\ref{Section9} are given in Table~\ref{table1}, in the rows corresponding to $N=200$, $n_d=3000$, and $n_\testp=600$.
\subsection{Training results, quantitative evaluation of the trained random network, and sensitivity analysis}
\label{Section11.6}
The quantitative evaluation of the trained random network is based on the five criteria detailed in Section~\ref{Section9}.\\

\noindent\textit{(i) Quantitative evaluation of the baseline configuration  and sensitivity analysis with respect to the training dataset and the number of hyperparameters}.
In this paragraph, we consider the case detailed in Sections~\ref{Section11.1} {\color{black}  to \ref{Section11.5}}, which corresponds to $N=200$, $n_d=3000$, $n_\testp=600$, and, $n_h=0$ and $10$. We present the training results and the corresponding quantitative evaluation for this configuration.
Additionally, we analyze how these results vary as functions of the number of training samples $n_d$ and the
hyperparameter dimension $n_h$, which controls the latent Gaussian random field.
For the three configurations presented below,  the optimal value
$\bftheta_\ptrial^\optp = (h_{1,\pptrial}^\optp,  h_{2,\pptrial}^\optp , \zeta_{s,\pptrial}^\optp)$, obtained  by solving the trial-based global optimization problem defined by Eq.~\eqref{eq8.6}, is provided in  Table~\ref{table1} for each configuration.
\begin{table}[h]
  \caption{Optimal value $\bftheta_\ppptrial^\optpp$ obtained  by solving the trial-based global optimization problem  for $N=200$}
  \label{table1}
  \begin{center}
  \begin{tabular}{|c|c|c|c|c|}
    \hline
        $n_d$ & $n_\testp$ & $h_{1,\ptrial}^\optp$  & $h_{2,\ptrial}^\optp$ & $\zeta_{s,\ptrial}^\optp$  \\
    \hline
     1000 & 200  & 0.0957   & 0.0642 & 0.0726 \\
     3000 & 600  & 0.0870   & 0.0580 & 0.0668 \\
    10000 & 2000 & 0.0870   & 0.0580 & 0.0668 \\
    \hline
  \end{tabular}
\end{center}
\end{table}

\noindent All values of the evaluation criteria are summarized in  Table~\ref{table2}, in which the values of
$\curJ(\bftheta^\optp_\ptrial)$ {\color{black} (defined by Eq.~\eqref{eq8.6})} and
$\NLL_\ANNp^\opt$ {\color{black} (defined by \eqref{eq8.14})} are divided by $n_d$, and
$\NLL_\ANNp^\testp$  {\color{black} (defined by \eqref{eq9.1})} is divided by $n_\testp$.
A detailed discussion of these results is provided in Section~\ref{Section12}.
\begin{table}[h]
  \caption{Evaluation criteria of the baseline configuration  and sensitivity analysis for $N=200$}
  \label{table2}
  \begin{center}
  \resizebox{\textwidth}{!}{%
  \begin{tabular}{|c|c|c|c|c|c|c|c|c|c|c|c|c|c|c|}
    \hline
          &  &  & Eqs. & \eqref{eq8.6}&\eqref{eq8.14}&\eqref{eq8.16}&\eqref{eq9.1}&\eqref{eq9.2}&\eqref{eq9.5}&\eqref{eq9.6}&\eqref{eq9.9}&\eqref{eq9.9}& \eqref{eq9.10} \\
    \hline
        $n_d$ & $n_\testp$ &   & $n_h$ & $\scriptstyle\curJ(\bftheta^\optp_\pptrial)$ & $\scriptstyle\NLLg_\ANNpp^\optp$  &
        $\scriptstyle\Delta_\max (n^\optp)$    & $\scriptstyle\NLLg_\ANNpp^\testp $ & $\scriptstyle\criter_o$ & $\scriptstyle\criter_{\bfalpha}$ &  $\scriptstyle\criter_{\bfbeta}$ & $\scriptstyle\CRPSp^\trainingpp_\testp$ & $\scriptstyle\CRPSp^\ANNpp_\testp$ & $\scriptstyle\criter_\CRPSpp$ \\
    \hline
     1000 & 200  & $\TRIAL$ & -  & 77.0 &  -   &   -   & 84.2 & 0.085 & 0.203 & 0.117 & 0.504 & 0.504 & 0.223 \\
          &      & $\PGD$   & 0  &  -   & 71.9 & 0.039 & 91.6 & 0.214 & 0.205 & 0.112 & 0.504 & 0.499 & 0.259 \\
          &      & $\PGD$   & 10 &  -   & 81.8 & 0.015 & 75.2 & 0.088 & 0.192 & 0.108 & 0.504 & 0.498 & 0.223 \\
          &      & $\PGD$   & 30 &  -   & 79.5 & 0.021 & 85.5 & 0.069 & 0.192 & 0.112 & 0.504 & 0.499 & 0.214 \\
     \hline
     3000 & 600  & $\TRIAL$ & -  & 79.2 &  -   &   -   & 80.3 & 0.014 & 0.205 & 0.128 & 0.489 & 0.495 & 0.243 \\
          &      & $\PGD$   & 0  &  -   & 71.4 & 0.031 & 79.5 & 0.102 & 0.198 & 0.127 & 0.489 & 0.497 & 0.239 \\
          &      & $\PGD$   & 10 &  -   & 76.0 & 0.014 & 78.2 & 0.028 & 0.199 & 0.120 & 0.488 & 0.495 & 0.194 \\
          &      & $\PGD$   & 30 &  -   & 87.5 & 0.013 & 79.8 & 0.096 & 0.211 & 0.120 & 0.488 & 0.482 & 0.220 \\
     \hline
    10000 & 2000 & $\TRIAL$ & -  & 76.0 &  -   &   -   & 78.7 & 0.033 & 0.205 & 0.228 & 0.136 & 0.485 & 0.245 \\
          &      & $\PGD$   & 0  &  -   & 60.3 & 0.019 & 63.4 & 0.050 & 0.227 & 0.131 & 0.485 & 0.480 & 0.207 \\
          &      & $\PGD$   & 10 &  -   & 69.4 & 0.025 & 72.9 & 0.040 & 0.228 & 0.128 & 0.485 & 0.485 & 0.217 \\
          &      & $\PGD$   & 30 &  -   & 71.9 & 0.017 & 67.8 & 0.060 & 0.213 & 0.128 & 0.485 & 0.489 & 0.197 \\
    \hline
  \end{tabular}
  } 
\end{center}
\end{table}
The last evaluation criterion  concerns, for a given  sample from the test dataset, the comparison of conditional probability density functions and the associated conditional confidence intervals. For consistency, we consider the configuration detailed in Sections~\ref{Section11.1} {\color{black} to \ref{Section11.5}}, corresponding to $N=200$, $n_d=3000$, $n_\testp=600$, and $n_h =10$.
For a selected  sample input $\xx^6_\testp \in\RR^{20}$ from the test dataset, and for the $k$-th component of the real-valued random output $\YY_k$,  selected  for $k = 5$, $26$, and $90$, Figs.~\ref{figure12a} to \ref{figure12c} show the graphs of the conditional probability density functions
$y\mapsto  p^\ANNpp_k(y \, \vert \, \xx^6_\testp \, ; \bftheta_t^\optp)$ and $y\mapsto p^\trainingpp_k(y \, \vert \, \xx^6_\testp)$, which are estimated as explained in  Section~\ref{Section9.5}.  Table~\ref{table3} presents, for $k = 5$, $26$, and $90$, a comparison of the conditional confidence intervals $\CI^\ANNpp_k(\xx^6_\testp) = [\alpha^\ANNpp_k(\xx^6_\testp) \, , \beta^{\,\ANNpp}_k(\xx^6_\testp)]$ and $\CI^\trainingpp_k(\xx^6_\testp)=  [\alpha^\trainingpp_k(\xx^6_\testp) \, , \beta^{\,\trainingpp}_k(\xx^6_\testp)]$, computed using Eqs.~\eqref{eq9.3} and \eqref{eq9.4}.
\begin{figure}[H]
    \centering
    \begin{subfigure}[b]{0.30\textwidth}
        \centering
        \includegraphics[width=\textwidth]{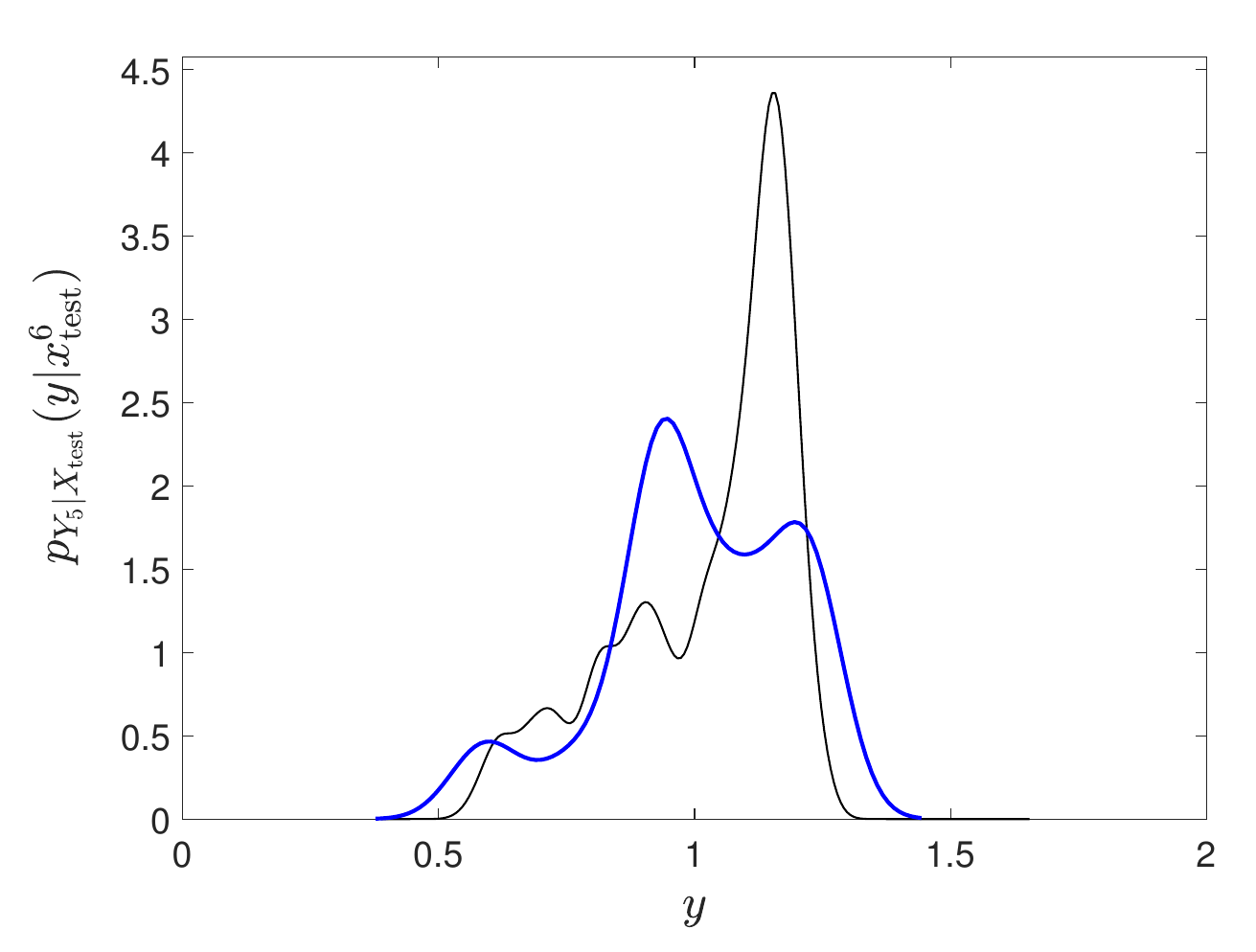}
        \caption{Conditional pdf for $k=5$.}
        \label{figure12a}
    \end{subfigure}
    \begin{subfigure}[b]{0.30\textwidth}
        \centering
        \includegraphics[width=\textwidth]{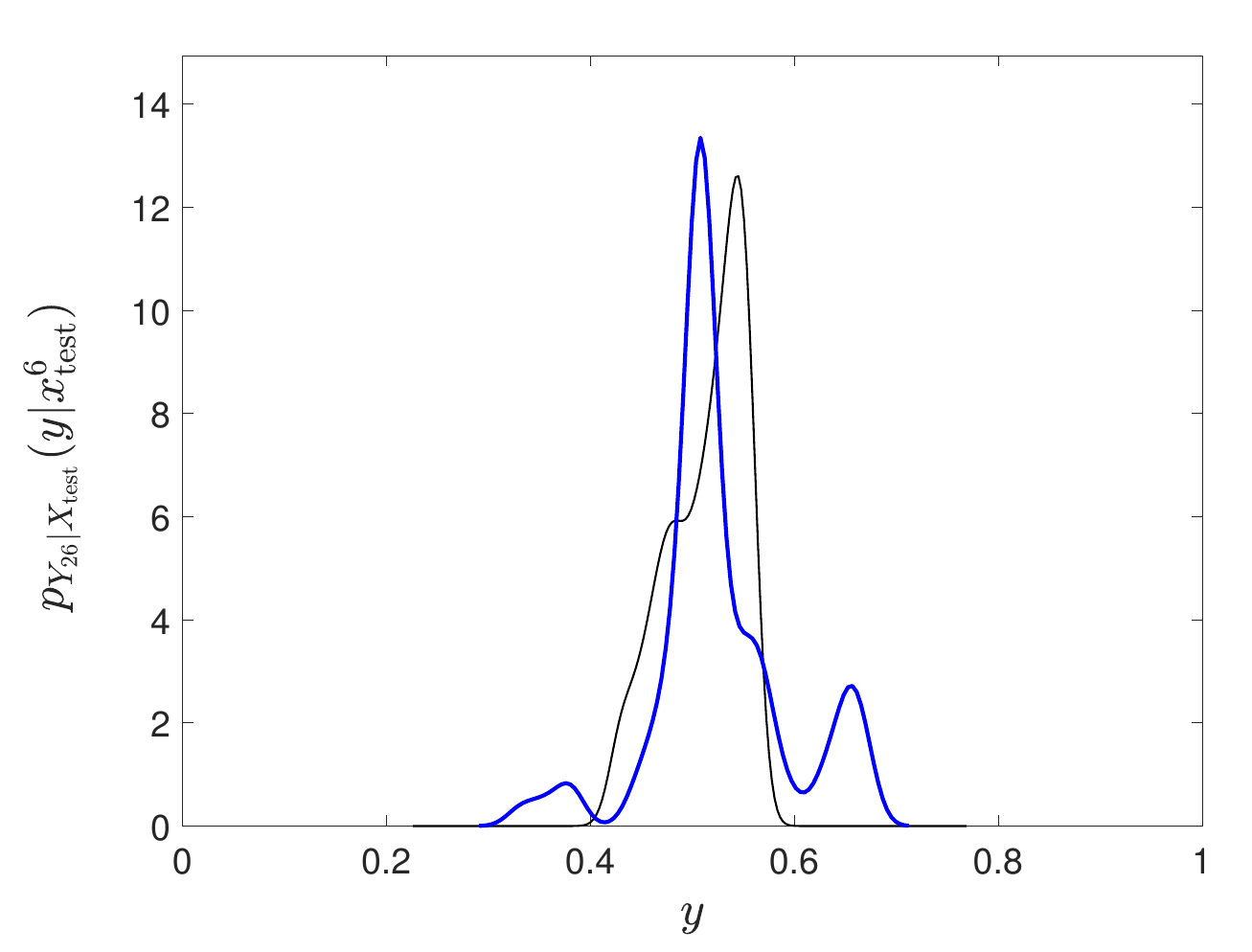}
        \caption{Conditional pdf for $k=26$.}
        \label{figure12b}
    \end{subfigure}
    \begin{subfigure}[b]{0.30\textwidth}
        \centering
        \includegraphics[width=\textwidth]{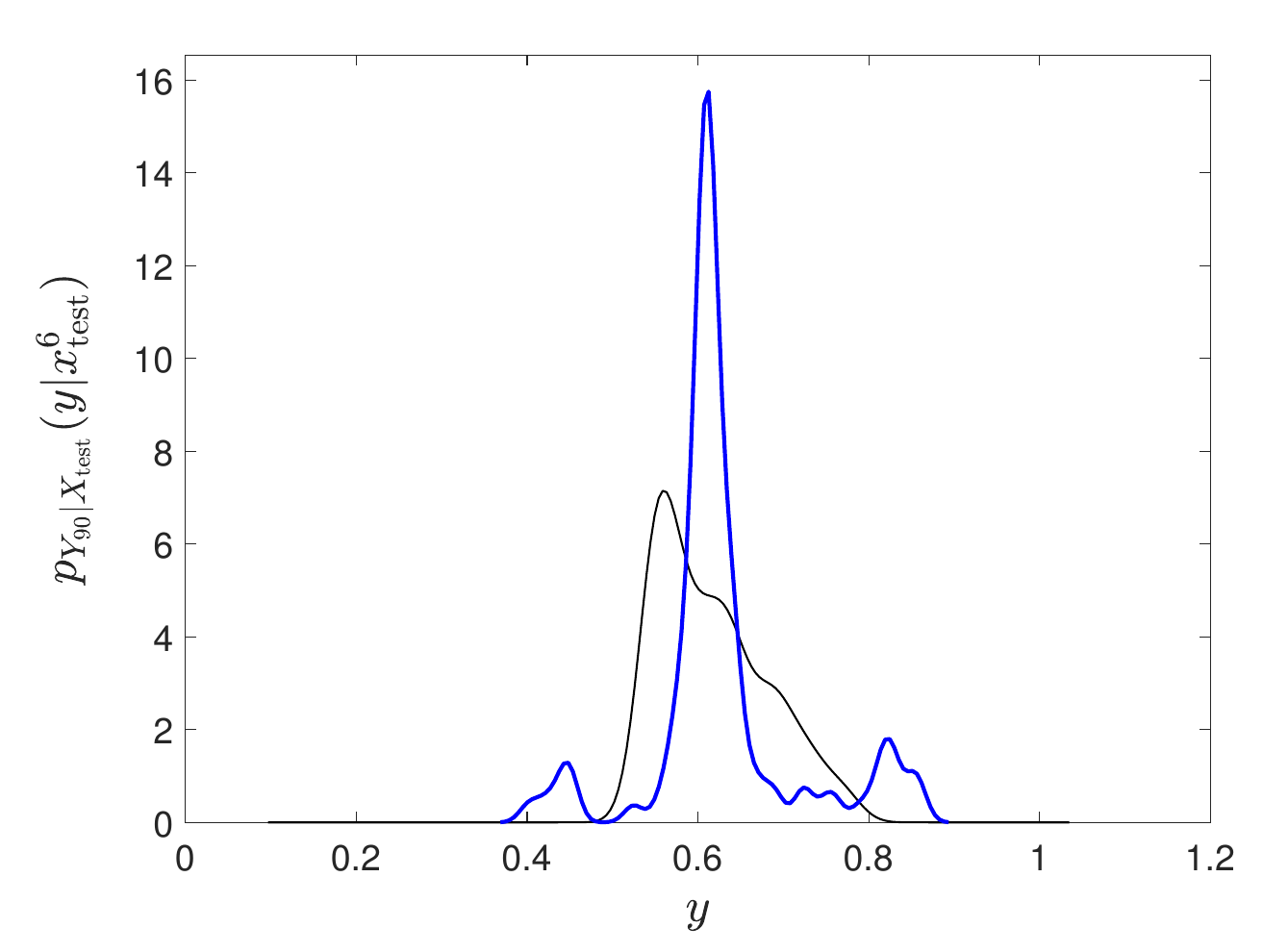}
        \caption{Conditional pdf for $k=90$.}
        \label{figure12c}
    \end{subfigure}
\caption{For $k\in\{5, 26, 90\}$, conditional probability density functions  $y\mapsto  p^\ANNpp_k(y \, \vert \, \xx^6_\testp \, ; \bftheta_t^\optp)$ (thick black line) and $y\mapsto p^\trainingpp_k(y \, \vert \, \xx^6_\testp)$ (thin black line), given the
 sample input $\xx^6_\testp \in\RR^{20}$ in the test dataset, for $N=200$, $n_d=3000$, $n_\testp=600$, {\color{black} and $n_h=10$}.}
    \label{figure12}
\end{figure}
\begin{table}[h]
  \caption{For $k = 5$, $26$, and $90$, conditional confidence intervals $\CI^\ANNpp_k(\xx^6_\testp)$ and $\CI^\trainingpp_k(\xx^6_\testp)$ for given the sample input $\xx^6_\testp \in\RR^{20}$, for $N=200$, $n_d=3000$, $n_\testp=600$, and $n_h=0$}
  \label{table3}
  \begin{center}
  \begin{tabular}{|c|c|c|}
    \hline
        $\scriptstyle k$ & $\scriptstyle \CIp^\trainingpp_k(\xx^6_\testp)$  & $\scriptstyle \CIp^\ANNpp_k(\xx^6_\testp)$ \\
    \hline
     5  & $[0.62\, , 1.23]$   & $[0.58\, , 1.24]$ \\
     26 & $[0.43\, , 0.57]$   & $[0.43\, , 0.66]$ \\
     90 & $[0.52\, , 0.76]$   & $[0.44\, , 0.83]$ \\
    \hline
  \end{tabular}
\end{center}
\end{table}

\noindent\textit{(ii) Sensitivity analysis}.
The sensitivity analysis is conducted with respect to the configuration corresponding to the first case presented in Table~\ref{table2}, for which $N=200$, $n_d=1000$, $n_\testp=200$. This configuration, labeled "Ref", is reproduced in Table~\ref{table4} and serves as the reference for comparison. Four sensitivity analysis cases  are considered:
\begin{itemize}
\item \textit{SA1} - Finite element mesh size: using $160$ subdivisions along the major circle and $48$ along the minor circle (instead of $80$ and $24$), yielding a regular grid with $n_o= 7680$ nodes and $n_\pelem = 15\,360$ triangular elements (each with three nodes), forming a structured mesh $\curS_h$.
\item \textit{SA2} -  Percentile threshold  parameter: setting $\tau_\pprc = 5\%$ instead of $75\%$, which means that $95\%$ of inter-neurons connections are retained.
\item \textit{SA3} -  Smooth differentiable functions $\hh^{(k,j)}$ {\color{black} on $\curS$, as defined in Section~\ref{Section11.2}-(ii), used instead of  the non-smooth differentiable case
   defined in Section~\ref{Section11.2}-(i) when $n_h \geq 1$.}
\item \textit{SA4} -  Number of neurons: using $N=400$ instead of $200$.
\end{itemize}

\noindent The parameters associated with these configurations and their corresponding results are summarized in Table~\ref{table4}.
\begin{table}[h]
  \caption{Sensitivity analysis with respect to a baseline configuration for $N=200$}
  \label{table4}
  \begin{center}
  \resizebox{\textwidth}{!}{%
  \begin{tabular}{|c|c|c|c|c|c|c|c|c|c|c|c|c|c|}
    \hline
    Case & $N$ & $n_d$ & $n_\testp$ & $\tau_\pprc$ &  & $h_{1,\pptrial}^\optp$ & $h_{2,\pptrial}^\optp$ & $\zeta_{s,\pptrial}^\optp$ &
    $n_h$ & $\scriptstyle\curJ(\bftheta^\optp_\pptrial)$ & $\scriptstyle\NLLg_\ANNpp^\optp$  & $\scriptstyle\NLLg_\ANNpp^\testp $ & $\scriptstyle\criter_\CRPSpp$ \\
    \hline
    Ref &  200 & 1000 & 200 & 75 & $\TRIAL$ & 0.0957 & 0.0642 & 0.0726 &  -  & 77.0 & -    & 84.2 & 0.223 \\
        &      &      &     &    & $\PGD$   &        &        &        &  10 &  -   & 81.8 & 75.2 & 0.223 \\
    \hline
    SA1 &  200 & 1000 & 200 & 75 & $\TRIAL$ & 0.0978 & 0.0650 & 0.0740 &  -  & 81.4 & -    & 88.2 & 0.231 \\
        &      &      &     &    & $\PGD$   &        &        &        &  10 &  -   & 80.9 & 89.4 & 0.206 \\
    \hline
    SA2 &  200 & 1000 & 200 &  5 & $\TRIAL$ & 0.0617 & 0.0760 & 0.0760 &  -  & 72.4 & -    & 89.2 & 0.227 \\
    \hline
    SA3 &  200 & 1000 & 200 & 75 & $\TRIAL$ & 0.0957 & 0.0642 & 0.0726 &  -  & 77.0 & -    & 84.2 & 0.223 \\
        &      &      &     &    & $\PGD$   &        &        &        &  0  &  -   & 72.0 & 90.9 & 0.263 \\
        &      &      &     &    & $\PGD$   &        &        &        &  10 &  -   & 79.9 & 78.6 & 0.191 \\
        &      &      &     &    & $\PGD$   &        &        &        &  30 &  -   & 87.6 & 86.3 & 0.208 \\
    \hline
    SA4 &  400 & 1000 & 200 & 75 & $\TRIAL$ & 0.0674 & 0.0963 & 0.0668 &  -  & 84.1 & -    & 84.8 & 0.242 \\
        &      &      &     &    & $\PGD$   &        &        &        &  0  &  -   & 84.0 & 86.1 & 0.254 \\
        &      &      &     &    & $\PGD$   &        &        &        &  10 &  -   & 72.5 & 72.8 & 0.214 \\
    \hline
  \end{tabular}
  } 
\end{center}
\end{table}
\section{Discussions based on the obtained results}
\label{Section12}
In this section, we analyze the numerical results from Section~\ref{Section11}, focusing on the best-performing configuration: $N=200$ neurons, $n_d=3000$ training samples, $n_\testp=600$ test samples, and $n_h =10$, for which the CRPS criterion attains a minimum value of $0.194$.
\paragraph{(i) Neuron placement induced by the latent field.}
As shown in Fig.~\ref{figure3b}, the neuron locations are concentrated in regions where the latent Gaussian field exhibits high local variability. This outcome is a direct consequence of the inhomogeneous Poisson process, whose intensity is defined by the variance map $\Lambda(\bfx;\bftheta)$. Areas of low variance, such as the black zones of the torus surface, receive no neural allocation, not because field values are small, but due to minimal uncertainty.
This allocation strategy reflects an intrinsic adaptation of the architecture to the epistemic structure of the latent field, concentrating representational capacity in regions of high uncertainty where functional variability is most pronounced.
\paragraph{(ii) Optimization convergence and sensitivity to hyperparameters.}
Figure~\ref{figure7} confirms smooth convergence of the {\color{black} loss $\NLL_\ANNp (n)$ and} the normalized maximum parameter change $\Delta_{\max}(n)$ when $n_h=0$ (183 hyperparameters). For $n_h=10$ (203 hyperparameters), Fig.~\ref{figure9a} shows mild fluctuations in the loss function, attributed to the hyperparameters $\bfbeta^{(1)}_{n_h}$ and $\bfbeta^{(2)}_{n_h}$, which govern the anisotropic diffusivity of the latent field and therefore modulate the directional structure of the emergent architecture. Despite this sensitivity
{\color{black} (see also the discussion provided  in Section~\ref{Section8}-(v))}  Fig.~\ref{figure9b} confirms convergence via monotonic decay of $\Delta_{\max}(n)$. The results validate the effectiveness of the projected Adam-based gradient descent scheme, even in the presence of non-convex parameter-field coupling.
\paragraph{(iii) CRPS-based distributional accuracy.}
Table~\ref{table2} shows that the CRPS score $\CRPS^\ANNpp_\testp$ from the random ANN model closely matches the nonparametric reference $\CRPS^\trainingpp_\testp$, with only marginal absolute differences. This proximity indicates that the learned model faithfully reproduces the conditional output distributions. The relative discrepancy measure $\criter_\CRPSpp$, though not negligible, is modest (minimum 0.194) and reflects localized deviations rather than systematic errors.
These discrepancies likely stem {\color{black} from spatially varying variance} or multimodal behavior in the conditional outputs—features that are difficult to capture perfectly across the entire test domain. Importantly, the absolute magnitude of these deviations remains small, and their localized nature suggests that the model does not suffer from systematic miscalibration.
\paragraph{(iv) Generalization and overfitting analysis}
The overfitting analysis is based on the normalized comparison of negative log-likelihoods (NLL) computed on the training and test datasets through the criterion $\criter_o$. For the best model, the obtained normalized values are
$\NLL_\ANNp^\testp /n_\testp = 78.2$, $\NLL_\ANNp^\optp/n_d  = 76.0$, and $\criter_o=0.028$.
The criterion $\criter_o$ quantifies the normalized discrepancy between the training and test likelihoods after correcting for dataset size. The value $\criter_o = 0.028 \ll 1$ indicates a relatively small deviation between train and test losses, suggesting an absence of significant overfitting.
Importantly, the low value of $\criter_o$ is particularly significant given the non-Gaussian nature and high-dimensionality of the output space (here, $n_\pout = 100$), where overfitting is typically more likely due to the curse of dimensionality.
This effect is mitigated in our model by the geometric constraint imposed on neuron placement and connectivity, which guides capacity toward regions of functional complexity and suppresses overparameterization.
Classical neural network models tend to overfit when data are sparse relative to the number of parameters or when outputs are not Gaussian. The present random ANN framework, relying on a generative latent geometry and a limited number of neurons, seems to maintain generalization capability even in these conditions.
The overfitting analysis supports the statistical soundness of the random ANN model in high dimensions. Despite the relatively small network size, the model generalizes well, demonstrating its robustness for learning non-Gaussian conditional distributions with high output dimension.
\paragraph{(v) Confidence interval accuracy for non-Gaussian output random systems in high dimensions}
Let us give a brief positioning related to the confidence interval accuracy.
In the context of modeling conditional random variables $\YY \,\vert \,\xx$ with strongly non-Gaussian distributions  using neural networks with random weights, the accurate estimation of $p_c$-confidence intervals $\CI^\ANNpp_k(\xx) = [\alpha^\ANNpp_k(\xx) \, , \beta^{\,\ANNpp}_k(\xx)]$ for each component $\YY_k\,\vert \,\xx$ remains a challenging task. When employing state-of-the-art methodologies such as Bayesian neural networks with approximate inference \cite{Blundell2015,Hernandez2015}, Monte Carlo dropout \cite{Gal2016}, or deep ensembles \cite{Lakshminarayanan2017}, empirical coverage errors for nominal $p_c= 95\%$ intervals typically range between $5\%$ and $15\%$.  Quantile regression networks \cite{Tagasovska2019} often exhibit higher miscoverage rates in the $10\%-20\%$ range due to quantile crossing and data sparsity. More expressive conditional density estimators such as normalizing flows \cite{Papamakarios2021}, multimodal neural processes \cite{Jung2023}, or neural conditional probability models based on operator learning \cite{Kostic2024} can reduce this error below $10\%$ when adequately trained on large datasets. Recently, likelihood-ratio-based confidence intervals have been proposed to construct asymmetric, data-adaptive intervals with theoretical guarantees \cite{Sluijterman2025}.
In contrast, neural architectures with purely random weights (e.g., Extreme Learning Machines or random feature models) tend to yield significantly less reliable confidence intervals, with typical errors exceeding $20\%$ unless complemented with Bayesian output layers, bootstrap aggregation, or conformal prediction techniques \cite{Louizos2017,VanAmersfoort2020}.

In relatively higher-dimensional settings, such as our illustrative example with input dimension $n_\pin = 20$ and output dimension $n_\pout = 100$, the complexity of modeling the non-Gaussian conditional distribution $\YY \,\vert\, \xx$ increases considerably.

To estimate per-component confidence intervals reliably in such regimes, a significantly larger number of hidden neurons is required. Based on random feature approximation theory \cite{Rahimi2007,Rudi2017}, quantile resolution requirements, and multimodal density complexity, we estimate that a random-weight single hidden layer neural network would require between $5\,000$ and $15\,000$ hidden units to accurately resolve $95\%$ confidence intervals for all $n_\pout$ components. This scaling is consistent with theoretical bounds on random feature-based kernel approximations \cite{Bach2017}, recent asymptotic results on random-feature models in high-dimensional regimes~\cite{Schroder2024,Dabo2025}, and the need to represent high-entropy conditional output distributions with adequate resolution~\cite{Gonon2023}. The neuron count may be reduced when structural assumptions (that is, conditional independence, low-rank covariance {\color{black} in $\YY\vert\xx$}) or compressive priors are exploited.

In this work, we have presented a random neural network architecture structured by latent random fields defined on compact, boundaryless, and multiply-connected manifolds, as introduced in the supervised learning framework. Using only $N = 200$ neurons with $25\%$ connectivity, trained on a dataset of $n_d = 1\,000$ input-output pairs and evaluated on a test set of $200$ points, we assessed the quality of the predicted confidence intervals $\CI^\ANNpp_k(\xx)$ by separately evaluating the errors on the lower bounds $\alpha^\ANNpp_k(\xx)$ and the upper bounds $\beta^{\,\ANNpp}_k(\xx)$.
We obtained an empirical average miscoverage error of $20\%$ on the lower bounds and $11\%$ on the upper bounds, computed across all output components and test samples. These values represent average absolute deviations from the nominal $95\%$ coverage rate and reflect pointwise empirical coverage performance.
This performance, achieved with significantly reduced network complexity, suggests that the geometric structure and stochastic flexibility introduced by the latent random field framework compensate for the lower neuron count. These results indicate that complex, high-dimensional, non-Gaussian conditional distributions can be effectively learned and quantified even with relatively compact random architectures, provided that structural priors are suitably exploited.
\paragraph{(vi) Probabilistic validation using CRPS}
As with the confidence interval analysis, we begin by providing a brief positioning of the CRPS criterion in the context of probabilistic model validation.

The Continuous Ranked Probability Score (CRPS) is widely recognized as a strictly proper scoring rule for validating probabilistic predictions, particularly in non-Gaussian and high-dimensional contexts~\cite{Gneiting2007,Matheson1976}. It quantifies the discrepancy between a predicted cumulative distribution function (CDF) and the empirical step function associated with an observed outcome by integrating the squared difference across the real line. In contrast to pointwise metrics such as the negative log-likelihood (NLL), CRPS captures the full shape of the predictive distribution, making it sensitive to both bias and dispersion. It is robust to heavy tails, supports diagnostic insight in multimodal settings, and has been successfully applied in various domains, including weather forecasting~\cite{Hersbach2000}, Bayesian deep learning~\cite{Lakshminarayanan2017}, and operator learning~\cite{Kovachki2023}.
For univariate predictions on normalized or moderately scaled variables, well-calibrated probabilistic models typically yield CRPS values between $0.1$ and $0.5$, with smaller values indicating better probabilistic fit. For instance, in high-dimensional regression problems with non-Gaussian targets, modern approaches such as deep ensembles~\cite{Lakshminarayanan2017}, quantile flows~\cite{Tagasovska2019}, or diffusion-based conditional models~\cite{Jung2023} report CRPS values ranging from $0.2$ to $0.5$, depending on task complexity, data availability, and output dimensionality. In such regimes, relative CRPS comparisons tend to be more informative than absolute values, especially when evaluating models of differing architectural complexity or inference paradigms.

The accuracy of the random ANN surrogate was assessed using the CRPS score $\CRPS^\ANNpp_\testp$, and compared with $\CRPS^\trainingpp_\testp$, computed using a nonparametric Gaussian Kernel Density Estimation (GKDE) reference model based on the training dataset. For a random ANN with $N = 200$ neurons, $n_d = 1000$ training samples, $n_\testp = 200$ test samples, and $n_h=30$, we obtained
$\CRPS^\ANNpp_\testp = 0.499$ and $\CRPS^\trainingpp_\testp = 0.504$,
with a relative discrepancy, computed using Eq.~\eqref{eq9.10}, given by $\criter_\CRPSpp = 0.214$ (for the best model, we obtain $\criter_\CRPSpp = 0.194$).

The CRPS values are remarkably close in absolute terms, differing by only $0.005$. This numerical proximity demonstrates that the predictive distribution generated by the random ANN is statistically nearly equivalent to the nonparametric GKDE benchmark,
which serves as a biased but smooth reference estimator,
despite the model's structural simplicity and compactness. Although the normalized criterion $\criter_\CRPSpp$ appears moderate at $21.4\%$ (and $19.4.4\%$ for the best model), it is primarily a consequence of normalization by small CRPS values, which magnifies even small absolute differences. As such, $\criter_\CRPSpp$ should be interpreted with care: its value reflects relative sensitivity rather than a substantial discrepancy in forecast skill.

That the random ANN achieves this level of accuracy using only $200$ neurons and $25\%$ connectivity underscores the expressive power of the model, particularly when its architecture is informed by latent geometric priors. This result supports the hypothesis that structured randomness can act as an effective inductive bias, enabling accurate distributional modeling even in high-dimensional, non-Gaussian settings with limited training data.

These findings reinforce the conclusions drawn from the NLL and confidence interval analyses. The proposed random ANN architecture, despite the absence of variational inference layers or deep hierarchical components, provides a robust and statistically valid surrogate model for uncertainty quantification. Its CRPS-based performance confirms its capacity to learn and represent complex conditional distributions with a level of fidelity comparable to nonparametric methods, while remaining computationally efficient and architecturally compact.
\paragraph{(vi) Efficiency and comparative evaluation}
The results obtained for the random neural architecture structured by latent random fields on compact boundaryless multiply-connected manifolds are evaluated with respect to established probabilistic performance criteria and the state of the art in high-dimensional non-Gaussian modeling. Using the results of the best model, the following points summarize the key findings:

\textit{Global distributional accuracy}. The model yields a CRPS score of $\CRPS^\ANNpp_\testp = 0.495$ on the test dataset, which is remarkably close to the nonparametric reference value $\CRPS^\trainingpp_\testp = 0.488$ derived from the empirical distribution of the training data. This proximity confirms that the random neural architecture accurately captures the global structure of the conditional output distribution, despite its strong non-Gaussianity.

\textit{Local distributional discrepancy}. The relative discrepancy criterion yields a value of $\criter_\CRPSpp = 0.194$, indicating moderate pointwise deviations between the model-based and reference CRPS values. This is expected given the strong variability of the output distribution in high dimension ($n_\ppout=100$), and the fact that the benchmark is constructed via kernel density estimation, which tends to produce locally smoother estimates. Importantly, the global match in CRPS shows that the model's deviations are not systematically biased but arise from localized fluctuations, which are natural in such settings.

\textit{Modeling efficiency}. The model uses only $200$ neurons with $25\%$ connectivity (i.e., sparse architecture), yet
successfully maintains predictive fidelity across $n_\testp = 600$ test samples after {\color{black} training on $n_d = 3000$, thereby} accurately modeling a $100$-dimensional non-Gaussian output variable.
This represents a highly efficient use of computational and statistical resources. The architectural regularization introduced by the latent manifold structure plays a key role in enabling generalization in such a high-dimensional stochastic setting. Note that this efficiency also holds for only $1000$ training samples.
\paragraph{(vii) Sensitivity to architectural and geometric design choices}
Based on the criteria $\criter_\CRPSpp$, the sensitivity analysis leads to the following observations.
\begin{itemize}
\item Refining the manifold mesh improves the results, as expected, but does not lead to a significant change in performance.
\item Increasing the connectivity between neurons does not yield further improvements. This suggests that a network with $25\%$ connectivity is sufficiently expressive to capture the non-Gaussian random mapping.
\item Using smooth differentiable functions for $\hh^{(k,j)}$ enhances performance, reducing the criterion $\criter_\CRPSpp$ from $0.223$ to $0.191$.
\item Increasing the number of hidden neurons does not improve the results in the present setting, indicating that this network size is adequate for the task.
\end{itemize}
\section{Conclusions}
Before summarizing the main contributions, we clarify the epistemological intent of this work. Rather than proposing a refinement of existing neural architectures or attempting to replace standard models through performance gains, this paper introduces a distinct paradigm in which the architecture of the neural system is itself a stochastic object, a realization of a generative process governed by latent anisotropic fields defined over compact, boundaryless, multiply-connected manifolds. In this framework, neurons, their connectivity, and synaptic weights emerge jointly from a spatially structured random field, incorporating topological and geometric uncertainty as intrinsic components of the model.
This shift reframes supervised learning as inference over a generative process that produces random function approximators, rather than optimization within a fixed architectural space. It opens a principled new axis of investigation, conceptually distinct from mainstream approaches rooted in deterministic network designs and parametric weight fitting.
This work introduced a new class of geometrically structured random neural networks defined on compact, boundaryless, multiply-connected manifolds, driven by latent anisotropic Gaussian fields and inhomogeneous Poisson point processes. The proposed architecture is not fixed but stochastically generated, yielding both neuron placements and synaptic weights as emergent properties of a latent spatial process. Supervised learning is recast as the inference of generative hyperparameters, rather than optimization over fixed weights, enabling a principled treatment of architectural uncertainty.
This work also initiates a mathematical analysis of a stochastic neural model whose architecture is generated by a latent Gaussian random field on a compact manifold. While a complete theoretical characterization remains open, the results already establish fundamental properties, such as a form of of expressive stochastic capability, supporting the model internal coherence and expressive potential. These initial developments lay the groundwork for a broader mathematical theory of geometry-driven stochastic learning systems.
The presented numerical experiments do not aim to benchmark predictive accuracy or compare against existing architectures. Instead, we provide diagnostic illustrations that elucidate how topological, anisotropic, and latent stochastic parameters influence the emergent neural architecture and its induced function class. These numerical experiments validated the expressivity and efficiency of the model in a challenging test case involving $100$-dimensional outputs with strong non-Gaussian features.
{\color{black} The model demonstrates strong performance under severe resource constraints (e.g., $200$ neurons, $25\%$ connectivity), achieving CRPS scores on par (in the considered setting) with those typically reported for significantly larger architectures, such as Bayesian deep ensembles, deep Gaussian processes, or process models with input-dependent noise.}
 Moreover, the model maintained high-quality empirical coverage of confidence intervals and exhibited minimal overfitting, highlighting its ability to generalize from limited data in high-dimensional output spaces.
The CRPS-based comparisons and overfitting diagnostics support the conclusion that geometric priors, introduced via latent manifold fields, act as powerful inductive biases. They allow the model to concentrate its representational capacity on informative regions of the input manifold, avoiding the overparameterization {\color{black} that often affects} conventional neural approaches. Furthermore, the model supports a statistically coherent treatment of output uncertainty without resorting to variational inference, dropout sampling, or deep ensembling.
While promising, several ways improvement remain. First, the training algorithm could be further optimized. Second, scaling the method to higher-dimensional manifolds or more complex output structures (e.g., tensor-valued responses, sequential data) will require algorithmic and computational refinement.
Future directions include embedding this geometry-driven framework within deep hierarchical compositions, or coupling it with learned activation dynamics via neural ODEs or latent diffusion processes. Such hybrid models may combine the generalization ability of structured randomness with the representational depth of modern architectures.
Finally, hybridizing this latent-geometric approach with learned activation dynamics or deep hierarchical structures may yield architectures that combine statistical robustness with greater modeling depth.
It is worth noting that the proposed model differs fundamentally from traditional Bayesian neural networks or GP-based methods, which define uncertainty in parameter or function space under fixed architectures.
{\color{black} In contrast, our approach places a prior on the architecture itself, which in turn induces a distribution over functions through geometric variability, essentially a function level prior shaped by manifold structure.}
This may complement recent developments in function-space inference by offering an alternative, geometry-centered route to model uncertainty.
In summary, this work establishes a probabilistic foundation for tractable and geometry-aware uncertainty quantification in neural systems, with interpretable inductive structure and scalable simulation algorithms. The integration of random topology, latent-field-driven weights, and manifold-based geometry offers a novel path toward statistically principled learning in data-scarce, geometrically structured domains.

This work also opens a new research direction in the development of Physics-Informed Neural Networks (PINNs) with randomized neural architectures that are statistically consistent with the structure of inhomogeneous physical systems. By leveraging latent random fields on manifolds as stochastic germs for generating PDE coefficients, the network architecture becomes naturally adapted to both the geometry and boundary conditions of the problem domain. This formulation establishes a principled framework for aligning the architecture of PINNs with the physical and probabilistic properties of the target system, laying the groundwork for ongoing developments in inverse modeling, uncertainty quantification, and data-driven discovery in complex stochastic media.

\section*{CRediT authorship contribution statement}
Christian Soize: Conceptualization, Methodology, Software, Computation, Writing.

\section*{Declaration of competing interest}
The author declares that they have no known competing financial interests or personal relationships that could have appeared to influence the work reported in this paper.


\end{document}